%% file: iclr2023_conference.tex
\documentclass{article} % For LaTeX2e

\newcommand{\ie}{\textit{i}.\textit{e}., }
\newcommand{\eg}{\textit{e}.\textit{g}., }

\def\LBCHmeanhns{10077.52}
\def\LBCHmeanhnsle{4.81E-07 }

\def\LBCHmedianhns{1665.60}
\def\LBCHmedianhnsle{5.73E-08 }

\def\LBCHHWRB{24 }
\def\LBCHHWRBle{1.10E-07 }

\def\LBCHmeanHWRNS{154.27}
\def\LBCHmeanHWRNSle{7.71E-09 }

\def\LBCHmedianHWRNS{50.63}
\def\LBCHmedianHWRNSle{2.53E-09 }

\def\LBCHmeanSABER{71.26}
\def\LBCHmeanSABERle{3.56E-09 }

\def\LBCHmedianSABER{50.63}
\def\LBCHmedianSABERle{2.53E-09 }

\def\LBCHnumframes{1.00E+09 }
\def\LBCHgametime{192.5 }
\def\agentmeanhns{4762.17}
\def\agentmeanhnsle{4.76E-10}

\def\agentmedianhns{1933.49}
\def\agentmedianhnsle{1.93E-10}

\def\agentHWRB{18}
\def\agentHWRBle{1.80E-10}

\def\agentmeanHWRNS{125.92}
\def\agentmeanHWRNSle{1.26E-11}

\def\agentmedianHWRNS{43.62}
\def\agentmedianHWRNSle{4.36E-12}

\def\agentmeanSABER{76.26}
\def\agentmeanSABERle{7.63E-12}

\def\agentmedianSABER{43.62}
\def\agentmedianSABERle{4.36E-12}

\def\agentnumframes{1.00E+11}
\def\agentgametime{19290}

\def\muzeromeanhns{4994.97}
\def\muzeromeanhnsle{2.50E-09}

\def\muzeromedianhns{2041.12}
\def\muzeromedianhnsle{1.02E-09}

\def\muzeroHWRB{19}
\def\muzeroHWRBle{9.50E-10}

\def\muzeromeanHWRNS{152.10}
\def\muzeromeanHWRNSle{7.61E-11}

\def\muzeromedianHWRNS{49.80}
\def\muzeromedianHWRNSle{2.49E-11}

\def\muzeromeanSABER{71.94}
\def\muzeromeanSABERle{3.60E-11}

\def\muzeromedianSABER{49.80}
\def\muzeromedianSABERle{2.49E-11}

\def\muzeronumframes{2.00E+10}
\def\muzerogametime{3858}

\def\dreamermeanhns{642.49}
\def\dreamermeanhnsle{3.21E-08}

\def\dreamermedianhns{178.04}
\def\dreamermedianhnsle{8.90E-09}

\def\dreamerHWRB{3}
\def\dreamerHWRBle{1.50E-08}

\def\dreamermeanHWRNS{38.60}
\def\dreamermeanHWRNSle{1.93E-09}

\def\dreamermedianHWRNS{4.29}
\def\dreamermedianHWRNSle{2.14E-10}

\def\dreamermeanSABER{27.73}
\def\dreamermeanSABERle{1.39E-09}

\def\dreamermedianSABER{4.29}
\def\dreamermedianSABERle{2.14E-10}

\def\dreamernumframes{2.00E+08}
\def\dreamergametime{38.58 }

\def\simplemeanhns{25.78}
\def\simplemeanhnsle{2.58E-07}

\def\simplemedianhns{5.55}
\def\simplemedianhnsle{5.55E-08}

\def\simpleHWRB{0}
\def\simpleHWRBle{0.00E+00}

\def\simplemeanHWRNS{4.80}
\def\simplemeanHWRNSle{4.80E-08}

\def\simplemedianHWRNS{0.13}
\def\simplemedianHWRNSle{1.25E-09}

\def\simplemeanSABER{4.80}
\def\simplemeanSABERle{4.80E-08}

\def\simplemedianSABER{0.13}
\def\simplemedianSABERle{1.25E-09}

\def\simplenumframes{1.00E+06}
\def\simplegametime{0.19}

\def\mueslimeanhns{2538.12}
\def\mueslimeanhnsle{1.27E-07}

\def\mueslimedianhns{1077.47}
\def\mueslimedianhnsle{5.39E-08}

\def\muesliHWRB{5}
\def\muesliHWRBle{2.50E-08}

\def\mueslimeanHWRNS{75.52}
\def\mueslimeanHWRNSle{3.78E-09}

\def\mueslimedianHWRNS{24.86}
\def\mueslimedianHWRNSle{1.24E-09}

\def\mueslimeanSABER{48.74}
\def\mueslimeanSABERle{2.44E-09}

\def\mueslimedianSABER{24.86}
\def\mueslimedianSABERle{1.24E-09}

\def\mueslinumframes{2.00E+08}
\def\muesligametime{38.5}

\def\goexploremeanhns{4989.31}
\def\goexploremeanhnsle{4.99E-09}

\def\goexploremedianhns{1451.55}
\def\goexploremedianhnsle{1.45E-09}

\def\goexploreHWRB{15}
\def\goexploreHWRBle{1.50E-09}

\def\goexploremeanHWRNS{116.89}
\def\goexploremeanHWRNSle{1.17E-10}

\def\goexploremedianHWRNS{50.50}
\def\goexploremedianHWRNSle{5.05E-11}

\def\goexploremeanSABER{71.80}
\def\goexploremeanSABERle{7.18E-11}

\def\goexploremedianSABER{50.50}
\def\goexploremedianSABERle{5.05E-11}

\def\goexplorenumframes{1.00E+10}
\def\goexploregametime{1929}

\usepackage{iclr2023_conference,times}

\input{math_commands.tex}

\iclrfinalcopy % Uncomment for camera-ready version, but NOT for submission.

\usepackage{hyperref}
\usepackage{url}

\usepackage{makecell}
\usepackage{comment}
\usepackage{bbm}
\usepackage{bm}
\usepackage[utf8]{inputenc} % allow utf-8 input
\usepackage[T1]{fontenc}    % use 8-bit T1 fonts
\usepackage{hyperref}       % hyperlinks
\usepackage{url}            % simple URL typesetting
\usepackage{booktabs}       % professional-quality tables
\usepackage{amsfonts}       % blackboard math symbols
\usepackage{nicefrac}       % compact symbols for 1/2, etc.
\usepackage{microtype}      % microtypography
\usepackage{xcolor}         % colors

\usepackage{svg}
\usepackage{graphicx}
\usepackage{subfigure}
\usepackage{amsmath}
\usepackage{amsthm}
\usepackage{mathrsfs}
\usepackage{algorithm}
\usepackage{algorithmic}
\usepackage{amssymb}

\usepackage{color}

\theoremstyle{plain}

\newtheorem{definition}{Definition}[section]

\newtheorem{Corollary}{\textbf{Corollary}}
\newtheorem*{Remark}{\textbf{Remark}}

\newtheorem{Assumption}{\textbf{Assumption}}
\newtheorem{Proposition}{\textbf{Proposition}}
\newtheorem{Example}{Example}

\definecolor{zise}{rgb}{0.83137255, 0.36470588, 0.83137255}

\definecolor{chengse}{rgb}{0.99607843, 0.31764706, 0.12156863}

\definecolor{airforceblue}{rgb}{0.36, 0.54, 0.66}
\definecolor{lightgreen}{rgb}{0.56, 0.93, 0.56}
\definecolor{lightseagreen}{rgb}{0.13, 0.7, 0.67}
\definecolor{limegreen}{rgb}{0.2, 0.8, 0.2}
\definecolor{lincolngreen}{rgb}{0.11, 0.35, 0.02}
\definecolor{mediumseagreen}{rgb}{0.24, 0.7, 0.44}
\definecolor{napiergreen}{rgb}{0.16, 0.5, 0.0}
\definecolor{parisgreen}{rgb}{0.31, 0.78, 0.47}
\definecolor{teagreen}{rgb}{0.82, 0.94, 0.75}
\definecolor{yellow-green}{rgb}{0.6, 0.8, 0.2}
\definecolor{applegreen}{rgb}{0.55, 0.71, 0.0}

\newcommand*{\imgintext}[1]{%
  \raisebox{-.3\baselineskip}{%
    \includegraphics[
      height=\baselineskip,
      width=\baselineskip,
      keepaspectratio,
    ]{#1}%
  }%
}
\title{Learnable Behavior Control: Breaking Atari Human World Records via Sample-Efficient Behavior Selection}
\author{Jiajun Fan$^1$ , Yuzheng Zhuang$^2$ ,  Yuecheng Liu$^2$ ,  Jianye Hao$^2$ ,  Bin Wang$^2$ \AND  Jiangcheng Zhu$^3$,    Hao Wang$^4$ ,  Shutao Xia$^1$  \\ \\
$^1$ Tsinghua Shenzhen International Graduate School, Tsinghua University\\ $^2$ Huawei Noah's Ark Lab, $^3$ Huawei Cloud, $^4$ Zhejiang University\\
$^1$ fanjj21@mails.tsinghua.edu.cn, xiast@sz.tsinghua.edu.cn, $^4$ haohaow@zju.edu.cn,  \\
$^{2,3}$ \{zhuangyuzheng, liuyuecheng1, haojianye, wangbin158, zhujiangcheng\}@huawei.com
\vspace{-0.15in} 
}

\begin{document}

% of properly parameterized
\maketitle

\begin{abstract}
The exploration problem is one of the main challenges in deep reinforcement learning (RL). Recent promising works tried to handle the problem with population-based methods, which collect samples with diverse behaviors derived from a population of different exploratory policies. Adaptive policy selection has been adopted for behavior control. However, the behavior selection space is largely limited by the predefined policy population, which further limits behavior diversity. 
In this paper, we propose a general framework called \textbf{L}earnable \textbf{B}ehavioral \textbf{C}ontrol (LBC) to address the limitation, which a) enables a significantly enlarged behavior selection space via formulating a \textit{hybrid behavior mapping} from all policies; b) constructs a unified  \textit{learnable process} for behavior selection. We introduce LBC into distributed off-policy actor-critic methods and achieve behavior control via optimizing the selection of the behavior mappings with bandit-based meta-controllers.
Our agents have achieved \LBCHmeanhns\% mean human normalized score and surpassed \LBCHHWRB human world records within 1B training frames in the Arcade Learning Environment,
which demonstrates our significant state-of-the-art (SOTA) performance without degrading the sample efficiency. 

% Open-sourced code will be implemented with Mindspore \citep{MS} and released on our \href{https://gitee.com/mindspore/models/tree/master/research/rl}{homepage}.
\end{abstract}
\vspace{-0.1in}
\begin{figure*}[!h]
\vspace{-0.1in}
\centering
	\subfigure{
    \includegraphics[width=0.84\textwidth,height=0.14\textheight]{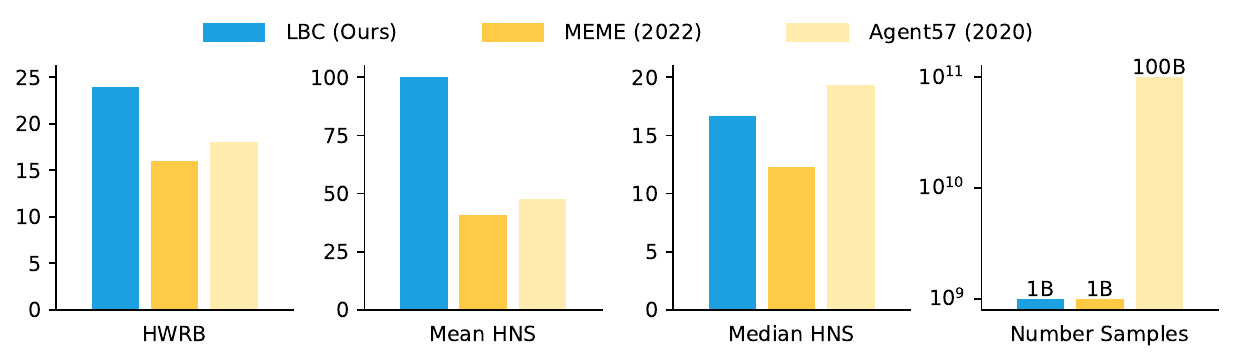}
	}
	\vspace{-0.1in}
	\centering
	\caption{Performance on the 57 Atari games. Our method achieves the highest mean human normalized scores \citep{agent57}, is the first to breakthrough 24 human world records \citep{atarihuman}, and demands the least training data. }
	\label{fig: intro performance}
	\vspace{-0.15in}
\end{figure*}

% The human normalized scores is calculated as $\frac{\text{Agent}_{\text{score}}-\text{Random}_{\text{score}}}{\text{Human}_{\text{score}}-\text{Random}_{\text{score}}}$. The human world records breakthrough is calculated as $\text{Agent}_{\text{score}} \ge \text{Human World Records}_{\text{score}}$

\section{Introduction}
\label{sec: introduction}
Reinforcement learning (RL) has led to tremendous progress in a variety of domains ranging from video games~\citep{dqn} to robotics~\citep{trpo,ppo}. 
However, efficient exploration remains one of the significant challenges. 
Recent prominent works tried to address the problem with population-based training  ~\citep[PBT]{PBT} wherein a population of policies with different degrees of exploration is jointly trained to keep both the long-term and short-term exploration capabilities throughout the learning process. A set of actors is created to acquire diverse behaviors derived from the policy population~\citep{ngu, agent57}. 
Despite the significant improvement in the performance, these methods suffer from the aggravated high sample complexity due to the joint training on the whole population while keeping the diversity property. 
   To acquire diverse behaviors, NGU~\citep{ngu} uniformly selects policies in the population regardless of their contribution to the learning progress~\citep{ngu}. As an improvement, Agent57 adopts an adaptive policy selection mechanism that each behavior used for sampling is periodically selected from the population according to a meta-controller~\citep{agent57}. Although Agent57 achieved significantly better results on the Arcade Learning Environment (ALE) benchmark, it costs tens of billions of environment interactions as much as NGU.  To handle this drawback, GDI \citep{gdi} adaptively combines multiple advantage functions learned from a single policy to obtain an enlarged behavior space without increasing policy population size. However, the population-based scenarios with more than one learned policy has not been widely explored yet.  Taking a further step from GDI, we try to enable a larger and non-degenerate behavior space by learning different combinations across a population of different learned policies.
   
   % the behavior space may degenerate as the difference among the advantage functions learned from the same policy would gradually decrease (See App. \ref{sec: KL Divergence Between policy models}). 

\begin{figure*}[!t]
\vspace{-0.3in}
\centering
	\includegraphics[width=0.8\textwidth]{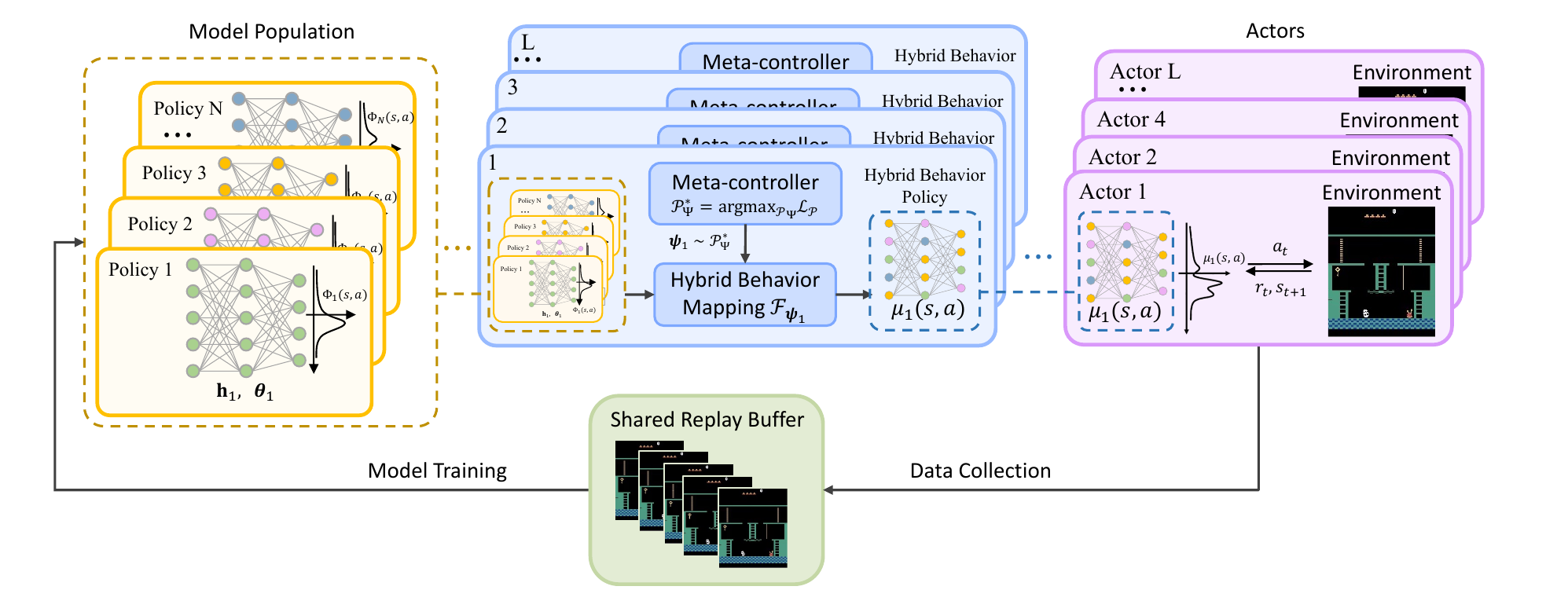}
	\vspace{-0.15in}
	\centering
	\caption{A General Architecture of Our Algorithm. }
	\label{fig: fw}
	\vspace{-0.2in}
\end{figure*}

In this paper,  we attempt to further improve the sample efficiency of population-based reinforcement learning methods by taking a step towards a more challenging setting to control behaviors with significantly enlarged behavior space with a population of different learned policies.  Differing from all of the existing works where each behavior is derived from a single selected learned policy, we formulate the process of getting behaviors from all learned policies as \emph{hybrid behavior mapping}, and the behavior control problem is directly transformed into selecting appropriate mapping functions. By combining all policies, the behavior selection space increases exponentially along with the population size. As a special case that \emph{population size degrades to one}, diverse behaviors can also be obtained by choosing different behavior mappings. This two-fold mechanism enables tremendous larger space for behavior selection. 
By properly parameterizing the mapping functions, our method enables a unified  learnable process, and we call this general framework \textbf{L}earnable \textbf{B}ehavior \textbf{C}ontrol.

We use the Arcade Learning Environment (ALE) to evaluate the performance of the proposed methods, which is an important testing ground that requires a broad set of skills such as perception, exploration, and control ~\citep{agent57}. 
Previous works use the normalized human score 
to summarize the performance on ALE and claim superhuman performance ~\citep{ale}. However, the human baseline is far from representative of the best human player, which greatly underestimates the ability of humanity. In this paper, we introduce a more challenging baseline, \ie the human world records baseline (see ~\citet{atarihuman,dreamerv2} for more information on Atari human world records). We summarize the number of games that agents can outperform the human world records (\ie HWRB, see Figs. \ref{fig: intro performance}) to claim a real superhuman performance in these games, inducing a more challenging and fair comparison with human intelligence. Experimental results show that the sample efficiency of our method also outperforms the concurrent work MEME~\cite{meme}, which is 200x faster than Agent57. In summary, our contributions are as follows:

\begin{enumerate}
    \item \textbf{A data-efficient RL framework named LBC}. We propose a general framework called \textbf{L}earnable \textbf{B}ehavior \textbf{C}ontrol (LBC), which enables a significantly enlarged behavior selection space without increasing the policy population size via formulating a hybrid behavior mapping from all policies, and constructs a unified  learnable process for behavior selection.
    \item \textbf{A family of LBC-based RL algorithms}. We provide a family of LBC-based algorithms by combining LBC with existing distributed off-policy RL algorithms, which shows the generality and scalability of the proposed method.
    \item \textbf{The state-of-the-art performance with superior sample efficiency}. From Figs. \ref{fig: intro performance}, our method has achieved 10077.52\% mean human normalized score (HNS) and surpassed 24 human world records within 1B training frames in the Arcade Learning Environment (ALE), which demonstrates our state-of-the-art (SOTA) sample efficiency.
\end{enumerate}

\section{Background}
\label{sec: Background}

\subsection{Reinforcement Learning}

The Markov Decision Process $\left(\mathcal{S}, \mathcal{A}, p, r, \gamma, \rho_{0}\right)$ ~\citep[MDP]{howard1960dynamic} can be used to model RL. With respect to a discounted episodic MDP, the initial state $s_0$ is taken as a sample from the initial distribution $\rho_0(s): \mathcal{S} \rightarrow \mathbb{P}(\mathcal{S})$, where $\mathbb{P}$ is used to denote the probability distribution. Every time $t$, the agent selects an action $a_t \in \mathcal{A}$ in accordance with the policy $\pi(a_t|s_t): \mathcal{S} \rightarrow \mathbb{P}(\mathcal{A})$ at state $s_t \in \mathcal{S}$.  In accordance with the transition distribution $p\left(s^{\prime} \mid s, a\right): \mathcal{S} \times \mathcal{A} \rightarrow \mathbb{P}(\mathcal{S})$, the environment gets the action $a_t$, creates the reward $r_t \sim r(s,a): \mathcal{S} \times \mathcal{A} \rightarrow \mathbf{R}$, and transfers to the subsequent state $s_{t+1}$. Until the agent achieves a terminal condition or a maximum time step, the process continues. The discounted state visitation distribution is defined as $d_{\rho_0}^{\pi} (s) = (1 - \gamma) \mathbb{E}_{s_0 \sim \rho_0} 
 \left[ \sum_{t=0}^{\infty} \gamma^t \textbf{P} (s_t = s | s_0) \right]$. Define return $G_t = \sum_{k=0}^\infty \gamma^k r_{t+k}$ wherein $\gamma \in(0,1)$ is the discount factor.  Finding the optimal policy $\pi^*$ to maximize the expected sum of discounted rewards $G_t$ is the aim of reinforcement learning:

\begin{equation}
\setlength{\abovedisplayskip}{0.5pt}
\setlength{\belowdisplayskip}{0.5pt}
\label{eq_accmulate_reward}
\begin{aligned}
\pi^{*}
:= \underset{\pi}{\operatorname{argmax}} \mathbb{E}_{s_t \sim d_{\rho_0}^{\pi}} \mathbb{E}_{\pi} \left[G_t = \sum_{k=0}^{\infty} \gamma^{k} r_{t+k} | s_t \right], \\
\end{aligned}
\end{equation}
% where $\gamma \in(0,1)$ is the discount factor. 

% We can further define the value function  as $V^\pi(s) = \mathbb{E}_{\pi} \left[\sum_{k=0}^{\infty}\gamma^k r_{t+k+1} | s_t=s\right]$, the state-action value function as $ Q^\pi(s,a) = \mathbb{E}_{\pi} \left[\sum_{k=0}^{\infty}\gamma^k r_{t+k+1} | s_t=s, a_t=a \right]$, and the advantage function as $A^{\pi}(s,a) = Q^\pi (s,a) - V^\pi (s)$.

\begin{comment}
1. PBT 和 population-based的关系没有说清楚
2. 
\end{comment}

\subsection{Behavior Control For Reinforcement Learning}

 In value-based methods, a behavior policy can be derived from a state-action value function $Q_{\bm{\theta}, \mathbf{h}}^\pi (s,a)$ via $\epsilon$-greedy. In policy-based methods, a behavior policy can be derived from the policy logits $\Phi_{\bm{\theta},\mathbf{h}}$ ~\citep{policylogits} via Boltzmann operator. For convenience, we define that a behavior policy can be derived from the learned policy model $\Phi_{\bm{\theta},\mathbf{h}}$  via a behavior mapping, which normally maps a single policy model to a behavior, \eg $\epsilon\text{-greedy}(\Phi_{\bm{\theta},\mathbf{h}})$.  In PBT-based methods, there would be a set of policy models $\{ \Phi_{\bm{\theta}_1, \mathbf{h}_1},...,\Phi_{\bm{\theta}_\mathrm{N}, \mathbf{h}_\mathrm{N}} \}$, each of which is parameterized by $\bm{\theta}_i$ and trained under its own hyper-parameters $\mathbf{h}_i$, wherein $\bm{\theta}_i \in \bm{\Theta}= \{\bm{\theta}_1,...,\bm{\theta}_\mathrm{N}\}$ and $\mathbf{h}_i \in \mathbf{H}=\{\mathbf{h}_1,...,\mathbf{h}_\mathrm{N}\}$.

The behavior control in population-based methods is normally achieved in two steps: i) select a policy model $\Phi_{\bm{\theta},\mathbf{h}}$ from the population. ii) applying a behavior mapping to the selected policy model. When the behavior mapping is rule-based for each actor (\eg $\epsilon$-greedy with rule-based $\epsilon$ ), the behavior control can be transformed into the policy model selection (See Proposition \ref{proposition: Behavior Control via policy models Selection}). Therefore, the optimization of the selection of the policy models becomes one of the critical problems in achieving effective behavior control. Following the literature on PBRL, NGU adopts a uniform selection, which is unoptimized and inefficient. Built upon NGU, Agent57 adopts a meta-controller to adaptively selected a policy model from the population to generate the behavior for each actor, which is implemented by a non-stationary multi-arm bandit algorithm. However, the policy model selection requires maintaining a large number of different policy models, which is particularly data-consuming since each policy model in the population holds heterogeneous training objectives. 

To handle this problem, recent notable work GDI-H$^3$  \citep{gdi} enables to obtain an enlarged behavior space via adaptively controls the temperature of the softmax operation over the weighted advantage functions. However, since the advantage functions are derived from the same target policy under different reward scales, the distributions derived from them may tend to be similar (e.g., See App. \ref{sec: KL Divergence Between policy models}), thus would lead to degradation of the behavior space.
Differing from all of the existing works where each behavior is derived from a single selected learned policy, in this paper, 
we try to handle this problem via three-fold: i) we bridge the relationship between the learned policies and each behavior via a hybrid behavior mapping,  ii) we propose a general way to build a non-degenerate large behavior space for population-based methods in Sec. \ref{sec: Boltzmann Based Behavior Space Construction}, iii) we propose a way to optimize the hybrid behavior mappings from a population of different learned models in Proposition. \ref{proposition: Behavior Control via behavior mapping optimization}.

\section{Learnable Behavior Control}
\label{sec: Learnable Behavior Control}
% \yuecheng{Missing transition sentence}
In this section, we first formulate the behavior control problem and decouple it into two sub-problems: behavior space construction and behavior selection. Then, we discuss how to construct the behavior space and select behaviors based on the formulation. By integrating both, we can obtain a general framework to achieve behavior control in RL,  called learnable behavior control (LBC). 
\subsection{Behavior Control Formulation}
\label{sec: Problem Formulation}

\paragraph{Behavior Mapping}  Define \emph{behavior mapping} $\mathcal{F}$ as a mapping from some policy model(s) to a behavior. In previous works, a behavior policy is typically obtained using a single policy model. In this paper, as a generalization, we define two kinds of $\mathcal{F}$ according to how many policy models they take as input to get a behavior. 
The first one, \emph{individual behavior mapping}, is defined as a mapping from a single model to a behavior that is widely used in prior works, \eg $\epsilon$-greedy and Boltzmann Strategy for discrete action space and Gaussian Strategy for continuous action space; And the second one, \emph{hybrid behavior mapping}, is defined to map all policy models to a single behavior, \ie $\mathcal{F}(\Phi_{\bm{\theta}_1, \mathbf{h}_1},...,\Phi_{\bm{\theta}_\mathrm{N},\mathbf{h}_\mathrm{N}})$. The hybrid behavior mapping enables us to get a hybrid behavior by combining all policies together, which provides a greater degree of freedom to acquire a larger behavior space. For any behavior mapping $\mathcal{F}_{\bm{\psi}}$ parameterized by $\bm{\psi}$, there exists a family of behavior mappings $\mathcal{F}_{\bm{\Psi}} = \{ \mathcal{F}_{\bm{\psi}} | \bm{\psi}\in\bm{\Psi} \}$ that hold the same parametrization form with  $\mathcal{F}_{\bm{\psi}}$, where $\bm{\Psi} \subseteq \textbf{R}^k$  is a parameter set that contains all possible parameter $\bm{\psi}$.

\paragraph{Behavior Formulation} 
As described above, in our work, a behavior can be acquired by applying a behavior mapping $\mathcal{F}_{\bm{\psi}}$ to some policy model(s). For the individual behavior mapping case, a behavior can be formulated as $\mu_{\bm{\theta},\mathbf{h},\bm{\psi}}=\mathcal{F}_{\bm{\psi}}(\bm{\Phi}_{\bm{\theta},\mathbf{h}})$, which is also the most used case in previous works. As for the hybrid behavior mapping case, a behavior is formulated as $\mu_{\bm{\Theta},\mathbf{H},\bm{\psi}}=\mathcal{F}_{\bm{\psi}}(\bm{\Phi}_{\bm{\Theta},\mathbf{H}})$, wherein $\bm{\Phi}_{\bm{\Theta},\mathbf{H}}=\{\Phi_{\bm{\theta}_1, \mathbf{h}_1},...,\Phi_{\bm{\theta}_\mathrm{N}, \mathbf{h}_\mathrm{N}}\}$ is  a policy model set containing all policy models.

\paragraph{Behavior Control Formulation} Behavior control can be decoupled into two sub-problems: 1) which behaviors can be selected for each actor at each training time, namely the \emph{behavior space construction}. 2) how to select proper behaviors, namely the \emph{behavior selection}. Based on the behavior formulation, we can formulate these sub-problems:
\begin{definition}[Behavior Space Construction]
\label{def: bsc}
    Considering the RL problem that behaviors $\mu$ are generated from some policy model(s). We can acquire a family of realizable behaviors by applying a family of behavior mappings $\mathcal{F}_{\bm{\Psi}}$ to these policy model(s). Define the set that contains all of these realizable behaviors as the behavior space, which can be formulated as:
    \begin{equation}
\label{equ:behavior space formulation}
\mathbf{M}_{\bm{\Theta},\mathbf{H},\bm{\Psi}} =
\begin{cases}
   \{\mu_{\bm{\theta}, \mathbf{h},\bm{\psi}}=\mathcal{F}_{\bm{\psi}}(\Phi_{\mathbf{h}}) | \bm{\theta} \in \bm{\Theta}, \mathbf{h} \in \mathbf{H}, \bm{\psi} \in \bm{\Psi} \}, \quad\text{for individual behavior mapping}\\
  \{\mu_{\bm{\Theta}, \mathbf{H},\bm{\psi}}=\mathcal{F}_{\bm{\psi}}(\bm{\Phi}_{\bm{\Theta}, \mathbf{H}}) |  \bm{\psi} \in \bm{\Psi} \}, \quad\text{for hybrid behavior mapping}
\end{cases}
\end{equation}

\end{definition}

\begin{definition}[Behavior Selection]
\label{def: bs}
   Behavior selection  can be formulated as finding a optimal selection distribution $\mathcal{P}^*_{\mathbf{M}_{\bm{\Theta},\mathbf{H},\bm{\Psi}}}$ to select the behaviors $\mu$ from behavior space $\mathbf{M}_{\bm{\Theta},\mathbf{H},\bm{\Psi}}$ and maximizing some optimization target $\mathcal{L}_{\mathcal{P}}$, wherein $\mathcal{L}_{\mathcal{P}}$ is the optimization target of behavior selection:
    \begin{equation}
\label{equ:behavior selection formulation}
    \mathcal{P}^*_{\mathbf{M}_{\bm{\Theta},\mathbf{H},\bm{\Psi}}} := \underset{\mathcal{P}_{\mathbf{M}_{\bm{\Theta},\mathbf{H},\bm{\Psi}}}}{\operatorname{argmax}}  \mathcal{L}_\mathcal{P}
\end{equation}
\end{definition}

\subsection{Behavior Space Construction}
\label{sec: Behavior Space Construction}
In this section, we further simplify the \eqref{equ:behavior space formulation}, and discuss how to construct the behavior space.
\begin{Assumption}
\label{ass: share model}
    Assume all policy models share the same network structure, and $\mathbf{h}_i$ can uniquely index a policy model $\Phi_{\bm{\theta}_i, \mathbf{h}_i}$. Then, $\Phi_{\bm{\theta},\mathbf{h}}$ can be abbreviated as $\Phi_{\mathbf{h}}$. 
        % Assume all policy models share the same network structure, and do not consider the impact of random initialization and other randomization operations during training (such as stochastic batch sampling), the parameters $\bm{\theta}$ of a learned model $\Phi_{\bm{\theta},\mathbf{h}}$ is only conditioned on its hyper-parameters $\mathbf{h}$. Thus $\Phi_{\bm{\theta},\mathbf{h}}$ can be abbreviated as $\Phi_{\mathbf{h}}$.
        % $\Phi_{\bm{\theta},\mathbf{h}}$ can be uniquely indexed by the hyper-parameters $\mathbf{h}$, namely $\Phi_{\bm{\theta},\mathbf{h}}=\Phi_{\mathbf{h}}$.
\end{Assumption}
Unless otherwise specified, in this paper, we assume Assumption \ref{ass: share model} holds. Under Assumption \ref{ass: share model}, the behavior space defined in \eqref{equ:behavior space formulation} can be simplified as,
\begin{equation}
\label{equ: behavior space mapping-based}
\mathbf{M}_{\mathbf{H},\bm{\Psi}} =
\begin{cases}
   \{\mu_{\mathbf{h},\bm{\psi}}=\mathcal{F}_{\bm{\psi}}(\Phi_{\mathbf{h}}) | \mathbf{h} \in \mathbf{H}, \bm{\psi} \in \bm{\Psi} \}, \quad\text{for individual behavior mapping}\\
  \{\mu_{\mathbf{H},\bm{\psi}}=\mathcal{F}_{\bm{\psi}}(\bm{\Phi}_{\mathbf{H}}) |  \bm{\psi} \in \bm{\Psi} \}, \quad\text{for hybrid behavior mapping}
\end{cases}
\end{equation}

According to \eqref{equ: behavior space mapping-based}, four core factors need to be considered when constructing a behavior space: the network structure $\Phi$, the form of behavior mapping $\mathcal{F}$, the hyper-parameter set $\mathbf{H}$ and the parameter set $\bm{\Psi}$. Many notable representation learning approaches have explored how to design the network structure ~\citep{dt, got}, but it is not the focus of our work. In this paper, we do not make any assumptions about the model structure,  which means it can be applied to \textit{any} model structure. Hence, there remains three factors, which will be discussed below.

% The remaining three factors are related to the choice of a behavior

For cases that behavior space is constructed with \emph{individual behavior mappings},  there are two things to be considered  if one want to select a specific behavior from the behavior space: the policy model $\bm{\Phi}_{\mathbf{h}}$ and behavior mapping $\mathcal{F}_{\bm{\psi}}$. Prior methods have tried to realize behavior control via  selecting a policy model $\bm{\Phi}_{\mathbf{h}_i}$ from the population $\{ \Phi_{ \mathbf{h}_1},...,\Phi_{ \mathbf{h}_\mathrm{N}} \}$ (See Proposition \ref{proposition: Behavior Control via policy models Selection}). The main drawback of this approach is that only one policy model is considered to generate behavior,  leaving other policy models in the population unused. In this paper, we argue that we can tackle this problem via \emph{hybrid behavior mapping}, wherein the hybrid behavior is generated based on all policy models. 

In this paper, we only consider the case that all of the N policy models are used for behavior generating, \ie $\mu_{\mathbf{H},\bm{\psi}}=\mathcal{F}_{\bm{\psi}}(\bm{\Phi}_{\mathbf{H}})$. Now there is only one thing to be considered , \ie the behavior mapping function $\mathcal{F}_{\bm{\psi}}$, and the behavior control problem will be transformed into the optimization of the behavior mapping (See Proposition \ref{proposition: Behavior Control via behavior mapping optimization}). We also do not make any assumptions about the form of the mapping. As an example, one could acquire a hybrid behavior from all policy models via network distillation, parameter fusion, mixture models, etc.

\begin{comment}
and RL agents can automatically balance the exploration-exploitation trade-off of all policy models simply by selecting proper behavior mapping.
    \bm{\psi}_j 是针对actor-wise的方法说的。 也就是说只有在选择的时候才需要考虑\bm{\psi}_j 平时用一个虚拟代词\bm{\psi}就可以了
    但是\bm{\psi}是针对全部的采样策略定义的。所以\bm{\psi}代表任意一种可能的取值，\bm{\psi}_j代表某种确定的取值。即\bm{\psi}_j = \bm{\psi} 
\end{comment}

\subsection{Behavior Selection}
\label{sec: Behavior Selection}
% Recent notable work GDI \citep{gdi} has noticed the effectiveness of optimizing the behavior selection and propose a way to achieve behavior control for a single learned policy, which has achieved great success. However, it still lacks a general study of how to achieve behavior control for a population of learned policies. In this section, we will propose several ways to handle this problem.

According to \eqref{equ: behavior space mapping-based}, each behavior can be indexed by $\mathbf{h}$ and $\bm{\psi}$ for individual behavior mapping cases, and when the $\bm{\psi}$ is not learned for each actor, the behavior selection can be cast to the selection of $\mathbf{h}$  (see Proposition \ref{proposition: Behavior Control via policy models Selection}). As for the hybrid behavior mapping cases, since each behavior can be indexed by $\bm{\psi}$, the behavior selection can be cast into the selection of $\bm{\psi}$ (see Proposition \ref{proposition: Behavior Control via behavior mapping optimization}). Moreover, according to  \eqref{equ:behavior selection formulation}, there are two keys in behavior selection: \textbf{1)} Optimization Target $\mathcal{L}_{\mathcal{P}}$. \textbf{2)}  The optimization algorithm to learn the selection distribution $\mathcal{P}_{\mathbf{M}_{\mathbf{H},\bm{\Psi}}}$ and maximize $\mathcal{L}_{\mathcal{P}}$. In this section, we will discuss them sequentially.

\paragraph{Optimization Target} Two core factors have to be considered for the optimization target: the diversity-based measurement $V^{\mathrm{TD}}_{\mu}$ 
~\citep{DIAYN} and the value-based measurement $V^{\mathrm{TV}}_{\mu}$~\citep{DvD}. By integrating both, the optimization target can be formulated as:
\begin{equation}
\label{equ: reward-diversity trade-off problem}
\begin{aligned}
    \mathcal{L}_{\mathcal{P}}
    & =\mathcal{R}_{\mu \sim \mathcal{P}_{\mathbf{M}_{\mathbf{H},\bm{\Psi}}} } + c \cdot \mathcal{D}_{\mu \sim \mathcal{P}_{\mathbf{M}_{\mathbf{H},\bm{\Psi}}} } \\ 
    & = 
    \mathbb{E}_{\mu \sim \mathcal{P}_{\mathbf{M}_{\mathbf{H},\bm{\Psi}}}}  [V^{\mathrm{TV}}_{\mu}+c \cdot V^{\mathrm{TD}}_{\mu}], \\
\end{aligned}
\end{equation}
wherein, $\mathcal{R}_{\mu \sim \mathcal{P}_{\mathbf{M}_{\mathbf{H},\bm{\Psi}}} } $ and $\mathcal{D}_{\mu \sim \mathcal{P}_{\mathbf{M}_{\mathbf{H},\bm{\Psi}}} }$ is the expectation of value and diversity of behavior $\mu$ over the selection distribution $ \mathcal{P}_{\mathbf{M}_{\mathbf{H},\bm{\Psi}}}$. When $\mathcal{F}_{\bm{\psi}}$ is \textit{unlearned} and \textit{deterministic} for each actor, behavior selection for \textit{each actor} can be simplified into the selection of the policy model:

\begin{Proposition}[Policy Model Selection]
\label{proposition: Behavior Control via policy models Selection} 
    When $\mathcal{F}_{\bm{\psi}}$ is a deterministic and  individual behavior mapping for each actor  at each training step (wall-clock), \textit{e}.\textit{g}., \textbf{Agent57}, the behavior for each actor can be uniquely indexed by $\mathbf{h}$, so \eqref{equ: reward-diversity trade-off problem} can be simplified into
\begin{equation}
\label{equ:pbt target}
    \mathcal{L}_{\mathcal{P}} = \mathbb{E}_{\mathbf{h} \sim \mathcal{P}_\mathbf{H}} \left[ V_{\mu_{\mathbf{h}}}^{\mathrm{TV}} + c\cdot V_{\mu_{\mathbf{h}}}^{\mathrm{TD}} \right],
\end{equation}
where $\mathcal{P}_\mathbf{H}$ is a selection distribution of $\mathbf{h} \in \mathbf{H}=\{\mathbf{h}_1,...,\mathbf{h}_\mathrm{N}\}$. For each actor, the behavior is generated from a selected policy model $\Phi_{\mathbf{h}_i}$ with a pre-defined behavior mapping  $\mathcal{F}_{\bm{\psi}}$.
\end{Proposition}

In Proposition \ref{proposition: Behavior Control via policy models Selection}, the behavior space size is controlled by the policy model population size (\ie $|\mathbf{H}|$). However, maintaining a large population of different policy models is data-consuming. Hence, we try to control behaviors via optimizing the selection of behavior mappings:

\begin{Proposition}[Behavior Mapping Optimization]
\label{proposition: Behavior Control via behavior mapping optimization}
When all the policy models are used to generate each behavior, \textit{e}.\textit{g}., $\mu_{\bm{\psi}}=\mathcal{F}_{\bm{\psi}}(\Phi_{\bm{\theta}, \mathbf{h}})$ for single policy model cases or $\mu_{\bm{\psi}}=\mathcal{F}_{\bm{\psi}}(\Phi_{\bm{\theta}_1, \mathbf{h}_1},...,\Phi_{\bm{\theta}_\mathrm{N}, \mathbf{h}_\mathrm{N}})$ for $\mathrm{N}$ policy models cases, each behavior can be uniquely indexed by $\mathcal{F}_{\bm{\psi}}$, and \eqref{equ: reward-diversity trade-off problem} can be simplified into:
\begin{equation}
\label{equ:lbc target}
    \mathcal{L}_{\mathcal{P}} = \mathbb{E}_{\bm{\psi} \sim \mathcal{P}_{\bm{\Psi}}} \left[ V_{\mu_{\bm{\psi}}}^{\mathrm{TV}} + c\cdot V_{\mu_{\bm{\psi}}}^{\mathrm{TD}} \right],  
\end{equation}
where $\mathcal{P}_{\bm{\Psi}}$ is a selection distribution of $\bm{\psi} \in \bm{\Psi}$. 
\end{Proposition}
In Proposition \ref{proposition: Behavior Control via behavior mapping optimization}, the behavior space is majorly controlled by $|\Psi|$, which could be a continuous parameter space. Hence, a larger behavior space can be enabled.

\paragraph{Selection Distribution Optimization}  
Given the optimization target $\mathcal{L}_\mathcal{P}$, we seek to find the optimal behavior selection distribution $\mathcal{P}^*_{\mu}$ that maximizes $\mathcal{L}_\mathcal{P}$:
\begin{equation}
\begin{aligned}
        \mathcal{P}^*_{\mathbf{M}_{\mathbf{H},\bm{\Psi}}} := \underset{\mathcal{P}_{\mathbf{M}_{\mathbf{H},\bm{\Psi}}}}{\operatorname{argmax}}  \mathcal{L}_\mathcal{P} &\overset{(1)}{=}  \underset{\mathcal{P}_{\mathbf{H}}}{\operatorname{argmax}} \mathcal{L}_\mathcal{P}  \\
        & \overset{(2)}{=} \underset{\mathcal{P}_{\bm{\Psi}}}{\operatorname{argmax}}  \mathcal{L}_\mathcal{P},
\end{aligned}
\end{equation}
where (1) and (2) hold because we have Proposition~\ref{proposition: Behavior Control via policy models Selection} and \ref{proposition: Behavior Control via behavior mapping optimization}, respectively. This optimization problem can be solved with existing optimizers, \textit{e}.\textit{g}., evolutionary algorithm~\citep{PBT}, multi-arm bandits (MAB)~\citep{agent57}, etc.

\begin{comment}
    这并不能体现出我们的区别，因为这个肯定不止我们做过。所以我们的区别是mixture
\end{comment}

\section{LBC-BM: A Boltzmann Mixture based Implementation for LBC}
\label{sec: Learnable Behavior Control for Boltzmann Strategy}

In this section, we provide an example of improving the behavior control of off-policy actor-critic methods~\citep{impala} via optimizing the behavior mappings as Proposition \ref{proposition: Behavior Control via behavior mapping optimization}. We provide a practical design of hybrid behavior mapping, inducing an implementation of LBC, which we call \textbf{B}oltzmann \textbf{M}ixture based \textbf{LBC}, namely LBC-$\mathcal{BM}$. By choosing different $\mathbf{H}$ and $\bm{\Psi}$, we can obtain a family of implementations of LBC-$\mathcal{BM}$ with different behavior spaces (see Sec. \ref{sec: Ablation Study}).

\subsection{Boltzmann Mixture Based Behavior Space Construction}
\label{sec: Boltzmann Based Behavior Space Construction}

In this section, we provide a general hybrid behavior mapping design including three sub-processes:

% we provide a novel hybrid behavior mapping that can be implemented as the following three sub-processes:

\paragraph{Generalized Policy Selection} In Agent57, behavior control is achieved by selecting a single policy from the policy population at each iteration. Following this idea, we generalize the method to the case where multiple policies can be selected. More specifically, we introduce a importance weights vector $\bm{\omega}$ to describe how much each policy will contribute to the generated behavior, $\bm{\omega}=[\omega_1,...,\omega_{\mathrm{N}}],\omega_i\geq 0, \sum_{i=1}^{\mathrm{N}}\omega_i=1$, where $\omega_i$ represents the importance of $i$th policy in the population (\ie $\Phi_{\mathbf{h}_i}$). In particular, if $\bm{\omega}$ is a one-hot vector, \ie $\exists i\in\{1,2,...,\mathrm{N}\},\omega_i=1;\forall j\in\{1,2,...,\mathrm{N}\}\neq i,\omega_j=0$, then the policy selection becomes a single policy selection as Proposition \ref{proposition: Behavior Control via policy models Selection}. Therefore, it can be seen as a generalization of single policy selection, and we call this process \emph{generalized policy selection}.

\paragraph{Policy-Wise Entropy Control} In our work, we propose to use entropy control (which is typically rule-based controlled in previous works) to make a better trade-off between exploration and exploitation. For a policy model $\Phi_{\mathbf{h}_i}$ from the population, we will apply a entropy control function $f_{\tau_i}(\cdot)$, \ie $\pi_{\mathbf{h}_i,\tau_i}=f_{\tau_i}(\Phi_{\mathbf{h}_i})$, where $\pi_{\mathbf{h}_i,\tau_i}$ is the new policy after entropy control, and $f_{\tau_i}(\cdot)$ is parameterized by $\tau_i$.  Here we should note that the entropy of all the policies from the population is controlled in a policy-wise manner. Thus there would be a set of entropy control functions to be considered, which is parameterized by $\bm{\tau}=[\tau_1,...,\tau_{\mathrm{N}}]$. 
% For Atari task, we model the policy as a Boltzmann distribution, \ie $\pi_{\mathbf{h}_i,\tau_i} (a|s)=e^{\tau_i\Phi_{\mathbf{h}_i}(a|s)}\sum_{a'} e^{\tau_i\Phi_{\mathbf{h}_i}(a'|s)} $, where $\tau_i\in(0,\infty)$. The trade-off can thus be controlled by controlling the temperature (\eg $\tau_i=0$ for full exploitation, $\tau_i=\infty$ for full exploration).

\paragraph{Behavior Distillation from Multiple Policies}
Different from previous methods where only one policy is used to generate the behavior, in our approach, we combine N policies $[\pi_{\mathbf{h}_1,\tau_1},...,\pi_{\mathbf{h}_{\mathrm{N}},\tau_{\mathrm{N}}}]$, together with their importance weights $\bm{\omega}=[\omega_1,...,\omega_{\mathrm{N}}]$.  Specially, in order to make full use of these policies according to their importance, we introduce a \emph{behavior distillation function} $g$ which takes both the policies and importance weights as input, \ie
$\mu_{\mathbf{H}, \bm{\tau},\bm{\omega}}=g(\pi_{\mathbf{h}_1,\tau_1},...,\pi_{\mathbf{h}_{\mathrm{N}},\tau_{\mathrm{N}}}, \bm{\omega})$. The distillation function $g(\cdot,\bm{\omega})$ can be implemented in different ways, \eg knowledge distillation (supervised learning), parameters fusion, etc. In conclusion, the behavior space can be constructed as,
\begin{equation}
\label{equ: enlarged bs}
\mathbf{M}_{\mathbf{H},\bm{\Psi}} = \left\{ g\left( f_{\tau_1}(\Phi_{\mathbf{h}_1}),...,f_{\tau_{\mathrm{N}}}(\Phi_{\mathbf{h}_{\mathrm{N}}}),\omega_1,...,\omega_{\mathrm{N}} \right) | \bm{\psi}\in\bm{\Psi} \right\}
\end{equation}
wherein  $\bm{\Psi}=\{\bm{\psi}=(\tau_1,...,\tau_{\mathrm{N}},\omega_1,...,\omega_{\mathrm{N}})\}$,  $\mathbf{H}=\{\mathbf{h}_1,...,\mathbf{h}_\mathrm{N}\}$. Note that this is a general approach which can be applied to different tasks and algorithms by simply selecting different entropy control function $f_{\tau_i}(\cdot)$ and behavior distillation function $g(\cdot, \bm{\omega})$. 
As an example, for Atari task, we model the policy as a Boltzmann distribution, \ie $\pi_{\mathbf{h}_i,\tau_i} (a|s)=e^{\tau_i\Phi_{\mathbf{h}_i}(a|s)}\sum_{a'} e^{\tau_i\Phi_{\mathbf{h}_i}(a'|s)} $, where $\tau_i\in(0,\infty)$. The entropy can thus be controlled by controlling the temperature. As for the behavior distillation function, we are inspired by the behavior design of GDI, which takes a weighted sum of two softmax distributions derived from two advantage functions.  
  We can further extend this approach to the case to do a combination across different policies, \ie $\mu_{\mathbf{H}, \bm{\tau},\bm{\omega}} (a|s) = \sum_{i=1}^{\mathrm{N}}\omega_i  \pi_{\mathbf{h}_i,\tau_i}(a|s)$. This formula is actually a form of \emph{mixture model}, where the importance weights play the role of mixture weights of the mixture model. Then the behavior space becomes,

\begin{equation}
\label{equ: behavior space with hybrid behavior mapping}
\setlength{\abovedisplayskip}{0.5pt}
\setlength{\belowdisplayskip}{0.5pt}
    \mathbf{M}_{\mathbf{H},\bm{\Psi}} = \{ \mu_{\mathbf{H}, \bm{\psi}}  = \sum_{i=1}^{\mathbf{N}} \omega_i \operatorname{softmax}_{\tau_i}(\Phi_{\mathbf{h}_i}) | \bm{\psi}\in\bm{\Psi} \}
\end{equation}

% extend the approach used in GDI which combines two value functions together with a weighted sum. 
% where $A_1^\pi$ and $A_2^\pi$ are two advantage functions learned from a same policy $\pi$ with different reward scales
 
\subsection{MAB Based Behavior Selection}
\label{sec: MAB Behavior Selection}

According to Proposition \ref{proposition: Behavior Control via behavior mapping optimization}, the behavior selection over behavior space \ref{equ: behavior space with hybrid behavior mapping} can be simplified to the selection of $\bm{\psi}$. In this paper,  we use MAB-based meta-controller to select $\bm{\psi} \in \bm{\Psi}$. Since $\bm{\Psi}$ is a continuous multidimensional space, we discretize $\bm{\Psi}$ into $\mathrm{K}$ regions $\{\Psi_1,...,\Psi_\mathrm{K}\}$, and each region corresponds to an arm of MAB. At the beginning of a trajectory $i$, $l$-th actor will use MAB to sample a  region $\Psi_k$ indexed by arm k according to $\mathcal{P}_{\Psi}=\operatorname{softmax}(\operatorname{Score}_{\Psi_k})=\frac{e^{\operatorname{Score}_{\Psi_k}}}{\sum_j e^{\operatorname{Score}_{\Psi_j}}}$.
We adopt UCB score as  $\operatorname{Score}_{\Psi_k}  = V_{\Psi_k} + c \cdot \sqrt{\frac{\log (1 + \sum_{j \neq k}^{\mathrm{K}} N_{\Psi_j})}{1 + N_{\Psi_k}}}$ to tackle the reward-diversity trade-off problem in  \eqref{equ:lbc target} ~\citep{garivier2008upper}.  $N_{\Psi_k}$ means the number of the visit of $\Psi_k$ indexed by arm $k$.  $V_{\Psi_k}$ is calculated by the expectation of the undiscounted episodic returns to measure the value of each $\Psi_k$, and  the UCB item  is used to avoid selecting the same arm repeatedly and ensure sufficient diverse behavior mappings can be selected to boost the behavior diversity.  After an $\Psi_k$ is sampled, a  $\bm{\psi}$ will be uniformly sampled from  $\Psi_k$, corresponding to a behavior mapping $\mathcal{F}_{\bm{\psi}}$. With $\mathcal{F}_{\bm{\psi}}$, we can obtain a behavior $\mu_{\bm{\psi}}$ according to  \eqref{equ: behavior space with hybrid behavior mapping}. Then, the $l$-th actor acts $\mu_{\bm{\psi}}$ to obtain a trajectory $\tau_i$ and the undiscounted episodic return $G_{i}$, then $G_{i}$ is used to update the reward model $V_{\Psi_k}$ of region ${\Psi_k}$ indexed by arm $k$.  As for the nonstationary problem, we are inspired from GDI, which ensembles several MAB with different learning rates and discretization accuracy. We can extend to handle the nonstationary problem by jointly training a population of bandits from very exploratory to purely exploitative (i.e., different c of the UCB item, similar to the policy population of Agent57). Moreover, we will periodically replace the members of the MAB population to ease the nonstationary problem further. More details of implementations of  MAB can be found in App. \ref{Sec: appendix MAB}. Moreover, the mechanism of the UCB item for behavior control has not been widely studied in prior works, and we will demonstrate how it boosts behavior diversity in App. \ref{app: TSNE Analysis}.

\begin{comment}
1. 实验设置明确
2. 实验表述明确
3. 实验和Intro相互对照
4. 图片设置明确
5. 总结明确
\end{comment}

\section{Experiment}
\label{sec: experiment}
In this section, we design our experiment to answer the following questions:
\begin{itemize}
    \item Whether our methods can outperform prior SOTA RL algorithms in both sample efficiency and final performance in Atari 1B Benchmarks (See Sec. \ref{sec: Summary of Results} and Figs. \ref{fig:atari_results})?
     \item Can our methods adaptively adjust the exploration-exploration trade-off (See Figs. \ref{fig:entropy control})?
    \item  How to enlarge or narrow down the behavior space? What is the performance of methods with different behavior spaces (See Sec. \ref{sec: Ablation Study})? 
    \item  How much performance will be degraded without proper behavior selection (See Figs. \ref{fig:ablation study})?
\end{itemize}

\subsection{Experimental Setup}
\label{sec: Practical Implement Based on LBC}

\subsubsection{Experimental Details}
 We conduct our experiments in  ALE ~\citep{ale}. The standard pre-processing settings of Atari are identical to those of Agent57~\citep{agent57}, and related parameters have been concluded in App. \ref{Sec: appendix hyper-parameters}.  We employ a separate evaluation process to record scores continuously. We record the undiscounted episodic returns averaged over five seeds using a windowed mean over 32 episodes. To avoid any issues that aggregated metrics may have, App. \ref{appendix: experiment results} provides full learning curves for all games and detailed comparison tables of raw and normalized scores. Apart from the mean and median HNS, we also report how many human worlds records our agents have broken to emphasize the superhuman performance of our methods. For more experimental details, see App. \ref{sec:app Experiment Details}. 

\subsubsection{Implementation Details}

We jointly train three polices, and each policy can be indexed by the hyper-parameters $\mathbf{h}_i=(\gamma_i,\mathcal{RS}_i)$, wherein $\mathcal{RS}_i$ is a reward shaping method~\citep{agent57}, and $\gamma_i$ is the discounted factor. Each policy model $\Phi_{\mathbf{h}_i}$ adopts the dueling network structure~\citep{dueling_q}, where $\Phi_{\mathbf{h}_i} = A_{\mathbf{h}_i} = Q_{\mathbf{h}_i} - V_{\mathbf{h}_i}$. More details of the network structure can be found in App. \ref{app: Model Architecture}. To correct for  harmful discrepancy of off-policy learning, we adopt V-Trace~\citep{impala} and  ReTrace ~\citep{retrace} to learn $V_{\mathbf{h}_i}$ and $Q_{\mathbf{h}_i}$, respectively. The policy is learned by policy gradient~\citep{ppo}. Based on \eqref{equ: behavior space with hybrid behavior mapping}, we could build a behavior space with a hybrid mapping as $\mathbf{M}_{\mathbf{H},\bm{\Psi}}= \{ \mu_{\mathbf{H},\bm{\psi}} = \sum_{i=1}^{3} \omega_i \operatorname{softmax}_{\tau_i}(\Phi_{\mathbf{h}_i})\}$, wherein  $\mathbf{H}=\{\mathbf{h}_1,\mathbf{h}_2,\mathbf{h}_3\}$, $\bm{\Psi}=\{\bm{\psi}=(\tau_1,\omega_1,\tau_2,\omega_2,\tau_3,\omega_3)|\tau_i \in (0, \tau^{+}), \sum_{j=1}^{3} \omega_j=1 \}$. The behavior selection is achieved by MAB described in Sec. \ref{sec: MAB Behavior Selection}, and more details can see App. \ref{Sec: appendix MAB}. Finally, we could obtain an implementation of LBC-$\mathcal{BM}$, which is our \textbf{main algorithm}. The target policy for $A_1^\pi$ and $A_2^\pi$ in GDI-H$^3$ is the same,  while in our work the target policy for $A_i^{\pi_i}$ is $\pi_i=\operatorname{softmax}(A_i)$.

\subsection{Summary of Results}
\label{sec: Summary of Results}

\begin{figure}[!t]
	\centering
	\vspace{-0.3in}
	\includegraphics[width=0.8\linewidth]{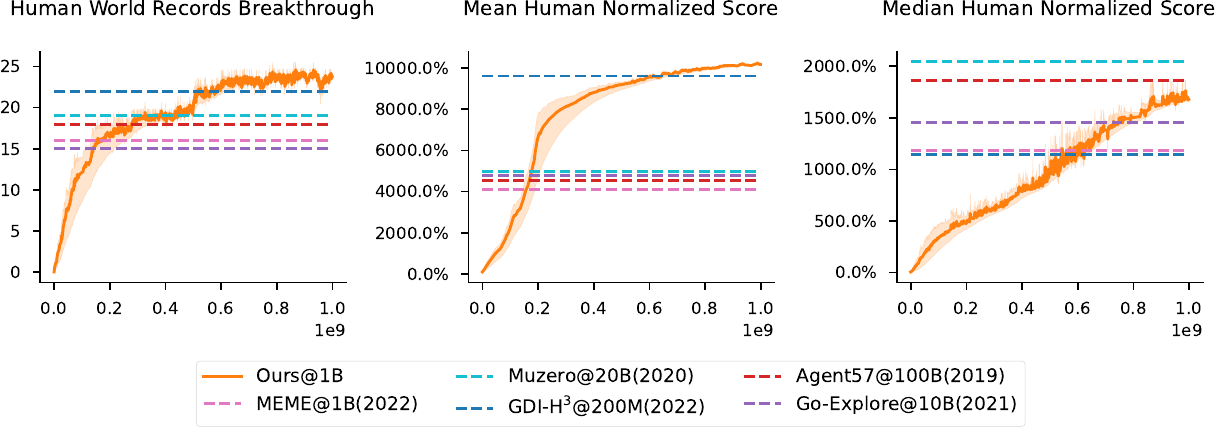}
	\centering
	\vspace{-0.1in}
	\caption{The learning curves in Atari. Curves are smoothed
with a moving average over 5 points. 
	}
	\label{fig:atari_results}
\end{figure}

\begin{figure}[!t]
\centering
	\vspace{-0.15in}
	\includegraphics[width=0.85\linewidth]{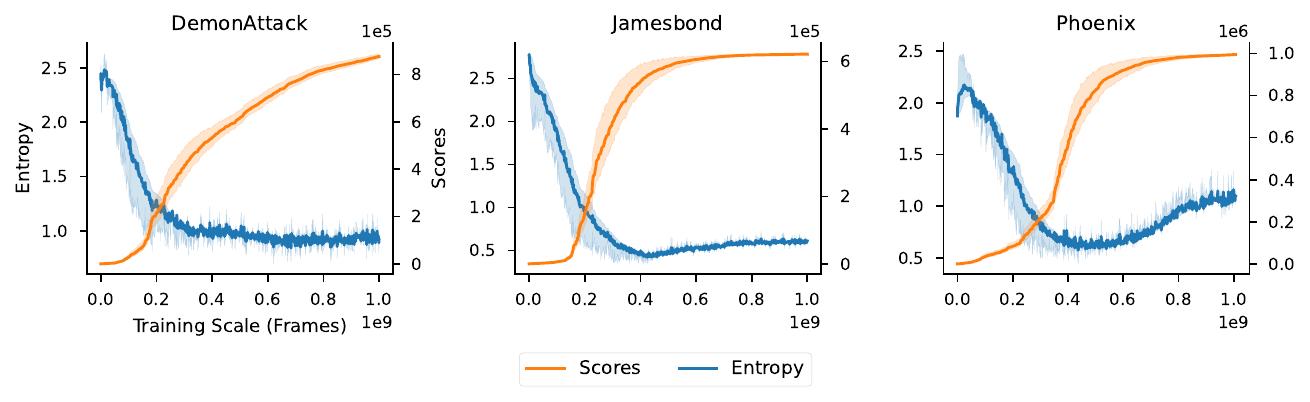}
	\centering
	\vspace{-0.2in}
\caption{Behavior entropy and scores curve across training
for different games where we achieved unprecedented performance. The names of the axes are the same as that of the leftmost figure.}
\label{fig:entropy control}
\vspace{-0.15in}
\end{figure}

% The aggregated results across games are reported in Figs. \ref{fig:atari_results}. From Figs. \ref{fig:atari_results}, EfficientZero~\citep{efficient_zero} achieves remarkable learning efficiency in smaller data volumes but fails to obtain comparable final performance as other SOTA methods like  Agent57 or  Muzero. Among the algorithms with comparable final performance, our agents achieve the best mean HNS and surpass the most human world records across 57 games of the Atari benchmark with minimal training frames, leading to the best learning efficiency. Noting that Agent57 reported the maximum scores across training as the final score, and if we report our performance in the same manner, our median is  1934\%, which is \textbf{higher} than Agent57 and demonstrates our SOTA performance.

\paragraph{Results on Atari Benchmark} The aggregated results across games are reported in Figs. \ref{fig:atari_results}. Among the algorithms with superb final performance, our agents achieve the best mean HNS and surpass the most human world records across 57 games of the Atari benchmark with relatively minimal training frames, leading to the best learning efficiency. Noting that Agent57 reported the maximum scores across training as the final score, and if we report our performance in the same manner, our median is  1934\%, which is \textbf{higher} than Agent57 and demonstrates our superior performance.

% On 200M training frames, GDI-H$^3$ obtains
% the highest mean HNS, median HNS and HWRB, which demonstrates its SOTA performance in the 200M Atari Benchmarks.  Among the existing algorithms using more data than 1B,  our agents achieve the best mean HNS and surpass the most human world records, which implies the effectiveness of our method.

 \paragraph{Discussion of  Results} With LBC, we can understand the mechanisms underlying the performance of GDI-H$^3$ more clearly: \textbf{i)} GDI-H$^3$ has a high-capacity behavior space and a meta-controller to optimize the behavior selection \textbf{ii)} only a single  target policy is learned, which enables stable learning and fast converge (See the case study of KL divergence in App. \ref{sec: KL Divergence Between policy models}). Compared to GDI-H$^3$, to ensure the behavior space will not degenerate, LBC maintains a population of diverse policies and, as a price, sacrifices some sample efficiency. Nevertheless, LBC can \textbf{continuously} maintain a significantly larger behavior space with hybrid behavior mapping, which enables RL agents to continuously explore and get improvement.
 
\subsection{Case Study: Behavior Control}   To further explore the mechanisms underlying the success of behavior control of  our method, we adopt a case study to showcase our control process of behaviors. As shown in Figs. \ref{fig:entropy control}, in most tasks, our agents prefer exploratory behaviors first (i.e., high stochasticity policies with high entropy), and, as training progresses, the agents shift into producing experience from more exploitative behaviors. On the verge of peaking, the entropy of the behaviors could be maintained at a certain level (task-wise)  instead of collapsing swiftly to zero to avoid converging prematurely to sub-optimal policies.

\subsection{Ablation Study}
\label{sec: Ablation Study}

In this section, we investigate several properties of our method. For more details, see App. \ref{app: Ablation Study}.

\paragraph{Behavior Space Decomposition} To explore the effect of different behavior spaces, we decompose the behavior space of our main algorithm via reducing  $\mathbf{H}$ and $\bm{\Psi}$: 

\textbf{1) Reducing $\mathbf{H}$.} When we set all the policy models of our main algorithm the same, the behavior space transforms from $\mathcal{F}(\Phi_{\mathbf{h}_1},\Phi_{\mathbf{h}_2},\Phi_{\mathbf{h}_3})$ into  $\mathcal{F}(\Phi_{\mathbf{h}_1},\Phi_{\mathbf{h}_1},\Phi_{\mathbf{h}_1})$. $\mathbf{H}$ degenerates from $\{\mathbf{h}_1,\mathbf{h}_2,\mathbf{h}_3\}$ into $\{\mathbf{h}_1\}$.  We can obtain a control group with a smaller behavior space by reducing $\mathbf{H}$. 
% (belongs to LBC-H$^\text{2}_\text{1}$-$\bm{\Psi}^\text{6}_{\infty}$)

% an LBC-based method LBC-$\mathcal{B}$-H$^\text{2}_\text{1}$-$\bm{\Psi}^\text{6}_{\infty}$, whose behavior space is a sub-space of our main algorithm. 

 \textbf{2)Reducing  $\mathbf{H}$ and $\bm{\Psi}$.} Based on the control group reducing $\mathbf{H}$, we can further reduce $\bm{\Psi}$ to further narrow down the behavior space. Specially,  we can directly adopt a individual behavior mapping to build the behavior space as $\mathbf{M}_{\mathbf{H},\bm{\Psi}} = \{\mu_{\psi}=\operatorname{softmax}_{\tau}(\Phi_{\mathbf{h}_1})\}$, where $\bm{\Psi}$ degenerates from $\{\bm{\omega}_1,\bm{\omega}_2,\bm{\omega}_3,\tau_1,\tau_2,\tau_3\}$ to $\{\tau\}$ and $\mathbf{H}=\{\mathbf{h}_1\}$. Then, we can obtain a control group with the smallest behavior space by reducing $\mathbf{H}$ and $\bm{\Psi}$.
 
The performance of these methods is illustrated in Figs. \ref{fig:ablation study}, and from left to right, the behavior space of the first three algorithms decreases in turn (According to Corollary \ref{Corollary: space larger} in App. \ref{app: Proof}). It is evident that narrowing the behavior space via reducing  $\mathbf{H}$ or $\bm{\Psi}$ will degrade the performance. On the contrary, the performance can be boosted by enlarging the behavior space, which could be a promising way to improve the performance of existing methods.

\begin{figure*}[!t]
	\centering
	\vspace{-0.3in}
 \includegraphics[width=0.85\textwidth]{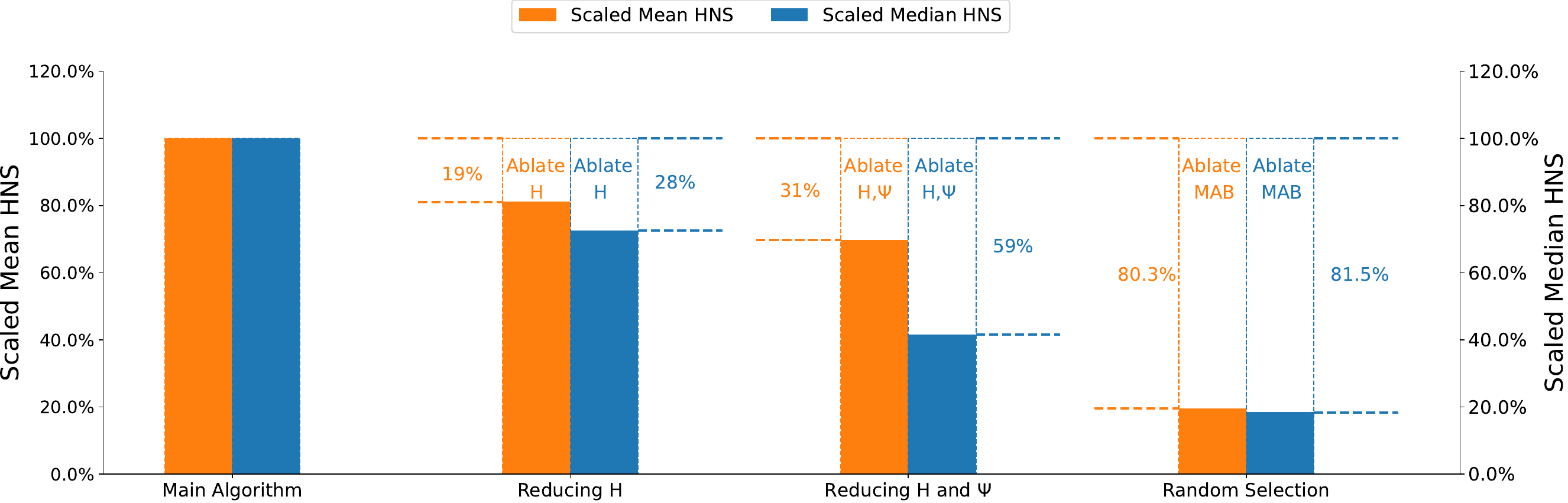}
	\centering
	\vspace{-0.15in}
\caption{Ablation Results. All the results are scaled by the main algorithm to improve readability.}
\label{fig:ablation study}
\vspace{-0.2in}
\end{figure*}

\paragraph{Behavior Selection}  To highlight the importance of an appropriate behavior selection, we replace the meta-controller of our main algorithm with a random selection. The ablation results are illustrated in Figs. \ref{fig:ablation study}, from which it is evident that, with the same behavior space, not learning an appropriate selection distribution of behaviors will significantly degrade the performance. We conduct a t-SNE analysis in App. \ref{app: TSNE Analysis} to demonstrate that our methods can acquire more diverse behaviors than the control group with pre-defined behavior mapping. Another ablation study that removed the UCB item has been conducted in App. \ref{app: TSNE Analysis} to demonstrate the behavior diversity may be boosted by the UCB item, which can encourage the agents to select more different behavior mappings.

\begin{comment}
    t-SNE  analysis 如何放到正文中，非常111关键。如何把tsne放到正文里面。
\end{comment}

\section{Conclusion}
\label{sec: conclusion}
We present the first deep reinforcement learning agent to break 24 human world records in Atari using only 1B training frames. To achieve this, we propose a general framework called LBC, which enables a significantly enlarged behavior selection space via formulating a hybrid behavior mapping from all policies, and constructs a unified  learnable process for behavior selection. We introduced LBC into off-policy actor-critic methods and obtained a family of implementations. A large number of experiments on Atari have been conducted to demonstrate the effectiveness of our methods empirically. Apart from the full results, we do detailed ablation studies to examine the effectiveness of the proposed components. While there are many improvements and extensions to be explored going forward, we believe that the ability of LBC to enhance the control process of behaviors results in a powerful platform to propel future research.
% We present the first deep reinforcement learning agent to break 24 human world records in Atari using only 1B training frames. To achieve this, we propose a general framework called LBC, which enables a significantly enlarged behavior selection space via formulating a hybrid behavior mapping from all policies, and constructs a unified  learnable process for behavior selection. We introduced LBC into off-policy actor-critic methods and obtained a family of implementations. A large number of experiments on Atari have been conducted to demonstrate the effectiveness of our methods empirically. Apart from the full results, we do detailed ablation studies to examine the effectiveness of the proposed components. 

% \section*{Acknowledgments}
% This work is majorly supported by the National Key R\&D Program of China (Grant Number 2021ZD0110400). In addition, this work is partly supported by the National Natural Science Foundation of China under Grant 62171248, and the R\&D Program of Shenzhen under Grant JCYJ20220818101012025.

% We are very grateful for the careful reading and insightful reviews of meta-reviewers and reviewers.

% \section*{Reproducibility Statement}
% Open-sourced code will be implemented with Mindspore \citep{MS} and released on our \href{https://gitee.com/mindspore/models/tree/master/research/rl}{website}.

% \begin{comment}

% \end{comment}

\bibliography{ref}
\bibliographystyle{iclr2023_conference}

\newpage

\appendix
\input{appendix}

\end{document}

%% file: math_commands.tex
\usepackage{amsmath,amsfonts,bm}

\def\eqref#1{equation~\ref{#1}}
\def\1{\bm{1}}

\DeclareMathAlphabet{\mathsfit}{\encodingdefault}{\sfdefault}{m}{sl}
\SetMathAlphabet{\mathsfit}{bold}{\encodingdefault}{\sfdefault}{bx}{n}

%% file: appendix.tex
\begin{center}
    \Large \textbf{Appendix}  
\end{center}
\paragraph{Overview of the Appendix} For ease of reading, we submit the main body of our paper together with the appendix as supplemental material. Below we will briefly introduce the structure of our appendix for easy reading. Readers can jump from the corresponding position in the body to the corresponding content in the appendix, or jump from the following table of contents to the corresponding content of interest in the appendix.

\begin{itemize}
    \item In App. \ref{app: Abbreviation and Notation}, we briefly summarize some common notations  in this paper for the convenience of readers.
    \item In App. \ref{app: background on RL},  we  summarize and recall the background knowledge used in this paper.
    \item In App. \ref{app: Proof},  we provide theoretical proofs.

    \item In App. \ref{Sec: appendix behavior space},  we provide several implementations of LBC-based behavior sapce design to facilitate future research.

    \item In App. \ref{Sec: appendix MAB}, we provide a detailed implementation of our MAB-based meta-controller.
    \item In App. \ref{app: Algorithm Pseudocode}, we provide a detailed implementation of our core framework.
    \item In App. \ref{app: An LBC-based Version of RL}, we use our framework to introduce an LBC-based version of well-known RL methods, which leads to a better understanding of their original counterparts. 
    \item In App. \ref{sec:app Experiment Details}, we provide relevant details of our experiments.
    \item In App. \ref{Sec: appendix hyper-parameters},  we summarize the hyper-parameters used in our experiments.
    
    \item In App. \ref{appendix: experiment results}, we provides full learning curves for all games and detailed comparison tables of raw and normalized scores.
    
    \item In App. \ref{app: Ablation Study}, we provides the design of the ablation study and the overall results.

    \item In App. \ref{app: Model Architecture}, we  summarize our  model architecture in detail to for reproducibility.

\end{itemize}

\clearpage
\section{Summary of Notation and Abbreviation}
\label{app: Abbreviation and Notation}
In this section, we briefly summarize some common notations in this paper for the convenience of readers, which are concluded in Tab. \ref{tab: notation}.

\begin{table}[!hb]
	\centering
	\caption{Summary of Notation}
	\label{tab: notation}
 \resizebox{\textwidth}{!}{% <------ Don't forget this %
	\begin{tabular}{l l l l }
	    \toprule
		\textbf{Symbol} &\textbf{Description} & \textbf{Symbol} &\textbf{Description}\\
		\midrule
	    $s$ & State & $\mathcal{S} $ & Set of all states \\
     
	    $a$ & Action & $\mathcal{A} $ & Set of all actions \\
     
	    $\mathbb{P}$  &  Probability distribution & $\mu $ & Behavior policy \\
     
	    $\pi $ & Policy & $G_t $ & Cumulative discounted reward at $t$ \\
     
	    $d_{\rho_0}^{\pi}$  & States visitation distribution of $\pi$ & $V_{\pi}$ & State value function of $\pi$\\
     
	    $Q_{\pi}$ &  State-action value function of $\pi$ & $A_{\pi}$ &  Advantage function of $\pi$\\
     
	    $\gamma$ & Discount-rate parameter & $\delta_{t}$ & Temporal-difference error at $t$\\

        $\mathcal{F}$ & Behavior mapping & $\Phi$ & Policy models \\
        
	    $\bm{\theta}$  & \makecell[l]{Parameters  of the policy network } & $\bm{\bm{\Theta}}$  & Set of  $\theta$ \\
	 
  $\mathbf{h}$ & hyper-parameters of policy models  & $\mathbf{H}$ & Set of $h$  \\
     
	    $\bm{\psi}$ & Parameters to index $\mathcal{F}$ & $\bm{\Psi}$ & Set of  $\bm{\psi}$  \\
    
	    $\mathbf{M}_{\bm{\bm{\Theta}},\mathbf{H},\bm{\Psi}}$ & Behavior policy sets/space parameterized by $\theta,\mathbf{h},\bm{\bm{\psi}}$ 
 & $\mu_{\theta,\mathbf{h},\bm{\bm{\psi}}}$ & A behavior policy of $\mathbf{M}_{\bm{\bm{\Theta}},\mathbf{H},\bm{\Psi}}$\\

        $|\mathbf{H}|$ & Size of set $\mathbf{H}$ & $|\bm{\Psi}|$ &  Size of set $\bm{\Psi}$ \\
	    
	    $\mathcal{P}_{\mathbf{M}_{{\bm{\bm{\Theta}},\mathbf{H},\bm{\Psi}}}}$ & Behavior selection distribution over $\mathbf{M}_{{\bm{\bm{\Theta}},\mathbf{H},\bm{\Psi}}}$& $\mathbbm{1}_{\bm{\bm{\psi}} = \bm{\bm{\psi}}_0}$ & One-point distribution $\textbf{P}(\bm{\bm{\psi}}=\bm{\bm{\psi}}_0)=1$ \\
	    
     $V_{\mu}^{\mathrm{TV}}$ &  Some measurement on the value of policy $\pi$ & $V_{\mu}^{\mathrm{TD}}$ & Some measurement on the  diversity of policy $\pi$ \\

     	      $\mathbf{M}_{\bm{\bm{\Theta}},\mathbf{H}}$ & Subspace of $\mathbf{M}_{\bm{\bm{\Theta}},\mathbf{H},\bm{\Psi}}$  & $\mathcal{RS}$ & Reward shaping\\
	    \bottomrule	    
	\end{tabular} 
 }
\end{table}

% \begin{table}[!hb]
% 	\centering
% 	\caption{Summary of Abbreviation}
% 	\label{tab: abbreviation}
% 	\begin{tabular}{l l}
% 		\toprule
% 		\textbf{Abbreviation} &\textbf{Description}\\
%         \midrule
		
%     Sec.  & Section  \citep{agent57} \\
		
% 		Figs. & Figures  \citep{dreamerv2} \\
		
% 		Fig. & Figure  \citep{agent57} \\
		
% 		Eq.     & Equation  \citep{agent57}\\
		
% 		Tab.    & Table  \citep{agent57} \\
		
% 		App.    & Appendix  \citep{agent57} \\
		
% 		SOTA & State-of-The-Art  \citep{agent57}    \\
		
% 		RL  & Reinforcement Learning  \citep{sutton} \\
		
% 		DRL & Deep Reinforcement Learning  \citep{sutton} \\
        
%         GPI & Generalized Policy Iteration  \citep{sutton} \\
        
%         LBC & Learnable Behavior Control  \\
        
%         PBRL & Population-based Reinforcement Learning  \\
%     \bottomrule    
		
% 	\end{tabular} 
% \end{table}

\clearpage
\section{Background}
\label{app: background on RL}

Similar to deep learning \citep{cl,actionr,6d,dclsr,wang2023dr,wang2022lhnn}, reinforcement learning is also a branch of machine learning.  Most of the previous work has introduced the background knowledge of RL in detail. In this section, we only summarize and recall the background knowledge used in this paper. If you are interested in the relevant content, we recommend you read the relevant material \citep{pi2,sutton}. The relevant notations and abbreviations have been documented in App. \ref{app: Abbreviation and Notation}.

\subsection{Policy Gradient}
\label{app: pg}

Policy gradient methods, denoted as PG \citep{williams1992simple,casa,entropy}, belong to the category of policy-based reinforcement learning approaches. These methods employ an optimization process to update the policy by maximizing the target function:
\begin{equation*}
    \mathcal{J} (\theta)=\mathbb{E}_{\pi}\left[\sum_{t=0}^{\infty} \log \pi_{\theta}\left(a_{t} \mid s_{t}\right) G(\tau)\right]
\end{equation*}

The Actor-Critic (AC) methods compute the policy gradient by updating the AC policy as follows \citep{sutton}:
\begin{equation*}
    \nabla_{\theta} \mathcal{J}(\theta)=\mathbb{E}_{\pi}\left[\sum_{t=0}^{\infty} \bm{\Phi}_{t} \nabla_{\theta} \log \pi_{\theta}\left(a_{t} \mid s_{t}\right)\right]
\end{equation*}
wherein $\bm{\Phi}_{t}$ could be  $Q^{\pi}\left(s_{t}, a_{t}\right)$ or $A^{\pi}\left(s_{t}, a_{t}\right)$. 

\subsection{Vtrace}
\label{app: vtrace}
V-Trace is an off-policy correction technique devised by the IMPALA framework \citep{impala} to address the dissimilarity between the target policy and behavior policy. The estimation of $V(s_t)$ in V-Trace is computed by:
\begin{equation*}
\label{Equ: vtrace}
    V^{\Tilde{\pi}} (s_t) 
        = \mathbb{E}_{\mu} [ 
        V(s_t) + \sum_{k \geq 0} \gamma^k 
     c_{[t:t+k-1]} \rho_{t+k}  \delta^{V}_{t+k} V ],
\end{equation*}
wherein $\delta^{V}_t V \overset{def}{=} r_t + \gamma V(s_{t+1}) - V(s_t)$ and $\rho_t = \min\{\frac{\pi_t}{\mu_t}, \Bar{\rho} \}$. 

\subsection{Retrace}
\label{app: retrace}
ReTrace, a methodology proposed in \citep{retrace}, computes $Q(s_t, a_t)$ by the following expression:
\begin{equation*}
\label{Equ: retrace}
    Q^{\Tilde{\pi}} (s_t, a_t) 
= \mathbb{E}_{\mu} [ Q(s_t, a_t) + \sum_{k \geq 0} \gamma^k 
c_{[t+1:t+k]} \delta^{Q}_{t+k} Q ],
\end{equation*}
where $\delta^{Q}_t Q \overset{def}{=} r_t + \gamma Q(s_{t+1}, a_{t+1}) - Q(s_t, a_t)$. 
\clearpage

\section{Proof}
\label{app: Proof}

\begin{proof}[Proof of \textbf{Proposition} \ref{proposition: Behavior Control via policy models Selection}]
\label{pf: pr 1}
When $\mathcal{F}_{\bm{\psi}}$ is a pre-defined or rule-based mapping, $\mathcal{F}_{\bm{\psi}}$ of actor $j$ at each training step (wall-clock) is deterministic, namely $\mathcal{F}_{\bm{\psi}} = \mathcal{F}_{\bm{\psi}_j}$, so each behavior of actor $j$ can be  uniquely indexed by $\mathbf{h} \in \mathbf{H}=\{\mathbf{h}_1,...,\mathbf{h}_\text{N}\}$, namely,

\begin{equation}
    \mathbf{M}_{\mathbf{H}} = \{\mu_{\mathbf{h}} = \mathcal{F}_{\bm{\psi}_j} (\Phi_\mathbf{h}) | \mathbf{h} \in \mathbf{H}\},
\end{equation}
where $\bm{\psi}_j$ is the same for each behavior of actor $j$. Hence, the selection distribution of behavior can be simplified into the selection distribution over $\mathbf{h} \in \mathbf{H}=\{\mathbf{h}_1,...,\mathbf{h}_N\}$ as:
\begin{equation}
\label{equ: select h}
    \mathcal{P}_{\mu \in \mathbf{M}_{\mathbf{H},\bm{\Psi}}} = \mathcal{P}_{\mu_\mathbf{h} \in \mathbf{M}_{\mathbf{H}}} = \mathcal{P}_\mathbf{H},
\end{equation}
where  $\mathbf{M}_{\mathbf{H}} = \{ \mathcal{F}_{\bm{\psi}_j} (\Phi_\mathbf{h}) | \mathbf{h} \in \mathbf{H}\}$ and $\mathcal{P}_{\mathbf{H}}$ is a selection distribution of $\mathbf{h} \in \mathbf{H}=\{\mathbf{h}_1,...,\mathbf{h}_N\}$ and $\mathrm{N}$ is the number of policy models.
Substituting \eqref{equ: select h} into \eqref{equ: reward-diversity trade-off problem}, we can obtain

\begin{equation*}
\setlength{\abovedisplayskip}{0.5pt}
\setlength{\belowdisplayskip}{0.5pt}
\begin{aligned}
        \mathcal{L}_{\mathcal{P}}
    & = 
    \mathbb{E}_{\mu \sim \mathcal{P}_{\mathbf{M}_{\mathbf{H},\bm{\Psi}}}}  [V^{\mathrm{TV}}_{\mu}+c \cdot V^{\mathrm{TD}}_{\mu}] \\
    &= \mathbb{E}_{\mu_\mathbf{h} \sim \mathcal{P}_{\mathbf{M}_\mathbf{H}}}  [V^{\mathrm{TV}}_{\mu_\mathbf{h}}+c \cdot V^{\mathrm{TD}}_{\mu_\mathbf{h}}] \\
    &= \mathbb{E}_{\mathbf{h} \sim \mathcal{P}_\mathbf{H}}  [V^{\mathrm{TV}}_{\mu_\mathbf{h}}+c \cdot V^{\mathrm{TD}}_{\mu_\mathbf{h}}] \\
\end{aligned}
\end{equation*}

\end{proof}

\begin{Corollary}[Behavior Circling in Policy Model Selection]
\label{Corollary:Behavior Circling in policy model Selection}
    When the policy models overlap as $\Phi_{\mathbf{h}_1} \approx ... \approx \Phi_{\mathbf{h}_\mathrm{N}}$ , all realizable behavior of actor $j$ will overlap as $\mathcal{F}_{\bm{\psi}_j}(\Phi_{\mathbf{h}_1}) \approx ... \approx \mathcal{F}_{\bm{\psi}_j}(\Phi_{\mathbf{h}_\mathrm{N}})$. Behaviors from different model can not be distinguished by $\mathbf{h}_i$, the behavior selection via policy model selection becomes invalid. 
\end{Corollary}

\begin{proof}[Proof of \textbf{Proposition} \ref{proposition: Behavior Control via behavior mapping optimization}]
   \label{pf: pr 2} 

   When $\bm{\theta},\mathbf{h}$ are shared among behaviors for each actor at each training step, such as $\mu=\mathcal{F}_{\bm{\psi}}(\Phi_{\bm{\theta},\mathbf{h}})$ or $\mu=\mathcal{F}_{\bm{\psi}}(\Phi_{\bm{\theta}_1, \mathbf{h}_1},...,\Phi_{\bm{\theta}_\mathrm{N}, \mathbf{h}_\mathrm{N}})$, each behavior for each behavior can be uniquely indexed by $\bm{\psi}$, namely, 
   
\begin{equation*}
    \mathbf{M}_{\bm{\Psi}} = \{ \mu_{\bm{\psi}} = \mathcal{F}_{\psi}(\bm{\Phi}_{1:\mathrm{N}}) | \bm{\psi} \in \bm{\Psi}\}, 
\end{equation*}
where $\bm{\Phi}_{1:\mathrm{N}}$ is the same among behaviors. Hence, the selection distribution of behavior can be simplified into the selection distribution of $\bm{\psi} \in \bm{\Psi}$ as 
\begin{equation}
\label{equ: select psi}
    \mathcal{P}_{\mu \in \mathbf{M}_{\mathbf{H},\bm{\Psi}}} = \mathcal{P}_{\mu_{\bm{\psi}} \in \mathbf{M}_{\bm{\Psi}}} = \mathcal{P}_\mathbf{\Psi},
\end{equation}
where  $\mathcal{P}_{\bm{\Psi}}$ is a selection distribution of $\bm{\psi} \in \bm{\Psi}$. Substituting \eqref{equ: select psi} into \eqref{equ: reward-diversity trade-off problem}, we can obtain

\begin{equation*}
\setlength{\abovedisplayskip}{0.5pt}
\setlength{\belowdisplayskip}{0.5pt}
\begin{aligned}
        \mathcal{L}_{\mathcal{P}}
    & = 
    \mathbb{E}_{\mu \sim \mathcal{P}_{\mathbf{M}_{\mathbf{H},\bm{\Psi}}}}  [V^{\mathrm{TV}}_{\mu}+c \cdot V^{\mathrm{TD}}_{\mu}] \\
    &= \mathbb{E}_{\mu_{\bm{\psi}} \sim \mathcal{P}_{\mathbf{M}_{\bm{\Psi}}}}  [V^{\mathrm{TV}}_{\mu_{\bm{\psi}}}+c \cdot V^{\mathrm{TD}}_{\mu_{\bm{\psi}}}] \\
    &= \mathbb{E}_{\bm{\psi} \sim \mathcal{P}_{\bm{\Psi}}}  [V^{\mathrm{TV}}_{\mu_{\bm{\psi}}}+c \cdot V^{\mathrm{TD}}_{\mu_{\bm{\psi}}}] \\
\end{aligned}
\end{equation*}

\end{proof}

The behavior mapping optimization may be a cure for the behavior circling:
\begin{Corollary}[Behavior Mapping Optimization Is An Antidote for Behavior Circling]
\label{Corollary: behavior mapping optimization is an Antidote for Behavior Circling}
As for an behavior mapping optimization method, the behavior of actor $j$ is indexed by $\mathcal{F}_{\bm{\psi}}$. When all the policy models overlap as $\Phi_{\mathbf{h}_1} \approx ... \approx \Phi_{\mathbf{h}_\mathrm{N}}$, the realizable behavior of actor $j$ are 
\begin{equation}
    \mathcal{F}_{\bm{\psi}_1}(\Phi_{\mathbf{h}_1},...,\Phi_{\mathbf{h}_\mathrm{N}}), \mathcal{F}_{\bm{\psi}_2}(\Phi_{\mathbf{h}_1},...,\Phi_{\mathbf{h}_\mathrm{N}}), ..., \mathcal{F}_{\bm{\psi}_\infty}(\Phi_{\mathbf{h}_1},...,\Phi_{\mathbf{h}_\mathrm{N}}),
\end{equation}
wherein $\bm{\psi}_i$ is a continuous parameter.
Assuming $\mathcal{F}_{\psi}$ can be uniquely indexed by $\psi$ and $\mathcal{F}_{\psi_i} \neq \mathcal{F}_{\psi_j}$, there are still infinite different behaviors that can be realized by actor $j$.
\end{Corollary}

\begin{Proposition}[Comparison of Behavior Space]
\label{Proposition: Comparison of Behavior Space}
Given two behavior space $\mathbf{M}_{\bm{\Theta}_1,\mathbf{H}_1,\bm{\Psi}_1}$ and $\mathbf{M}_{\bm{\Theta}_2,\mathbf{H}_2,\bm{\Psi}_2}$, if $\mathbf{M}_{\bm{\Theta}_1,\mathbf{H}_1,\bm{\Psi}_1}$ is a sub-space of $\mathbf{M}_{\bm{\Theta}_2,\mathbf{H}_2,\bm{\Psi}_2}$, the space $\mathbf{M}_{\bm{\Theta}_2, \mathbf{H}_2,\bm{\Psi}_2}$  is not less than  $\mathbf{M}_{\mathbf{H}_1,\bm{\Psi}_1}$. Furthermore, if $\mathbf{M}_{\bm{\Theta}_1,\mathbf{H}_1,\bm{\Psi}_1}$ is a sub-space of $\mathbf{M}_{\bm{\Theta}_2,\mathbf{H}_2,\bm{\Psi}_2}$ and $\mathbf{M}_{\bm{\Theta}_1,\mathbf{H}_1,\bm{\Psi}_1}$ is not equal to $\mathbf{M}_{\bm{\Theta}_2,\mathbf{H}_2,\bm{\Psi}_2}$, the space $\mathbf{M}_{\bm{\Theta}_2,\mathbf{H}_2,\bm{\Psi}_2}$  is larger than  $\mathbf{M}_{\bm{\Theta}_1, \mathbf{H}_1,\bm{\Psi}_1}$.
\end{Proposition}

\begin{proof}[Proof of \textbf{Proposition} \ref{Proposition: Comparison of Behavior Space}]
    Since $\mathbf{M}_{\bm{\Theta}_1,\mathbf{H}_1,\bm{\Psi}_1}$ and $\mathbf{M}_{\bm{\Theta}_2,\mathbf{H}_2,\bm{\Psi}_2}$ are sets. When $\mathbf{M}_{\bm{\Theta}_1,\mathbf{H}_1,\bm{\Psi}_1} \subseteq \mathbf{M}_{\bm{\Theta}_2,\mathbf{H}_2,\bm{\Psi}_2}$, $\mathbf{M}_{\bm{\Theta}_1,\mathbf{H}_1,\bm{\Psi}_1}$ is not larger than $ \mathbf{M}_{\bm{\Theta}_2,\mathbf{H}_2,\bm{\Psi}_2}$. When $\mathbf{M}_{\bm{\Theta}_1,\mathbf{H}_1,\bm{\Psi}_1} \subset \mathbf{M}_{\bm{\Theta}_2,\mathbf{H}_2,\bm{\Psi}_2}$, $\mathbf{M}_{\bm{\Theta}_1,\mathbf{H}_1,\bm{\Psi}_1}$ is smaller than $\mathbf{M}_{\bm{\Theta}_2,\mathbf{H}_2,\bm{\Psi}_2}$
\end{proof}

According to the behavior space construction formulation, we can draw the following Corollary:
\begin{Corollary}
\label{Corollary: space not smaller}
        Given the same policy model structure $\Phi$ and the same form of behavior mapping  $\mathcal{F}$. Under Assumption \ref{ass: share model}, the behavior space can be fully determined by $\textbf{H}$ and $\Psi$.  For any two behavior space $\mathbf{M}_{\mathbf{H}_1,\bm{\Psi}_1}$ and $\mathbf{M}_{\mathbf{H}_2,\bm{\Psi}_2}$, if $\mathbf{H}_1 \subseteq \mathbf{H}_2$ and $\bm{\Psi}_1 \subseteq \bm{\Psi}_2$, the behavior space $\mathbf{M}_{\mathbf{H}_1,\bm{\Psi}_1}$ is a sub-space of $\mathbf{M}_{\mathbf{H}_2,\bm{\Psi}_2}$. Based on that, the space $\mathbf{M}_{\mathbf{H}_2,\bm{\Psi}_2}$  is not smaller than  $\mathbf{M}_{\mathbf{H}_1,\bm{\Psi}_1}$.
\end{Corollary}

\begin{Corollary}
\label{Corollary: space larger}
        Given the same policy model structure $\Phi$ and the same form of behavior mapping  $\mathcal{F}$. Under Assumption \ref{ass: share model}, the behavior space can be fully determined by $\textbf{H}$ and $\Psi$.  For any two behavior space $\mathbf{M}_{\mathbf{H}_1,\bm{\Psi}_1}$ and $\mathbf{M}_{\mathbf{H}_2,\bm{\Psi}_2}$, if at least one of the following conditions holds: 
        \begin{itemize}
            \item $\mathbf{H}_1 \subset \mathbf{H}_2$ and $\bm{\Psi}_1 \subseteq \bm{\Psi}_2$, 
            \item $\mathbf{H}_1 \subseteq \mathbf{H}_2$ and $\bm{\Psi}_1 \subset \bm{\Psi}_2$ ,
            \item $\mathbf{H}_1 \subset \mathbf{H}_2$ and $\bm{\Psi}_1 \subset \bm{\Psi}_2$ ,
        \end{itemize}
        the behavior space $\mathbf{M}_{\mathbf{H}_1,\bm{\Psi}_1}$ is a sub-space of $\mathbf{M}_{\mathbf{H}_2,\bm{\Psi}_2}$ and $\mathbf{M}_{\mathbf{H}_1,\bm{\Psi}_1}$ is not equal to $\mathbf{M}_{\mathbf{H}_2,\bm{\Psi}_2}$. Based on that, the space $\mathbf{M}_{\mathbf{H}_2,\bm{\Psi}_2}$  is larger than  $\mathbf{M}_{\mathbf{H}_1,\bm{\Psi}_1}$.
\end{Corollary}

\clearpage
% \section{Running Examples of Behavior Control}
% \label{app: A Running Example of Behavior Control}

% \subsection{Policy Model Selection}
% We take Agent57 as the running example to help to understand the technique on Proposition \ref{proposition: Behavior Control via policy models Selection}. Agent57  adopted a meta-controller  to select a model $\theta_{(\beta_i,\gamma_i)}$ learned with hyper-parameters $\mathbf{h}_i=(\beta_i,\gamma_i)$  and then used these selected model to generate  $\epsilon_l$-greedy behaviors, where $\bm{\bm{\psi}}_i=(\epsilon_i)$ is rule-based (\textbf{unlearned}). Therefore, $\mathbf{H} = \{ (\beta_i,\gamma_i)| i=1,\dots,32\}$ and $\bm{\Psi} = \{(\epsilon_l)| l=1,...,\text{L}\}$.

% \subsection{Behavior Mapping Optimization}

% We take GDI-H$^3$ \citep{gdi} as the running example to help to understand the behavioral control technique on Proposition \ref{proposition: Behavior Control via behavior mapping optimization}. GDI-H$^3$ learns two different models indexed by the hyperparameter $\mathbf{h}_i=(\mathcal{RS}_i)$, so $\mathbf{H}=\{\mathcal{RS}_1,\mathcal{RS}_2\}$. GDI-H$^3$ adopted a meta-controller  to select a hybrid behavior mapping for each episode, which maps from two policy models into a behavior as $\pi=\epsilon \cdot \operatorname{Softmax}\left(\frac{A_{1}}{\tau_1}\right)+(1-\epsilon) \cdot \operatorname{Softmax}\left(\frac{A_{2}}{\tau_2}\right)$. Therefore, $\bm{\bm{\psi}}=(\tau_1,\tau_2,\epsilon)$.  The meta-controller is learned via optimizing the undiscounted cumulative returns.

% \clearpage
\section{Behavior Space Construction For More Tasks and Algorithms Via LBC}
\label{Sec: appendix behavior space}
Following the pipeline given in \eqref{equ: enlarged bs}, different implementations of LBC can be acquired by simply selecting different entropy control function $f_{\tau_i}(\cdot)$ and behavior distillation function $g(\cdot, \bm{\omega})$ according to the corresponding RL algorithms and tasks. 

\subsection{Selection for entropy control function}
Here we would give some examples for the selection of entropy control function $f_{\tau_i}(\cdot)$.

\paragraph{Continuous Control Tasks} For tasks with continuous action spaces, the entropy control function can be selected as gaussian distribution, \ie
\begin{equation}
\mathbf{M}_{\mathbf{H},\bm{\Psi}} = \left\{ g\left( \operatorname{Normal}(\Phi_{\mathbf{h}_1},\sigma_1),...,\operatorname{Normal}(\Phi_{\mathbf{h}_{\mathrm{N}}},\sigma_{\mathrm{N}}),\omega_1,...,\omega_{\mathrm{N}} \right) | \bm{\psi}\in\bm{\Psi} \right\}
\end{equation}
or uniform distribution, \ie
\begin{equation}
\begin{aligned}
\mathbf{M}_{\mathbf{H},\bm{\Psi}} = \{ g( \operatorname{Uniform}(\Phi_{\mathbf{h}_1} - b_1 /2,\Phi_{\mathbf{h}_1} & + b_1 /2),...,\operatorname{Uniform}(\Phi_{\mathbf{h}_{\mathrm{N}}} - b_{\mathrm{N}}/2,\Phi_{\mathbf{h}_{\mathrm{N}}} + b_{\mathrm{N}}/2), \\
 & \omega_1,...,\omega_{\mathrm{N}} ) | \bm{\psi}\in\bm{\Psi} \}
\end{aligned}
\end{equation}
where $\bm{\psi}=(\sigma_1,...,\sigma_{\mathrm{N}},\omega_1,...,\omega_{\mathrm{N}})$ for gaussian distribution and $\bm{\psi}=(b_1,...,b_{\mathrm{N}},\omega_1,...,\omega_{\mathrm{N}})$ for uniform distribution.

\paragraph{Discrete Control Tasks and Value-Based Algorithms}
\begin{equation}
\mathbf{M}_{\mathbf{H},\bm{\Psi}} = \left\{ g\left( \epsilon_1\mathrm{-greedy}(\Phi_{\mathbf{h}_1}),...,\epsilon_{\mathrm{N}}\mathrm{-greedy}(\Phi_{\mathbf{h}_{\mathrm{N}}}),\omega_1,...,\omega_{\mathrm{N}} \right) | \bm{\psi}\in\bm{\Psi} \right\}
\end{equation}
where $\bm{\psi}=(\epsilon_1,...,\epsilon_{\mathrm{N}},\omega_1,...,\omega_{\mathrm{N}})$.

\paragraph{Discrete Control Tasks and Policy-Based Algorithms}
\begin{equation}
\mathbf{M}_{\mathbf{H},\bm{\Psi}} = \left\{ g\left( \operatorname{softmax}_{\tau_1}(\Phi_{\mathbf{h}_1}),...,\operatorname{softmax}_{\tau_{\mathrm{N}}}(\Phi_{\mathbf{h}_{\mathrm{N}}}),\omega_1,...,\omega_{\mathrm{N}} \right) | \bm{\psi}\in\bm{\Psi} \right\}
\end{equation}
where $\bm{\psi}=(\tau_1,...,\tau_{\mathrm{N}},\omega_1,...,\omega_{\mathrm{N}})$.

\subsection{Selection for Behavior Distillation Function}
\paragraph{Mixture Model}
\begin{equation}
\mathbf{M}_{\mathbf{H},\bm{\Psi}} = \{ \sum_{i=1}^{\mathrm{N}} \omega_i f_{\tau_i}  (\Phi_{\mathbf{h}_i}) | \bm{\psi}\in\bm{\Psi} \}
\end{equation}

\paragraph{Knowledge Distillation}
The knowledge distillation method can been seen as a  derivative form of mixture model. The mixture model is simple and straightforward, but it requires more resources for model storage and inference. To address this disadvantage, we can distill the knowledge of multiple policies into a single network using knowledge distillation.
\begin{equation}
\mathbf{M}_{\mathbf{H},\bm{\Psi}} = \{ \operatorname{Distill}(f_{\tau_1} (\Phi_{\mathbf{h}_1}),..., f_{\tau_{\mathrm{N}}}( \Phi_{\mathbf{h}_{\mathrm{N}}}), \omega_1,...,\omega_{\mathrm{N}}) | \bm{\psi}\in\bm{\Psi} \}
\end{equation}
and the knowledge distillation process $\operatorname{Distill}(\cdot)$ can be realized by supervised learning.

\paragraph{Parameters Fusion} 
Define $\mu_f$ as the generated behavior policy which shares the same network structure with the policy models in the population, and is parameterized by $\bm{\theta}_f$. Define $\bm{\theta}_{\mathbf{h}_i}$ as the parameters of policy $f_{\tau_i}(\Phi_{\mathbf{h}_i, \tau_i})$. Then we can define the parameters fusion function $\operatorname{Fusion}(\cdot,\bm{\omega})$,
\begin{equation}
\mathbf{M}_{\mathbf{H},\bm{\Psi}} = \{\mu_f =\operatorname{Fusion}(f_{\tau_1} (\Phi_{\mathbf{h}_1}),..., f_{\tau_{\mathrm{N}}}( \Phi_{\mathbf{h}_{\mathrm{N}}}), \omega_1,...,\omega_{\mathrm{N}}) | \bm{\psi}\in\bm{\Psi} \}
\end{equation}
where $\bm{\theta}_f=\sum_{i=1}^{\mathrm{N}} \omega_i\bm{\theta}_{\mathbf{h}_i, \tau_i} $.

\clearpage
\section{Adaptive Control Mechanism}
\label{Sec: appendix MAB}

In this paper, we cast the behavior control into the behavior mapping optimization, which can be further simplified into the selection of $\psi \in \Psi$. We formalize this problem via multi-armed bandits (MAB).  In this section, we describes  the multi-arm bandit design of our method. For a more thorough explanation and analysis, we refer the readers to \citep{garivier2008upper}.

\subsection{Discretization}
Since $\Psi$ is a continuous space, the optimization of $\psi \in \Psi$ is a  continuous optimization problem.  However, MAB usually only handle discrete control tasks. Hence, we have to discretize $\Psi$ into $\mathrm{K}$ regions according to the discretization accuracy $\tau$, wherein each arm of MAB corresponds to a region of the continuous space. 

\begin{Remark}
    The discretization accuracy $\tau$  is related to the accuracy of the algorithm. In general, a higher discretization accuracy indicates a higher accuracy of the algorithm, but correspondingly, a higher computational complexity of the algorithm.
\end{Remark}

\begin{Example}[Example of Discretization]
    As for a $\epsilon$-greedy behavior mapping, $\psi=\epsilon$ and $\Psi = \{\psi=\epsilon|\epsilon\in[0,1]\}$. We can set the discretization accuracy $\tau=0.1$, and we can discretize $\Psi$ into $10$ regions corresponding to $\mathrm{K}=10$ arms. Each arm corresponds to an interval. For example, $k=1$ corresponds to $[0,0.1)$; $k=1$ corresponds to $[0.1,0.2)$...$k=10$ corresponds to $[0.9,1.0]$.
\end{Example}

\subsection{Sample and Update} 

We adopt the Thompson Sampling \citep{garivier2008upper}. $\mathcal{K}=\{1,...,\mathrm{K}\}$ denote a set of arms available to the decision maker, who is interested in maximizing the expected cumulative return \citep{agent57,DvD}. The optimal strategy for each actor is to pull the arm with the largest mean reward. At the beginning of each round, each actor will produce a sample mean from its mean reward model for each arm, and pulls the arm from which it obtained the largest sample. After observing the selected arm's reward, it updates its mean reward model.

In general, at each time $t$, MAB method will choose an arm $k_t$ from all possible arms $\mathcal{K} = \{1,...,\mathrm{K}\}$ according to a sampling distribution $\mathcal{P}_{\mathcal{K}}$, which is normally conditioned on the sequence of previous decisions and returns. Then we will uniformly sample the parameters $\psi$ from this discretized regions.  Based on the $\psi$, we can obtain the corresponding behavior according to $\mathcal{F}(\Phi_{h})$ or $\mathcal{F}(\Phi_{h_1},...,\Phi_{h_N})$. Executing the behaviors in the environment, each actor will receive a excitation/reward signal  $R_t(k_t)\in \mathbb{R}$, which will be used to update the MAB.

\subsection{Upper Confidence Bound}

The UCB \citep{garivier2008upper} are often used to encourage MAB to try more the arms with a low frequency of use. Let's first define the number of times that the arm $k$ has been selected within T rounds as follows:

\begin{equation}
\label{equ: N}
N_T(x)=\sum_{t=0}^T \mathbbm{1}_{k_t = x}.
\end{equation}

Then we can obtain the empirical mean of the arm x within T rounds as follows:

\begin{equation}
\label{equ: ucb score no z-score}
V_T(x) = N_T(x) \sum_{t=0}^T R_t(x)\mathbbm{1}_{k_t = x}. 
\end{equation}

The UCB methods  encourage the decision maker (actor-wise) to maximize the UCB scores:
\begin{equation}
\label{equ: ucb score}
\operatorname{Score}_x  = V_T(x) + c \cdot \sqrt{
\frac{\log (1+\sum_j \textbf{N}_T(j))}{1+ N_T(x)}}
\end{equation}
The optimal strategy for each actor is to pull the arm with the largest mean scores. At the beginning of each round, each actor will produce a sample mean from its mean reward model for each arm, and pulls the arm from which it obtained the largest sample. After observing the selected arm's scores, it updates its mean reward model.

\begin{Remark}
    In practical,  Z-score Normalization  are normally  used to  normalized $V_T(x)$, namely $\frac{V_T(x)-\mathbb{E}[V_T] }{\textbf{D}[V_T] }$, which can be formulated as 
    \begin{equation}
        \operatorname{Score}_x  = \frac{V_T(x)-\mathbb{E}[V_T(x)] }{\textbf{D}[V_T(x)] } + c \cdot \sqrt{
\frac{\log (1+\sum_j \textbf{N}_T(j))}{1+ N_T(x)}}
    \end{equation}
\end{Remark}

\subsection{Population-Based  MAB}
In the non-stationary scenario, the distributions of $V_T(x)$ could be shifted in the course of the lifelong learning. The standard UCB-based MAB failed  to adapt to the change of the reward distribution and thus we refer to a population-based MAB to handle this problem, which jointly train a population of MAB with different hyperparameters. The sampling and update procedure of MAB is slightly different from the origin MAB, which will be discussed in the following. The main implementation of our population-based MAB has been concluded in Algorithm \ref{alg:bva}. 

\subsubsection{MAB Population Formulation}
Assuming there are $N$ bandits $B_{\mathbf{h}_i}$ to from a population $\mathcal{B}=\{B_{\mathbf{h}_1},...,B_{\mathbf{h}_N}\}$, wherein each bandit can be uniquely indexed by its hyper-parameter $\mathbf{h}_i$ and keep other hyper-parameters remain the same such as the discretization. In this paper, $\mathbf{h}_i=c$, wherein $c$ is the trade-off coefficient in \eqref{equ: ucb score}, which is uniformly sampled from $[0.5,1.5]$, \ie randomly select a $c \in [0.5,1.5]$ while initializing each bandit. 

\subsubsection{Population-Based Sample}
During the sampling procedure, each bandit $B_{\mathbf{h}_i}$ will sample $D$ arm $k_{i} \in \mathcal{K}$ with the Top-D ucb-scores.  After all the bandits sample $D$ arm, there are $D \times N$ sampled arms. We summarize the number of times each arm is selected, and sorted in descending order by the number of times they are selected. Then,  we can obtain an arm $x_{j,t}$ that is selected the most times, which is the sample output of the population-based MAB. Finally, we uniformly sample a $\psi_{j,t}$ from  the region indexed by $x_{j,t}$.

\begin{Example}
    Assuming there are 7 bandits, and each bandit will sample $D=2$ arms from $\mathcal{K}=\{1,...,10\}$. Assuming that the sample output is as follows:
    \begin{equation*}
        1,2,1,3,2,4,5; 1,1,2,2,1,1,4.
    \end{equation*}
    Then, the arm $k=1$ is the arm being selected the most times, so we can get the sampled arm $x_{j,t}=1$. 
\end{Example}

\begin{Remark}
    Noting that, if there are more than one arm that is selected the most times, we can uniformly sample one from these arms.
\end{Remark}

\subsubsection{Population-Based Update}
With $x_{j,t}$, according to the behavior space \eqref{equ: behavior space with hybrid behavior mapping} , we can obtain a behavior $\mu_{j,t} =\mathcal{F}_{\psi} (\Phi_{\mathbf{h}})$ or $\mu_{j,t} =\mathcal{F}_{\psi_{j,t}} (\Phi_{\mathbf{h}_1},...,\Phi_{\mathbf{h}_\mathrm{N}})$.  Execute $\mu_{j,t}$ in the environment and we can obtain obtain the return $G_{j,t}$. With $G_{j,t}$, we can update each bandit in the population based on \eqref{equ: N} - \eqref{equ: ucb score}.

\subsubsection{Bandit Replacement}
Similar to the sliding windows ~\citep{agent57}, to tackle the  non-stationary  problem, we have to track the changes of optimization objectives in a timely manner. To achieve this, we update the replace in the population regularly so that it captures short-term information to improve its tracking performance. 

\begin{figure}[ht]
  \centering
  \begin{minipage}{.9\linewidth}
    \begin{algorithm}[H]
      \caption{Population-Based Multi-Arm Bandits (Actor-Wise)}  
          \begin{algorithmic}
          \STATE // For Each Actor j
          \STATE // Initialize Bandits Population
            \STATE Initialize each bandit  $B_{\mathbf{h}_i}$ in the population with different hyper-parameters $c$.
                \STATE Incorporate each bandit together to form a population of bandits $\mathcal{B}$.
            \FOR{each episode t}
            \FOR{each $B_{\mathbf{h}_i}$ in $\mathcal{B}$}
            \STATE Sample $D$ arms  with Top-D UCB Score via \eqref{equ: ucb score}.
            \ENDFOR
            \STATE Summarize $N\times D$ arms and count the selected times of each arm.
            \STATE Uniformly sample an arm among arms that selected the most times to obtain arm $x_{j,t}$.
            \STATE Uniformly sample a $\psi_{j,t}$ from  the region indexed by $x_{j,t}$.
            \STATE Obtain a behavior $\mu_{j,t} =\mathcal{F}_{\psi} (\Phi_{\mathbf{h}})$ or $\mu_{j,t} =\mathcal{F}_{\psi_{j,t}} (\Phi_{\mathbf{h}_1},...,\Phi_{\mathbf{h}_\mathrm{N}})$.
            \STATE Execute $\mu_{j,t}$ and obtain the return $G_{j,t}$.
                        \FOR{each $B_{\mathbf{h}_i}$ in $\mathcal{B}$}
            \STATE Update $B_{\mathbf{h}_i}$ via  \eqref{equ: N} and \eqref{equ: ucb score no z-score}.
            \ENDFOR
            \STATE Update each bandit in the population via  \eqref{equ: N} and \eqref{equ: ucb score no z-score}.
            \STATE // Replace Bandit from The Population
            \IF{t mod $T_{replace}$=0}
            \STATE Remove one bandit from the bandit population Uniformly and recreate (reinitialize) one into it.
            \ENDIF
            \ENDFOR
          \end{algorithmic}  
        \label{alg:bva}
    \end{algorithm}
  \end{minipage}
\end{figure}

Noting that there are many methods to solve this non-stationary problem at present, such as the sliding windows ~\citep{agent57}. Since this is not the main proposition of this paper, we just choose a feasible implementation to handle this problem.

\clearpage

\section{Algorithm Pseudocode}
\label{app: Algorithm Pseudocode}
We concluded our algorithm in in the Algorithm. \ref{app alg:LBC}. Apart from that, we also concluded our model architecture in App. \ref{app: Model Architecture}.

\begin{figure}[ht]
  \centering
  \vspace{-0.1in}
  \begin{minipage}{\linewidth}
    \begin{algorithm}[H]
      \caption{Learnable Behavior Control}
          \begin{algorithmic}
          \STATE Initialize the Data Buffer (DB), the Parameter Sever (PS), the Learner Push Parameter Interval $d_{push}$ and the Actor Pull Paramter Interval  $d_{pull}$.
          \STATE // LEARNER i
          \STATE Initialize the network parameter $\theta_i$ (for model structure, see App. \ref{app: Model Architecture})
          \FOR{Training Step t}
          \STATE Load data from DB.
          \STATE Estimate $Q_{\theta_i}$ by $Q^{\Tilde{\pi}} (s_t, a_t) = \mathbb{E}_{\mu} [ Q(s_t, a_t) + \sum_{k \geq 0} \gamma^k 
c_{[t+1:t+k]} \delta^{Q}_{t+k} Q ]$, wherein the target policy $\pi_i=\operatorname{softmax}(A_i)$.
          \STATE Estimate $V_{\theta_i}$ by $V^{\Tilde{\pi}} (s_t) 
        = \mathbb{E}_{\mu} [ 
        V(s_t) + \sum_{k \geq 0} \gamma^k 
     c_{[t:t+k-1]} \rho_{t+k}  \delta^{V}_{t+k} V ]$, wherein the target policy $\pi_i=\operatorname{softmax}(A_i)$.
          \STATE Update $\theta_i$ via $\nabla_{\theta} \mathcal{J}(\theta)=\mathbb{E}_{\pi}\left[\sum_{t=0}^{\infty} \bm{\Phi}_{t} \nabla_{\theta} \log \pi_{\theta_i}\left(a_{t} \mid s_{t}\right)\right].$
          \IF{t mod $d_{push}=0$}
          \STATE Push $\theta_i$ into PS.
          \ENDIF
          \ENDFOR
          \STATE // ACTOR j
         \FOR{kth episode at training step t}
            \STATE   Sample a $\bm{\psi}_{j,k}=(\tau_1,...,\tau_{\mathrm{N}},\omega_1,...,\omega_{\mathrm{N}})$ via the MAB-based meta-controller (see App. \ref{Sec: appendix MAB}).
            
            \STATE \textbf{Generalized Policy Selection.} Adjusting the contribution proportion of the each learned policies for the behavior via a importance weight $\bm{\omega}=(\omega_1,...,\omega_{\mathbf{N}})$.
            
            \STATE \textbf{Policy-Wise Entropy Control.} Adjusting  the entropy of each policy via a $\tau_i$, (e.g., $\pi_{\tau_i}=\operatorname{softmax}_{\tau_i}(\Phi_i)$).
            
            \STATE \textbf{Behavior Distillation from 
            Multiple Policies.} Distilling the entropy-controlled policies into a behavior policy $\mu_{j,k}$ via a mixture model $\mu=\sum_i^{\mathbf{N}} \pi_{\tau_i}$. 
            \STATE Obtaining episode $\tau_{j,k}$ and reward $G_{j,k}$ via executing $\mu_{j,k}$, and push $\tau_{j,k}$ into  DB.
            \STATE Update the meta-controller  with $(\mu_{j,k},G_{j,k})$.
          \IF{t mod $d_{pull}=0$ }
          \STATE Pull $\{\theta_1,...,\theta_\mathbf{N}\}$ from PS.
          \ENDIF
        \ENDFOR
          \end{algorithmic}
        \label{app alg:LBC}
    \end{algorithm}
  \end{minipage}
\end{figure}

\clearpage

% \clearpage
\section{An LBC-based Version of RL}
\label{app: An LBC-based Version of RL}
The behavior space is vital for RL methods, which can be used to categorize RL algorithms. Given the model structure $\Phi$, and the form of $\mathcal{F}$, the behavior space can be fully determined by $\mathbf{H}$ and $\bm{\Psi}$, which can be used to categorize RL methods. We say one algorithm belongs to LBC-$\mathbf{H}_{\text{N}}^{\text{C}}$-$\bm{\Psi}_{\text{L}}^{\text{K}}$ when \textbf{1)} the hyper-parameters $\mathbf{h}$ is a C-D vector and $\mathbf{h}$ has N possible values corresponding to N different policy models, and \textbf{2)} $\bm{\psi}$ is a K-D vector and $\bm{\psi}$ has L possible values corresponding to L realizable behavior mappings at each training step. Based on that, we can offer a general view to understand prior methods from the perspective of behavior control, which is illustrated in Tab. \ref{tab:lbc_rl}.
\begin{comment}
    虽然总空间是$N \times L$，但是这个空间是受限的，并不是每一个点都能全部取到。对于每个actor来说，他只能从N个model中做出选择。
\end{comment}
\begin{table*}[!htbp]
    \centering
    \caption{An LBC-based Version of RL Methods.}
    \label{tab:lbc_rl}
    \resizebox{\textwidth}{!}{% <------ Don't forget this %
    \begin{tabular}{l l l l l l}
    \toprule
                Algorithm &  PBT  & Agent57 &  DvD  & LBC-$\mathcal{BM}$ (Ours)\\
    \midrule

                $\mathcal{F}$ & $\epsilon-greedy$ & $\epsilon-greedy$ & identical mapping & $\sum_{i=1}^{\text{N}} \omega_i \operatorname{Softmax}_{\tau_i}(\Phi_{\mathbf{h}_i})$ \\

                $\mathbf{M}_{\mathbf{H},\bm{\Psi}}$ & $\{\mathcal{F}_{\psi}(\Phi_{\mathbf{h}_j})|\mathbf{h}_j\in\mathbf{H},\psi \in \bm{\Psi}\}$ & $\{\mathcal{F}_{\psi}(\Phi_{\mathbf{h}_j})|\mathbf{h}_j\in\mathbf{H},\psi \in \bm{\Psi}\}$ & $\{\mathcal{F}_{\psi}(\Phi_{\mathbf{h}_j})|\mathbf{h}_j\in\mathbf{H},\psi \in \bm{\Psi}\}$ & $\{\mathcal{F}_{\psi}(\Phi_{\mathbf{h}_1},...,\Phi_{\mathbf{h}_\text{N}})|\mathbf{h}_j\in\mathbf{H},\psi \in \bm{\Psi}\}$ \\
                
                $\mathbf{H}$  &  $\{\mathbf{h}_i=(h_1,..,h_\text{C})|i=1,...,\text{N}\}$ & $\{(\gamma_i,\beta_i)|i=1,...,\text{N}\}$ & $\{(\lambda_i)|i=1,...,\text{N}\}$  & $\{\mathbf{h}_i=(h_1,..,h_\text{C})|i=1,...,\text{N}\}$ \\
                
                   $\bm{\Psi}$  & $\{(\epsilon_l)|l=1,...,\text{L}\}$    & $\{(\epsilon_l)|l=1,...,\text{L}\}$ &  $\{1\}$ &       $\{(\omega_1,\tau_1,...,\omega_\text{N},\tau_\text{N})\}$  \\
                   
                     % & $\epsilon$-greedy ($\pi_{\bm{\Theta},\mathbf{H}}$) & $\epsilon$-greedy ($\mu_{\bm{\Theta},\mathbf{H}}$)  & $\mu_{{\bm{\Theta},\mathbf{H}}}$ & $\sum_{i=1}^{\text{N}} \omega_i \operatorname{Softmax}_{\tau_i}(\Phi_{\theta_i})$ \\

                 Category   & LBC-H$^\text{C}_\text{N}$-$\bm{\Psi}^\text{1}_\text{L}$ & LBC-H$^\text{2}_\text{N}$-$\bm{\Psi}^\text{1}_\text{L}$ & LBC-H$^\text{1}_\text{N}$-$\bm{\Psi}^\text{1}_\text{1}$& LBC-H$^\text{C}_\text{N}$-$\bm{\Psi}^\text{K=2N}_{\infty}$\\
                 $|\mathbf{M}_{\mathbf{H},\bm{\Psi}} |$ & $\mathrm{N} \times \mathrm{L}$ & $\mathrm{N} \times \mathrm{L}$ & $\mathrm{N}$ & $\infty$ \\
                 Meta-Controller ($\mathbf{H}$) & ES    & MAB  & MAB    & Ensemble  \eqref{equ: behavior space with hybrid behavior mapping} \\
                 Meta-Controller ($\bm{\Psi}$)  & Rule-Based  & Rule-Based   & Rule-Based & MAB                \\
                 % policy model Selection & Yes & Yes & Yes & All Selected \\
                 % behavior mapping optimization   & No & No & No & Yes \\
    \bottomrule
    \end{tabular}
    }
\end{table*}
\normalsize

\clearpage
\section{Experiment Details}
\label{sec:app Experiment Details}

\subsection{Implementation Details}

On top of the general training architecture is the Learner-Actor framework  \citep{impala}, which makes large-scale training easier. 
We employ the burn-in method  \citep{r2d2} to address representational drift and twice train each sample. The recurrent encoder with LSTM \citep{lstm} is also used to solve the partially observable MDP problem \citep{ale}. For a thorough discussion of the hyper-parameters, see App. \ref{Sec: appendix hyper-parameters}.

\subsection{Experimental Setup}

The undiscounted episode returns averaged over 5 seeds are captured using a windowed mean across 32 episodes in addition to the default parameters. All agents were evaluated on 57 Atari 2600 games from the arcade learning environment \citep[ALE]{ale} using the population's average score from model training. Noting that episodes would end at 100K frames, like per prior baseline techniques \citep{rainbow,agent57,laser,ngu,r2d2}.

\subsection{Resources Used}
\label{app: Resources Used}
All the experiment is accomplished using 10 workers with 72 cores CPU and 3 learners with 3 Tesla-V100-SXM2-32GB GPU.

\clearpage

\section{hyper-parameters}
\label{Sec: appendix hyper-parameters}
The hyper-parameters that we used in all experiments are like those of NGU \cite{ngu} and Agent57 \citep{agent57}. However, for completeness and readability, we detail
them below in  Tab. \ref{tab:fixed_model_hyper-parameters_atari}. We also include the hyper-parameters we used in the population-based MAB. For more details on the parameters in ALE, we refer the readers to see \citep{ale2}.

\begin{table}[H]
\begin{center}
\caption{Hyper-Parameters for Atari Experiments.}
\label{tab:fixed_model_hyper-parameters_atari}
\resizebox{\textwidth}{!}{% <------ Don't forget this %
 \begin{tabular}{l l l l }
\toprule
\textbf{Parameter} & \textbf{Value} & \textbf{Parameter} & \textbf{Value}  \\
\midrule
Burn-in & 40 & Replay & 2 \\

Seq-length & 80 & Burn-in Stored Recurrent State & Yes \\

Bootstrap & Yes  & Batch size & 64 \\

$V$-loss Scaling ($\xi$) & 1.0  & $Q$-loss Scaling ($\alpha$) & 5.0 \\

$\pi$-loss Scaling ($\beta$) & 5.0  & Importance sampling clip $\Bar{c}$ & 1.05 \\

Importance Sampling Clip $\Bar{\rho}$ & 1.05 & LSTM Units & 256 \\

Weight Decay Rate & 0.01 & Optimizer & Adam weight decay    \\

Learning Rate & 5.3e-4  & Weight Decay Schedule & Anneal linearly to 0   \\

Warmup Steps & 4000 & Learning Rate Schedule & Anneal linearly to 0 \\

AdamW $\beta_1$ & 0.9  & Auxiliary Forward Dynamic Task & Yes  \\

AdamW $\epsilon$ & 1e-6 & Learner Push Model Every $d_{push}$ Steps & 25   \\

AdamW $\beta_2$ & 0.98  & Auxiliary Inverse Dynamic Task & Yes \\

AdamW Clip Norm & 50.0  & Actor Pull Model Every $d_{pull}$ Steps & 64 \\

$\gamma_1$ & 0.997  &  $\mathcal{RS}_1$ &$ \operatorname{sign}(x) \cdot (\sqrt{|x|+1}-1)+0.001 \cdot x$\\

$\gamma_2$ & 0.999  & $\mathcal{RS}_2$ (log scaling) & $ \log(|x|+1) \cdot (2 \cdot \mathbbm{1}_{r\geq0} - \mathbbm{1}_{r\le 0})$ \\ 

$\gamma_3$ & 0.99 &  $\mathcal{RS}_3$  & $0.3 \cdot min(\tanh{x},0) + 5\cdot max(\tanh{x},0)$ \\ 

Population Num.   & 7  &  UCB $c$ & Uniformly sampled from $[0.5,1.5]$\\ 

D of Top-D  & 4 & Replacement Interval $T_{replace}$  & 50   \\
Range of $\tau_i$ & $[0,\exp{4}]$  & Range of $\omega_i$ & $[0,1]$\\ 
Discrete Accuracy of $\tau_i$ & 0.2 & Discrete Accuracy of $\omega_i$ & 0.1 \\
Max episode length   & 30 $min$ & Image Size & (84, 84) \\
Grayscaled/RGB      & Grayscaled & Life information & Not allowed  \\
Action Space & Full & Sticky action probability  & 0.0 \\
Num. Action Repeats & 4 & Random noops range  & 30\\
Num. Frame Stacks & 4 & Num. Atari Games & 57 (Full)\\
\bottomrule
\end{tabular} 
}
\end{center}
\end{table}

\clearpage

\section{Experimental Results}
\label{appendix: experiment results}

\subsection{Atari Games Learning Curves}

\renewcommand{\thesubfigure}{\arabic{subfigure}.}
\setcounter{subfigure}{0}

\begin{figure}[!ht] 
    \subfigure[Alien]{
    \includegraphics[width=0.3\textwidth,height=0.15\textheight]{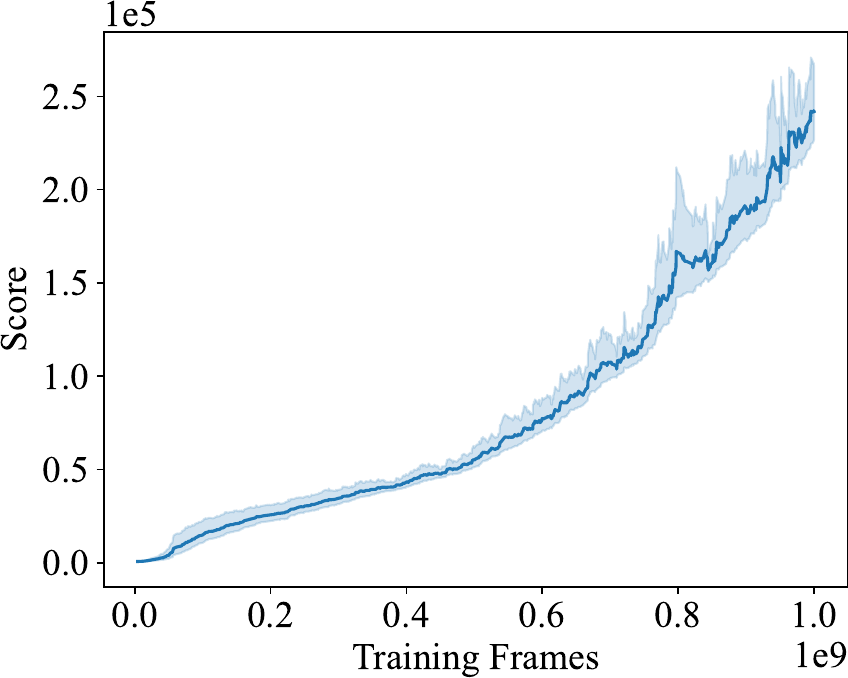}
    }
    \subfigure[Amidar]{
    \includegraphics[width=0.3\textwidth,height=0.15\textheight]{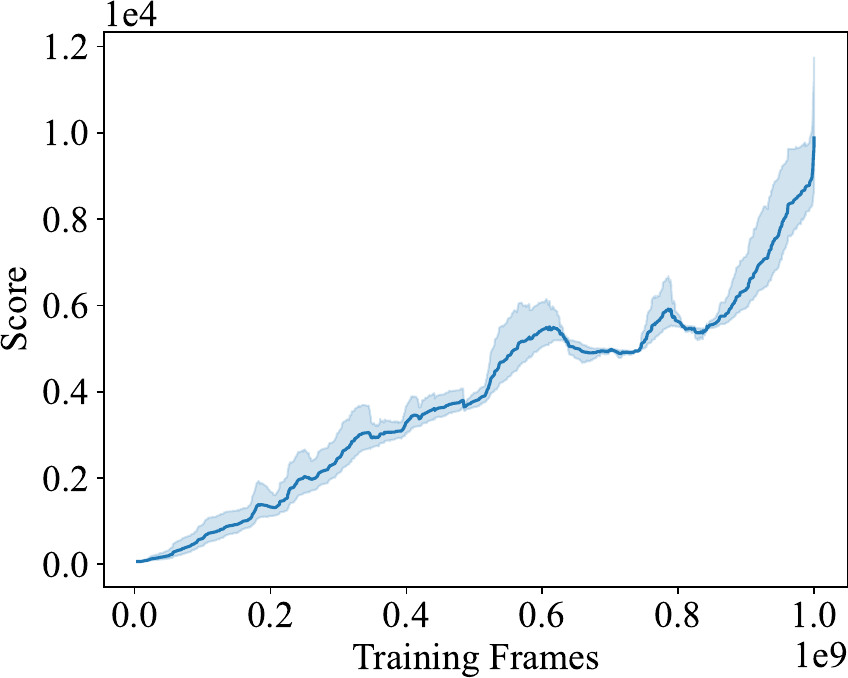}
    }
    \subfigure[Assault]{
    \includegraphics[width=0.3\textwidth,height=0.15\textheight]{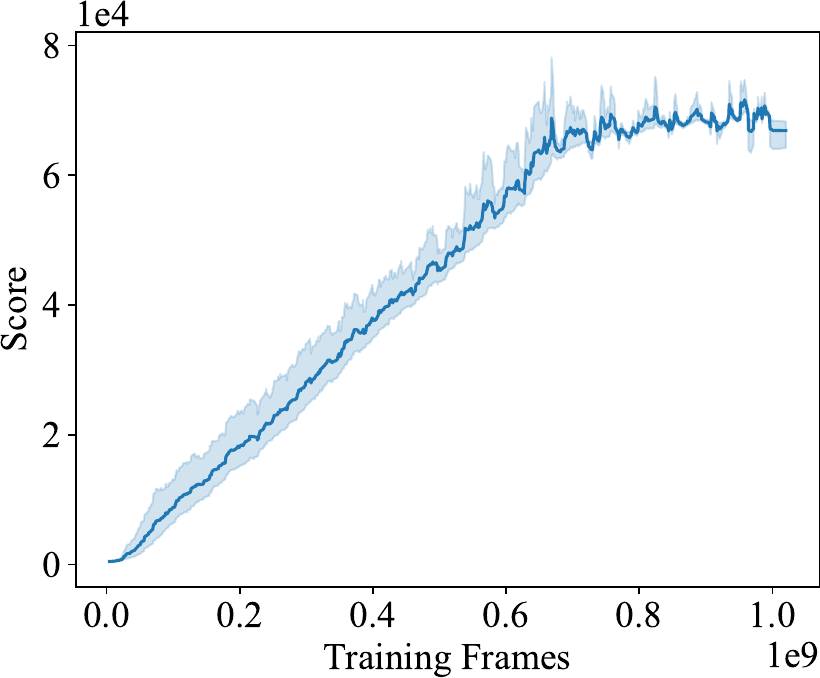}
    }
\end{figure}

\begin{figure}[!ht]
    \subfigure[Asterix]{
    \includegraphics[width=0.3\textwidth,height=0.15\textheight]{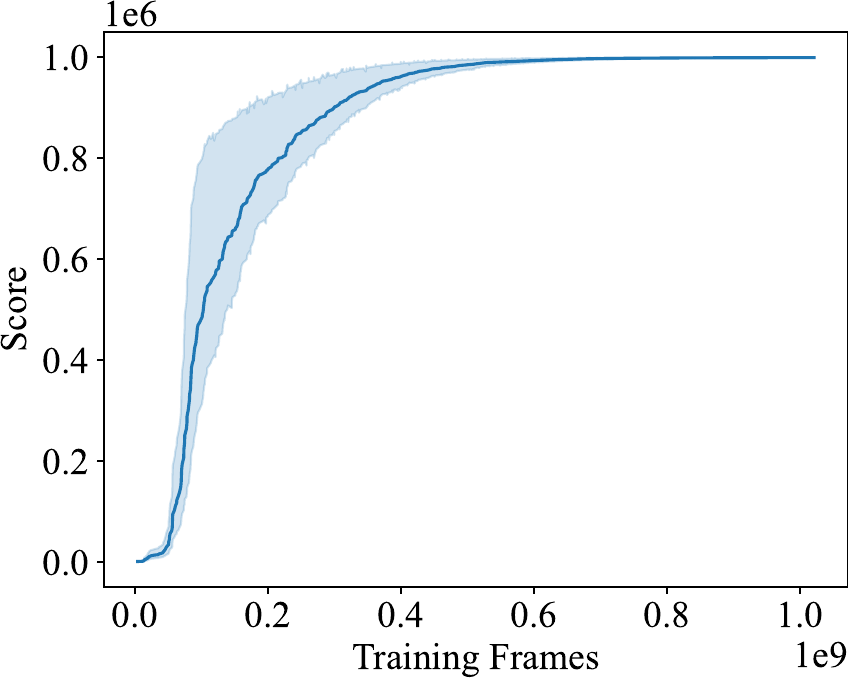}
    }
    \subfigure[Asteroids]{
    \includegraphics[width=0.3\textwidth,height=0.15\textheight]{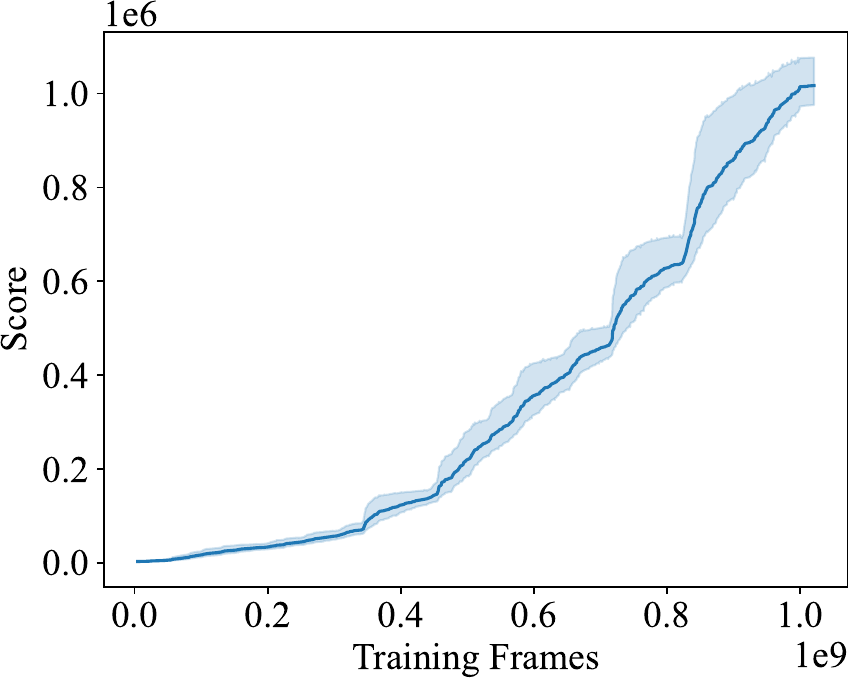}
    }
    \subfigure[Atlantis]{
    \includegraphics[width=0.3\textwidth,height=0.15\textheight]{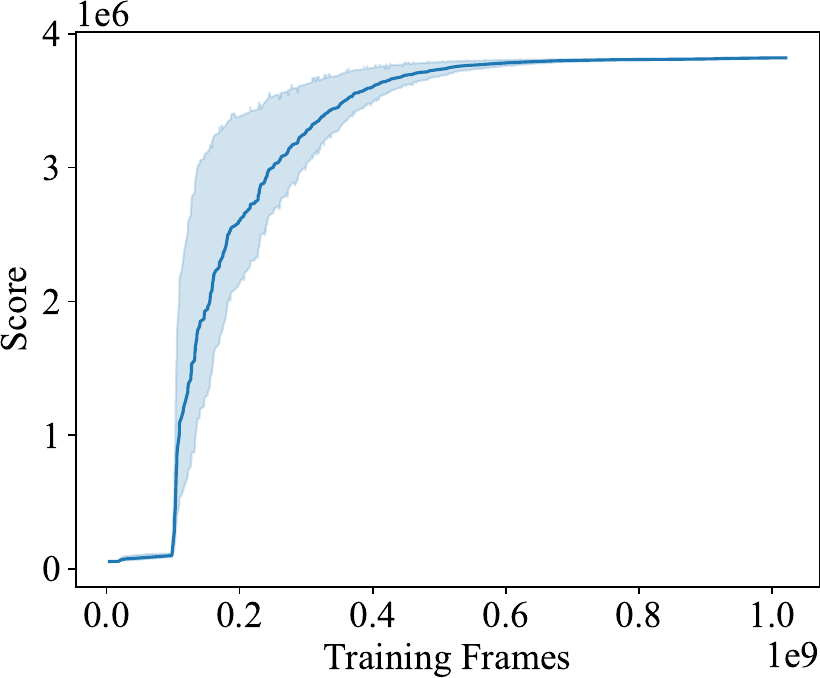}
    }
\end{figure}

\begin{figure}[!ht]
    \subfigure[Bank\_Heist]{
    \includegraphics[width=0.3\textwidth,height=0.15\textheight]{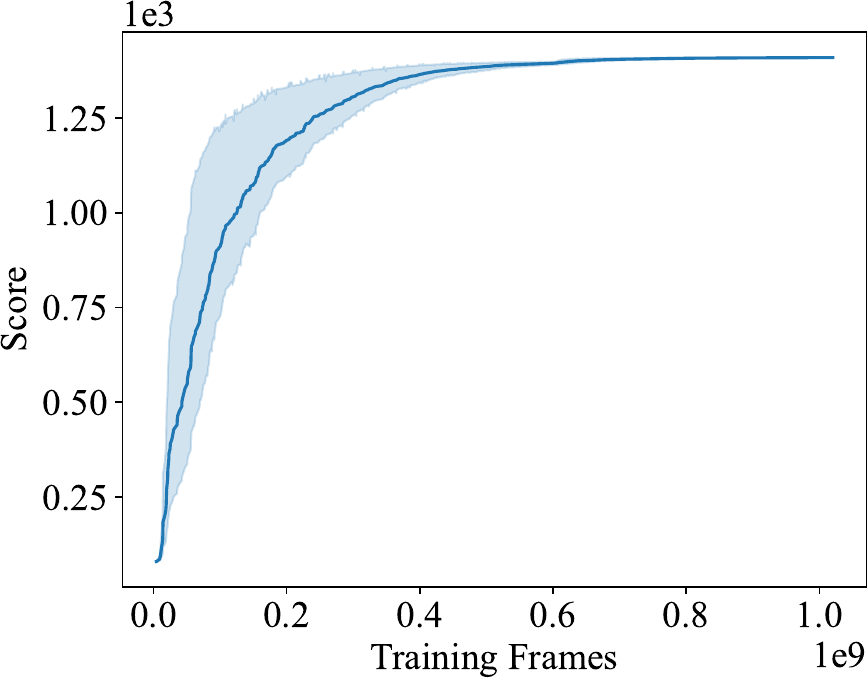}
    }
    \subfigure[Battle\_Zone]{
    \includegraphics[width=0.3\textwidth,height=0.15\textheight]{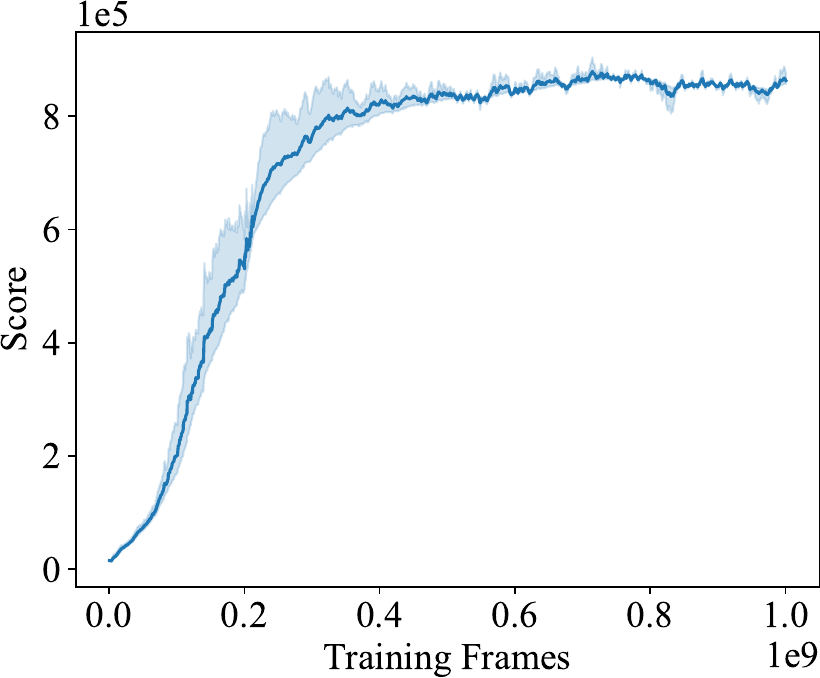}
    }
    \subfigure[Beam\_Rider]{
    \includegraphics[width=0.3\textwidth,height=0.15\textheight]{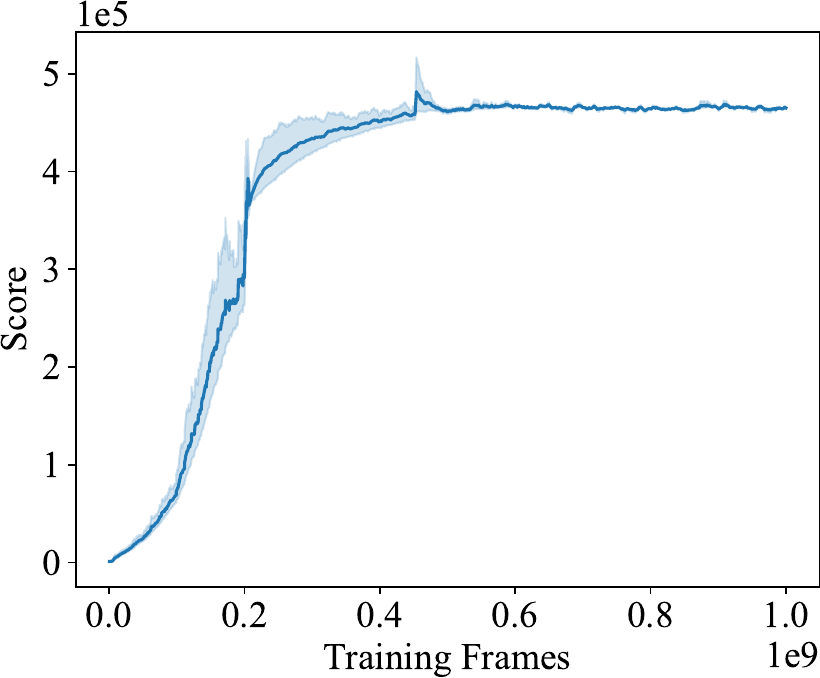}
    }
\end{figure}

\begin{figure}[!ht]
    \subfigure[Berzerk]{
    \includegraphics[width=0.3\textwidth,height=0.15\textheight]{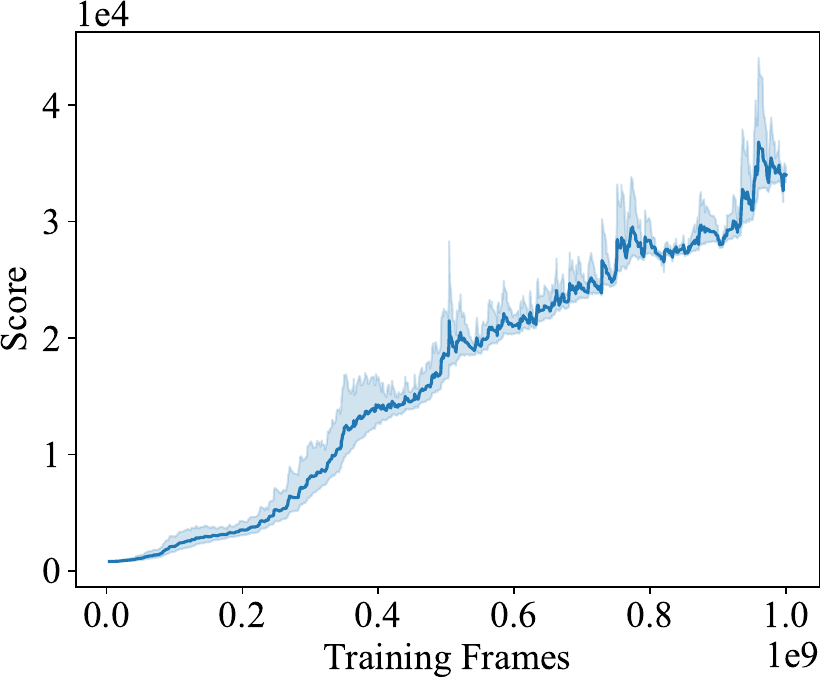}
    }
    \subfigure[Bowling]{
    \includegraphics[width=0.3\textwidth,height=0.15\textheight]{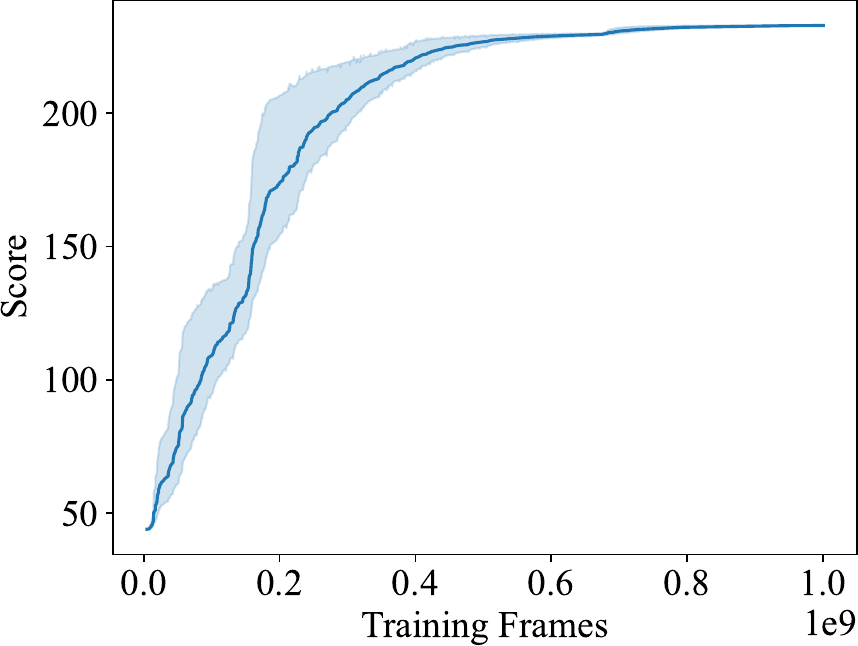}
    }
    \subfigure[Boxing]{
    \includegraphics[width=0.3\textwidth,height=0.15\textheight]{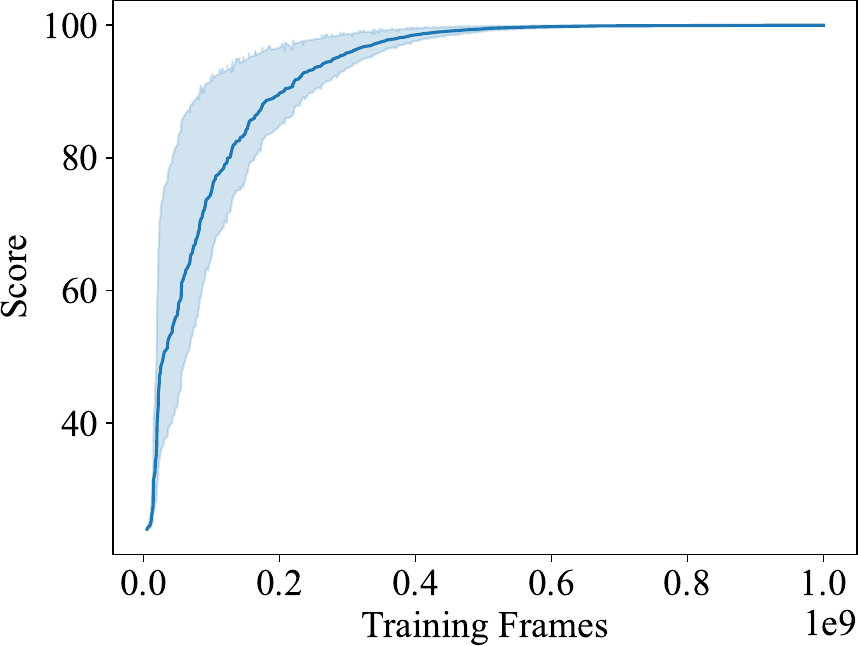}
    }
\end{figure}

\begin{figure}[!ht]
    \subfigure[Breakout]{
    \includegraphics[width=0.3\textwidth,height=0.15\textheight]{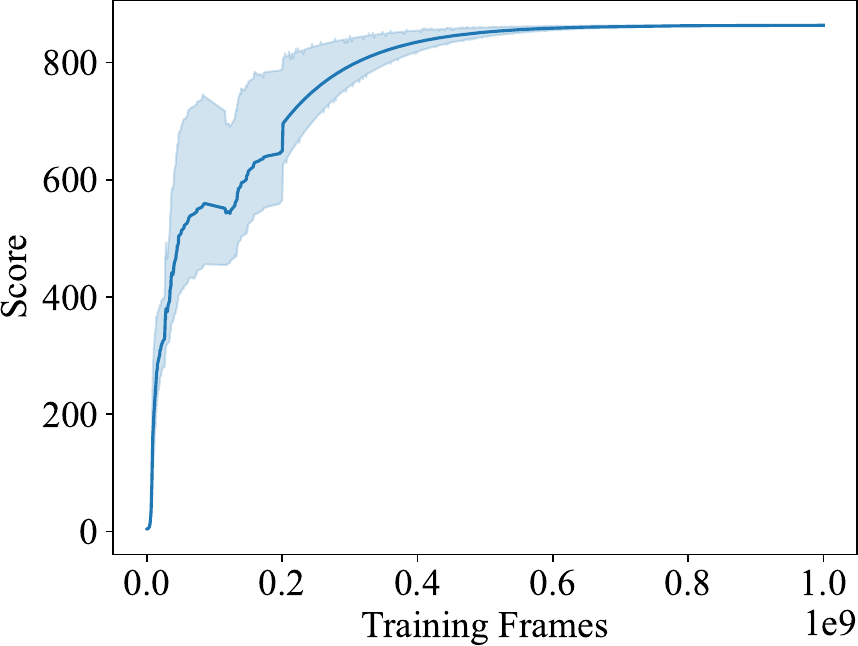}
    }
    \subfigure[Centipede]{
    \includegraphics[width=0.3\textwidth,height=0.15\textheight]{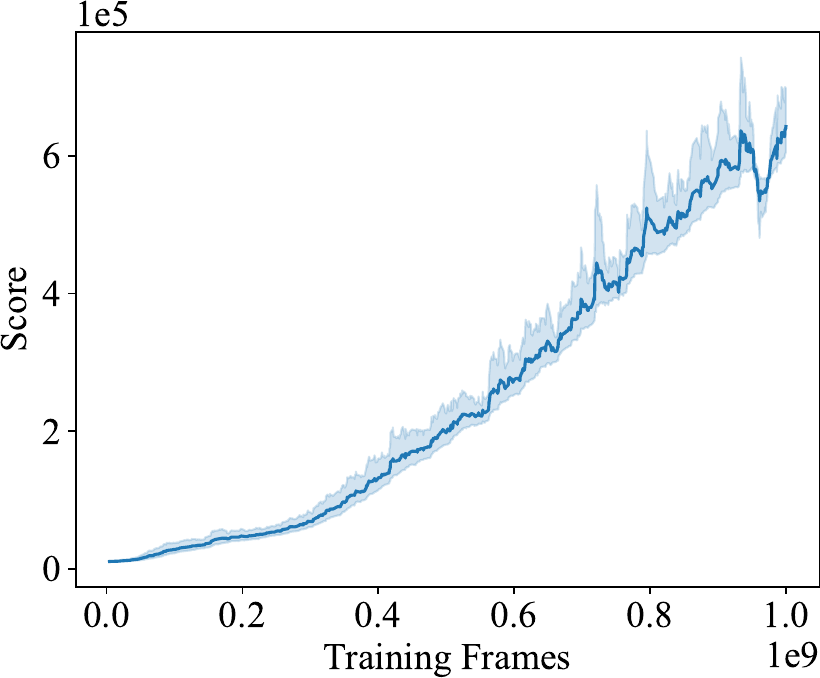}
    }
    \subfigure[Chopper\_Command]{
    \includegraphics[width=0.3\textwidth,height=0.15\textheight]{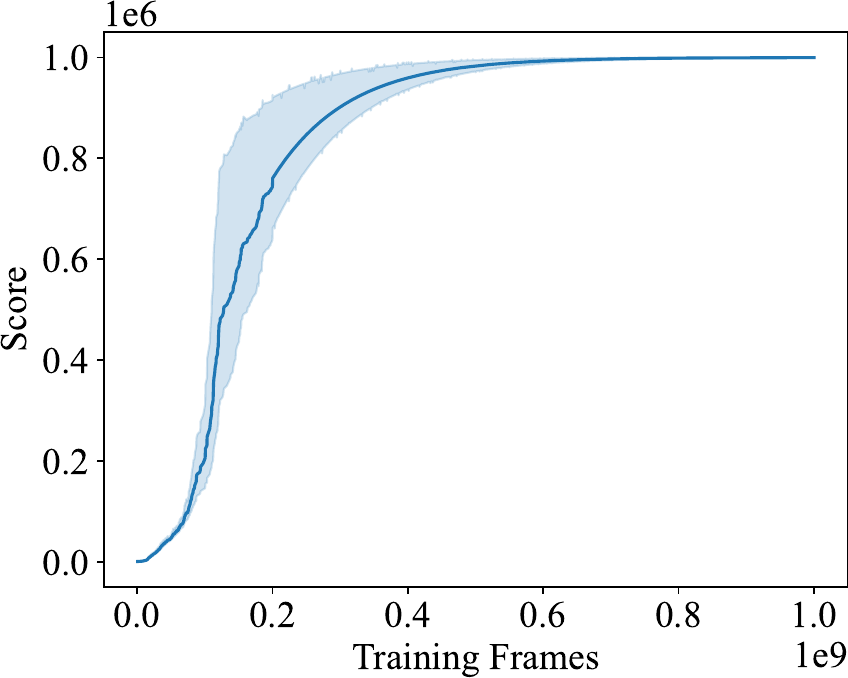}
    }
\end{figure}

\begin{figure}[!ht]
    \subfigure[Crazy\_Climber]{
    \includegraphics[width=0.3\textwidth,height=0.15\textheight]{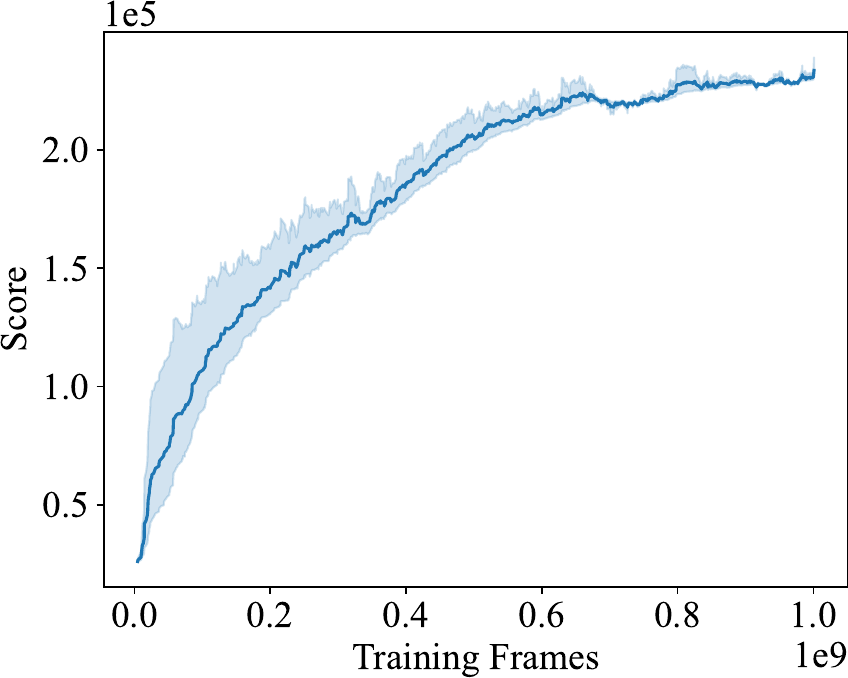}
    }
    \subfigure[Defender]{
    \includegraphics[width=0.3\textwidth,height=0.15\textheight]{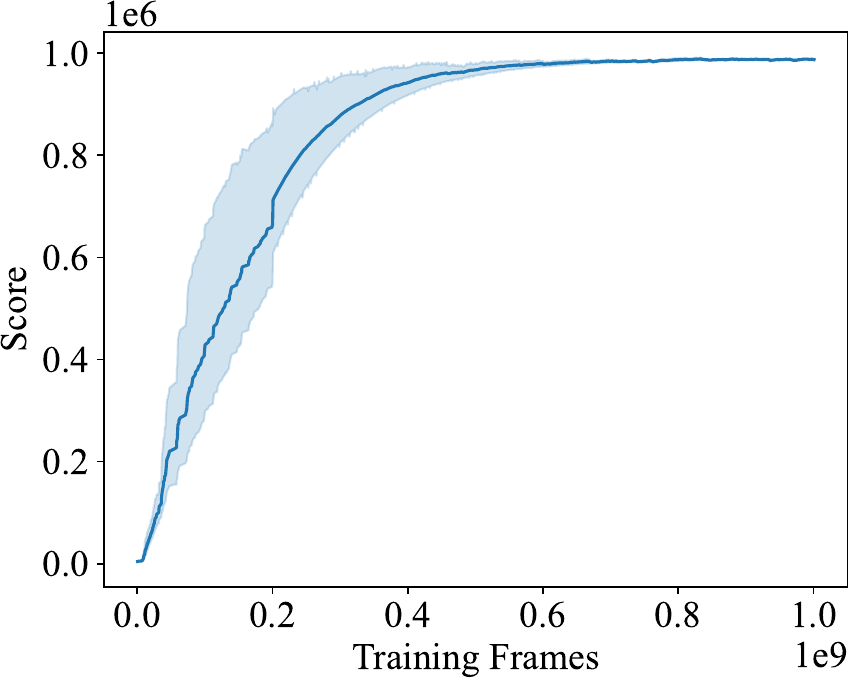}
    }
    \subfigure[Demon\_Attack]{
    \includegraphics[width=0.3\textwidth,height=0.15\textheight]{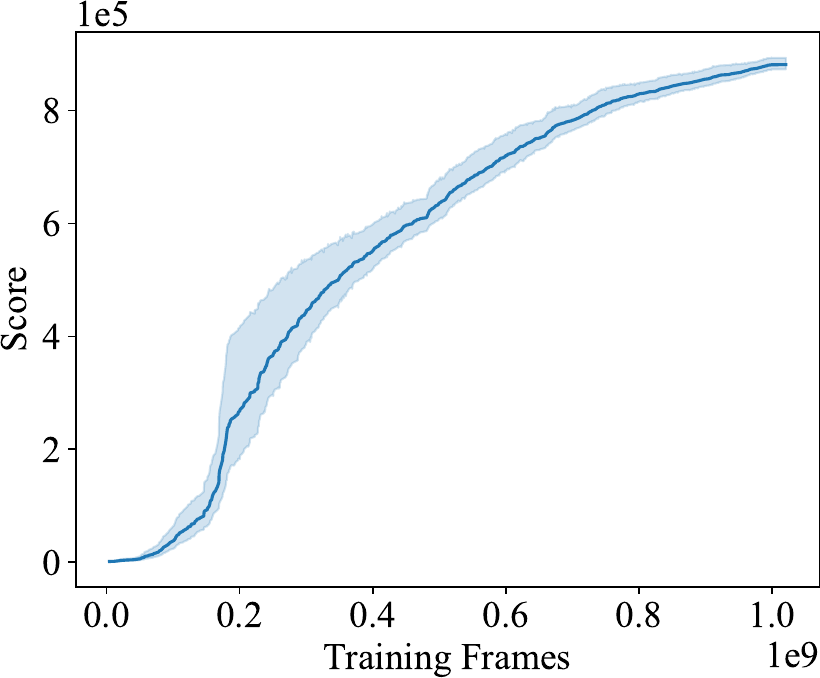}
    }
\end{figure}

\begin{figure}[!ht]
    \subfigure[Double\_Dunk]{
    \includegraphics[width=0.3\textwidth,height=0.15\textheight]{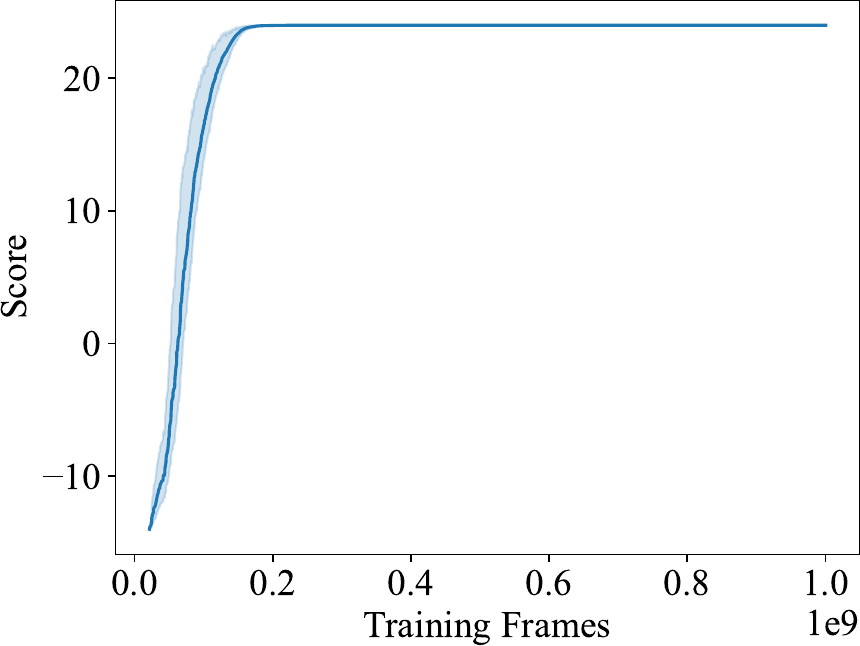}
    }
    \subfigure[Enduro]{
    \includegraphics[width=0.3\textwidth,height=0.15\textheight]{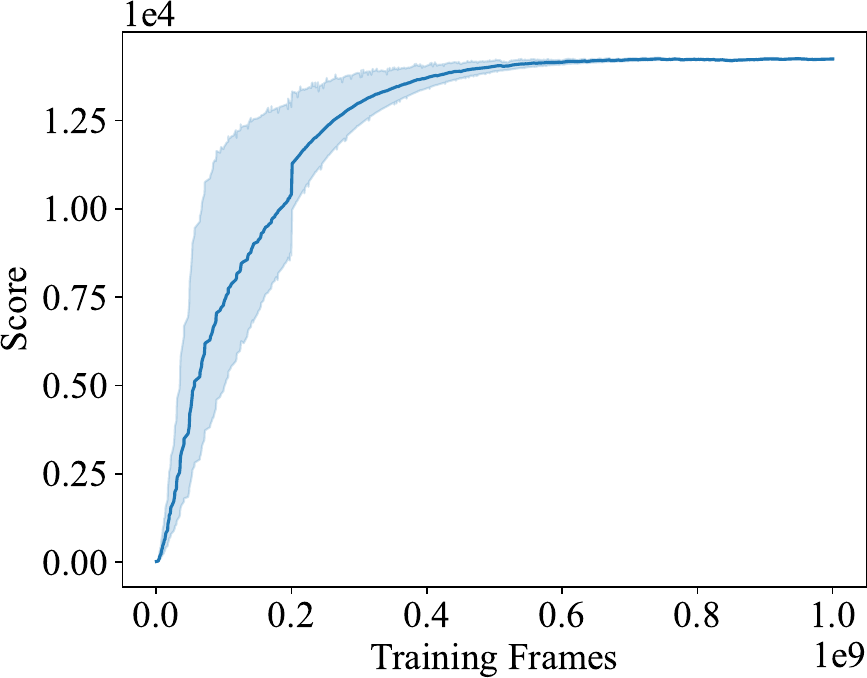}
    }
    \subfigure[Fishing\_Derby]{
    \includegraphics[width=0.3\textwidth,height=0.15\textheight]{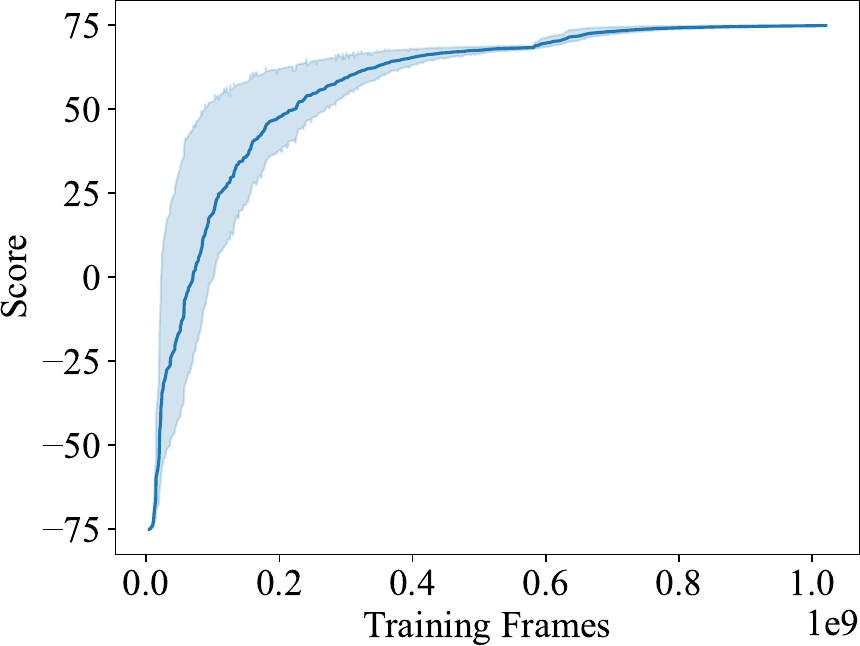}
    }
\end{figure}

\begin{figure}[!ht]
    \subfigure[Freeway]{
    \includegraphics[width=0.3\textwidth,height=0.15\textheight]{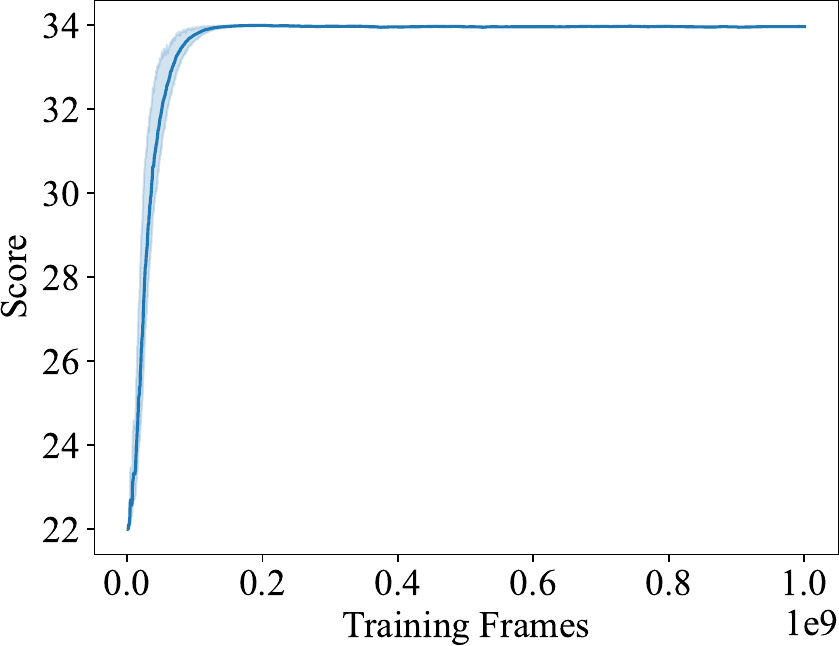}
    }
    \subfigure[Frostbite]{
    \includegraphics[width=0.3\textwidth,height=0.15\textheight]{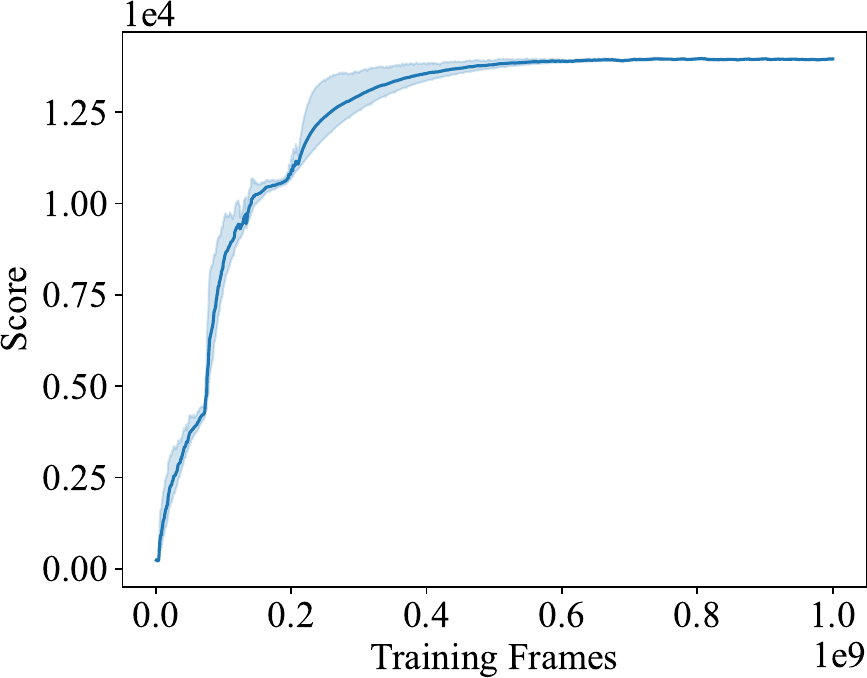}
    }
    \subfigure[Gopher]{
    \includegraphics[width=0.3\textwidth,height=0.15\textheight]{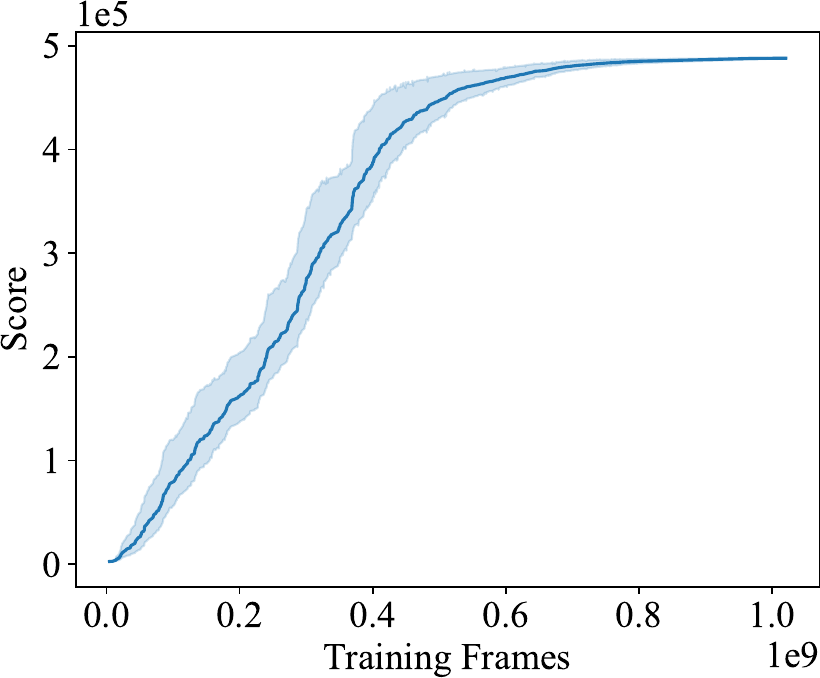}
    }
\end{figure}

\begin{figure}[!ht]
    \subfigure[Gravitar]{
    \includegraphics[width=0.3\textwidth,height=0.15\textheight]{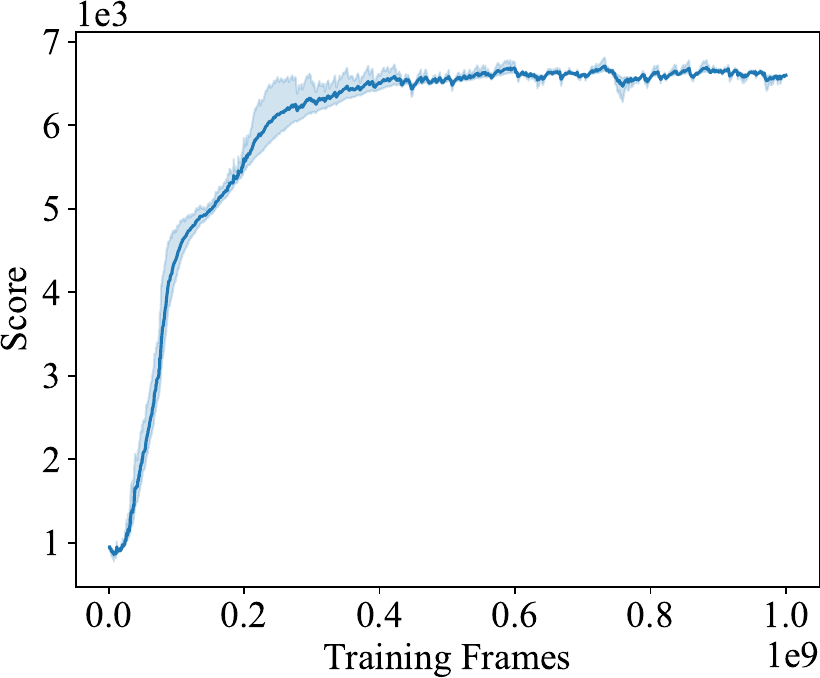}
    }
    \subfigure[Hero]{
    \includegraphics[width=0.3\textwidth,height=0.15\textheight]{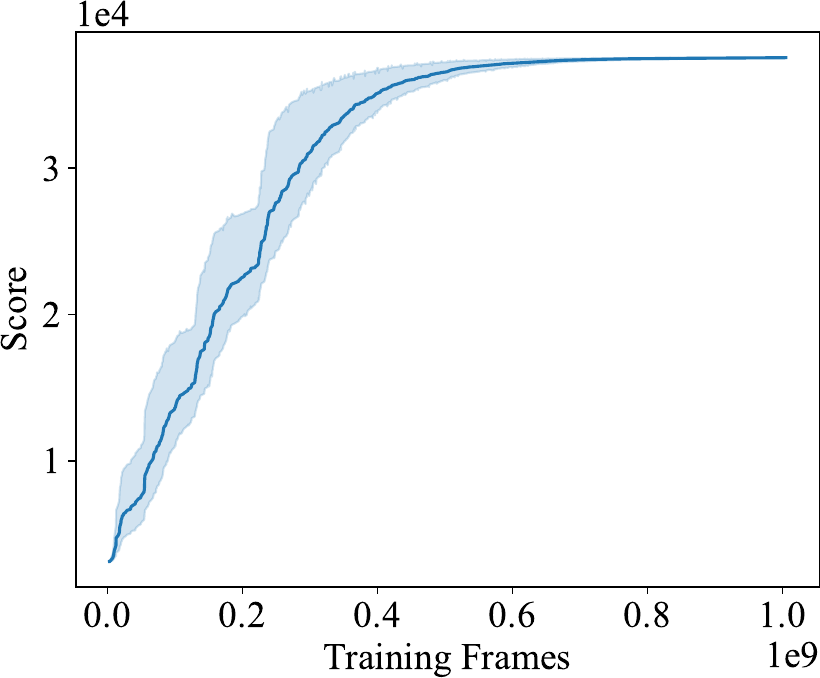}
    }
    \subfigure[Ice\_Hockey]{
    \includegraphics[width=0.3\textwidth,height=0.15\textheight]{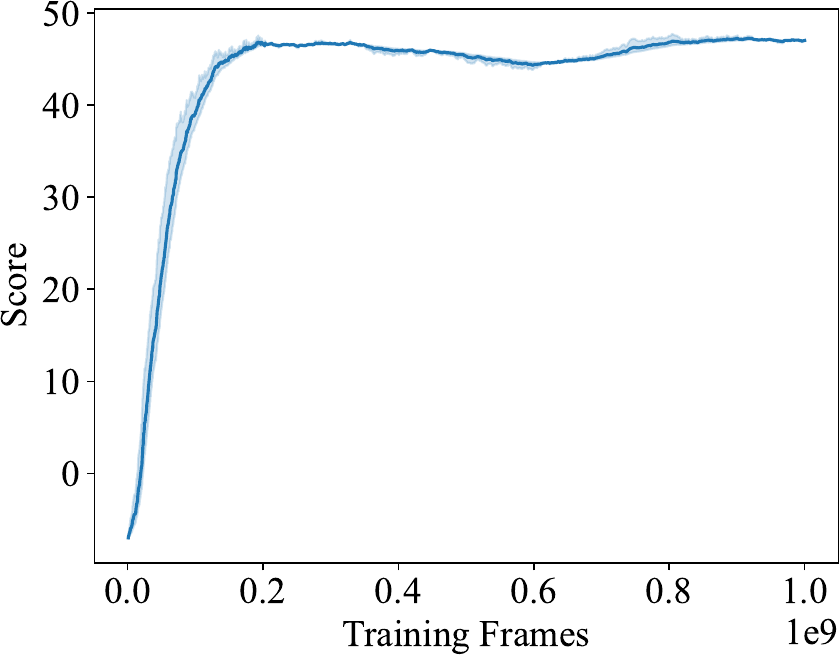}
    }
\end{figure}

\begin{figure}[!ht]
    \subfigure[Jamesbond]{
    \includegraphics[width=0.3\textwidth,height=0.15\textheight]{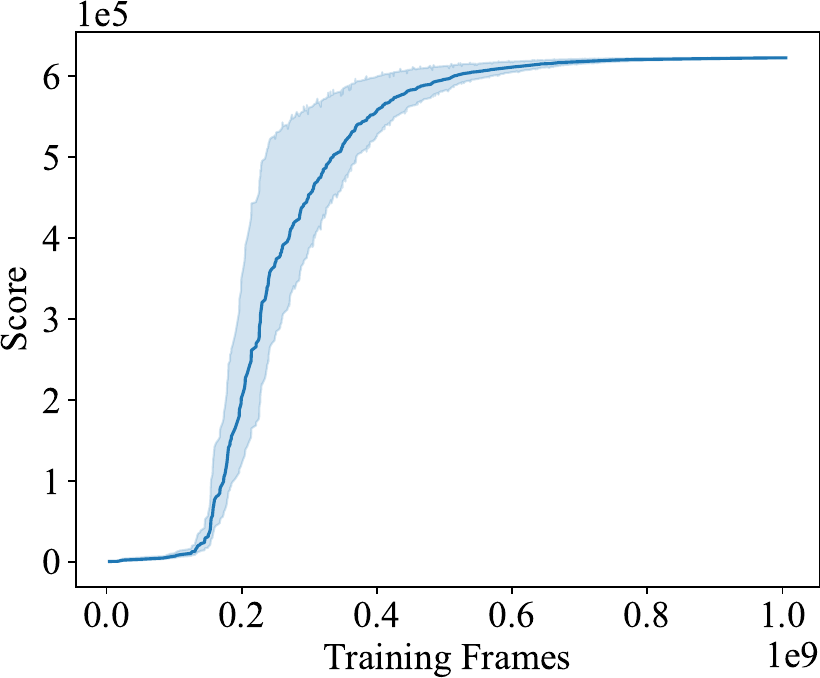}
    }
    \subfigure[Kangaroo]{
    \includegraphics[width=0.3\textwidth,height=0.15\textheight]{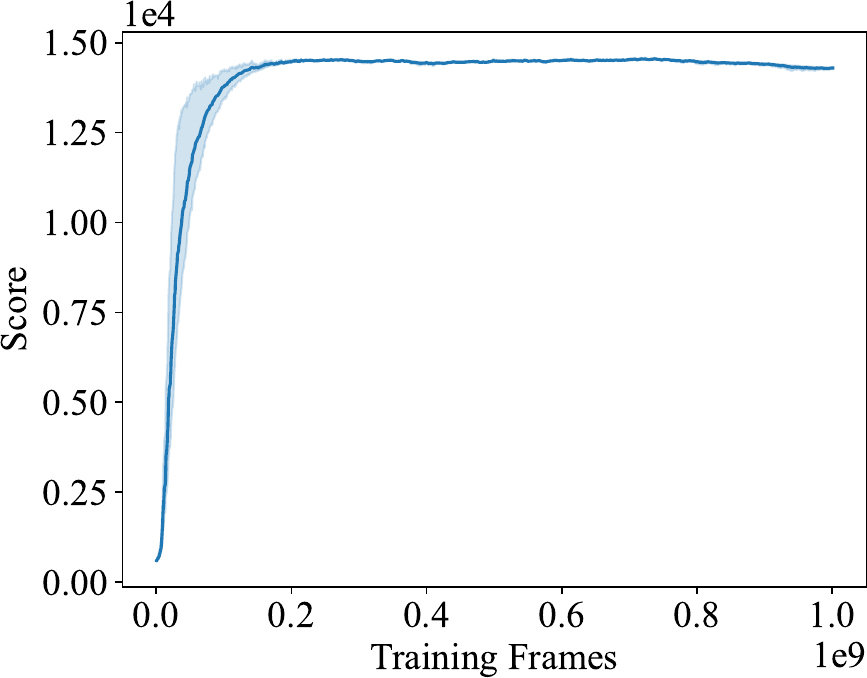}
    }
    \subfigure[Krull]{
    \includegraphics[width=0.3\textwidth,height=0.15\textheight]{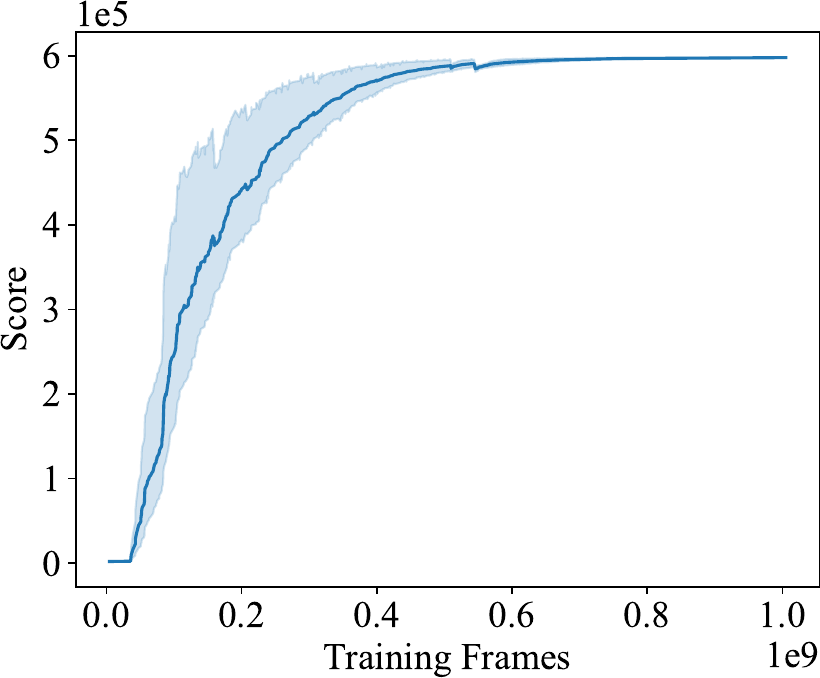}
    }
\end{figure}

\begin{figure}[!ht]
    \subfigure[Kung\_Fu\_Master]{
    \includegraphics[width=0.3\textwidth,height=0.15\textheight]{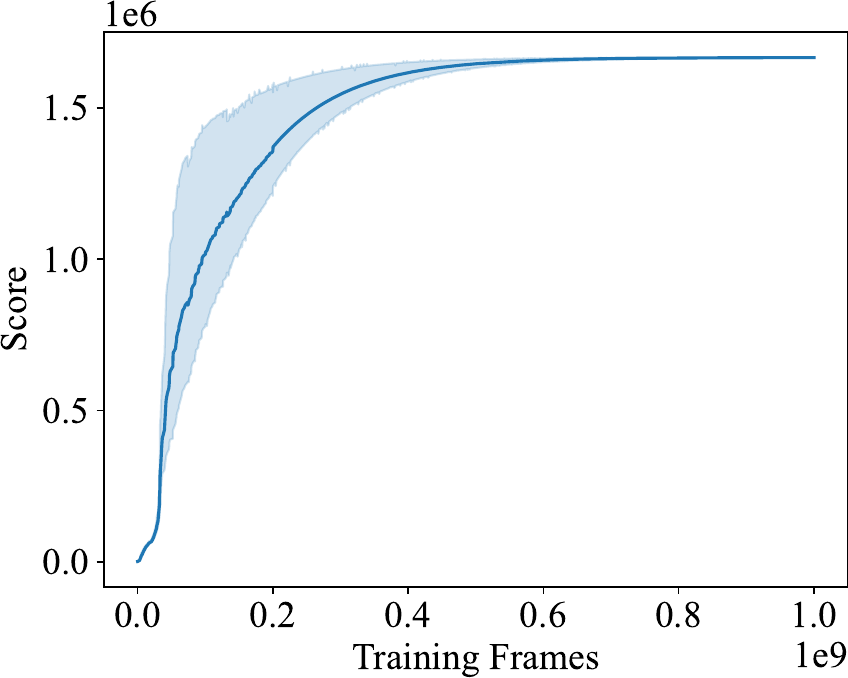}
    }
    \subfigure[Montezuma\_Revenge]{
     \includegraphics[width=0.3\textwidth,height=0.15\textheight]{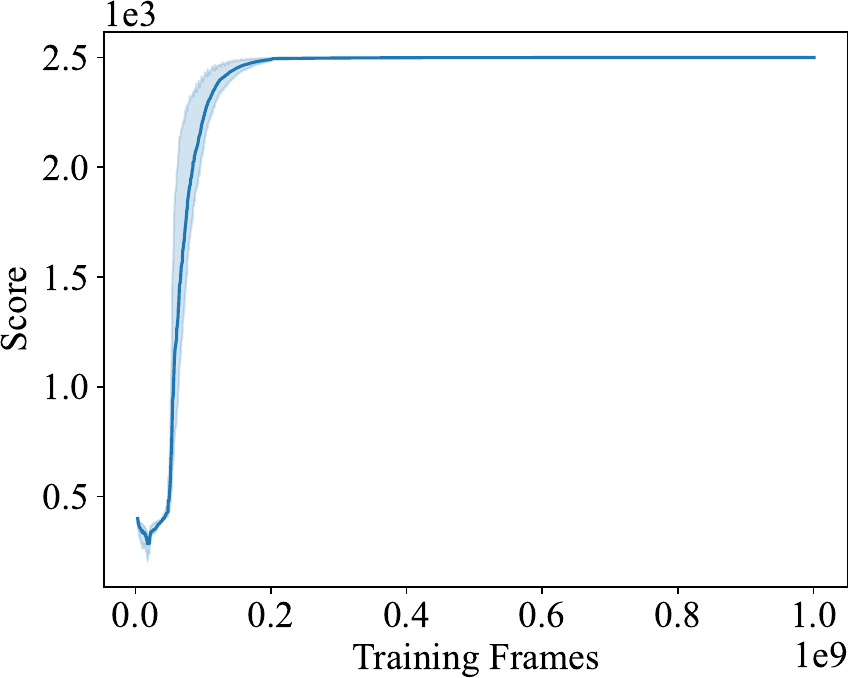}
    }
    \subfigure[Ms\_Pacman]{
    \includegraphics[width=0.3\textwidth,height=0.15\textheight]{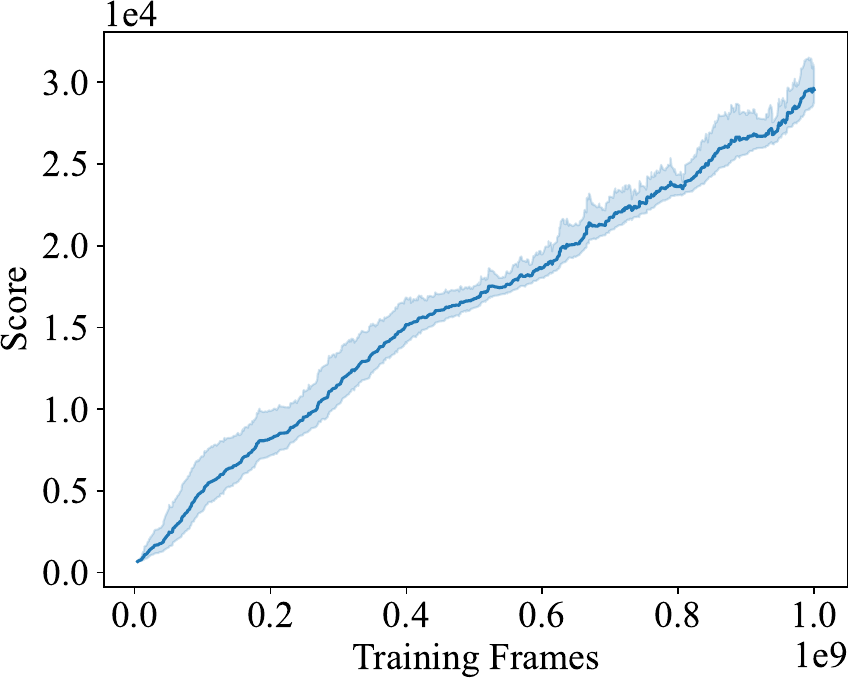}
    }
\end{figure}

\begin{figure}[!ht]
    \subfigure[Name\_This\_Game]{
    \includegraphics[width=0.3\textwidth,height=0.15\textheight]{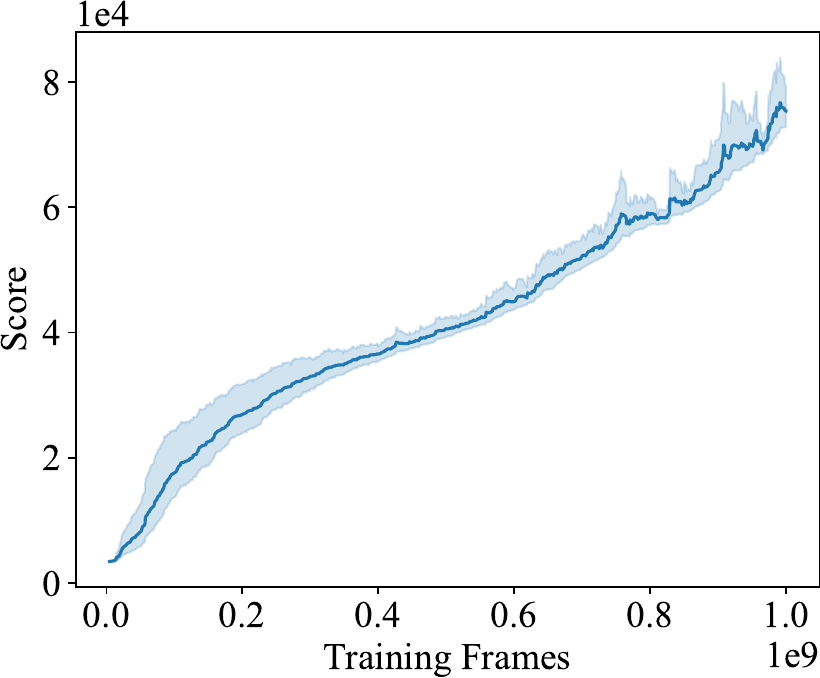}
    }
    \subfigure[Phoenix]{
     \includegraphics[width=0.3\textwidth,height=0.15\textheight]{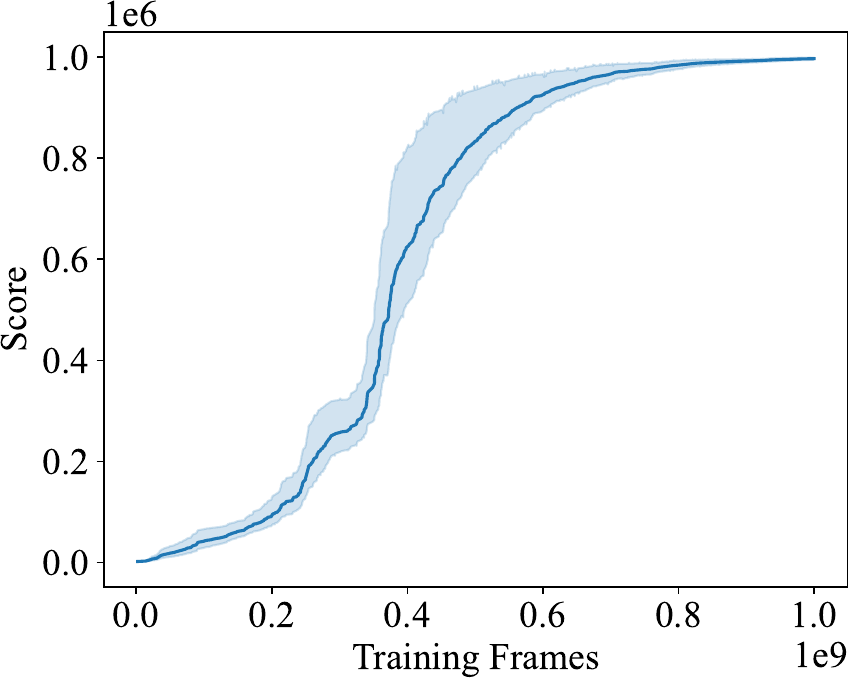}
    }
    \subfigure[Pitfall]{
    \includegraphics[width=0.3\textwidth,height=0.15\textheight]{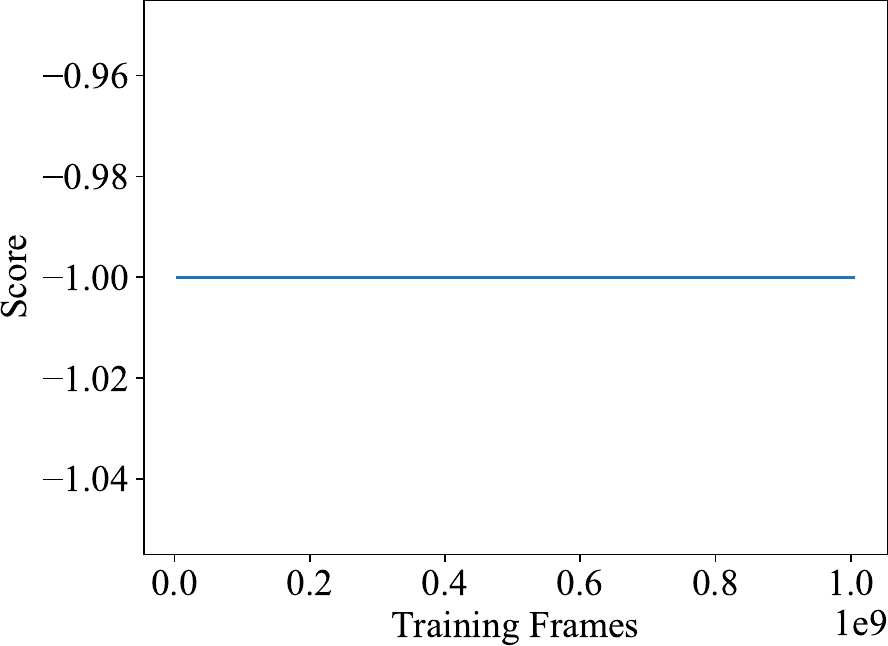}
    }
\end{figure}

\begin{figure}[!ht]
    \subfigure[Pong]{
    \includegraphics[width=0.3\textwidth,height=0.15\textheight]{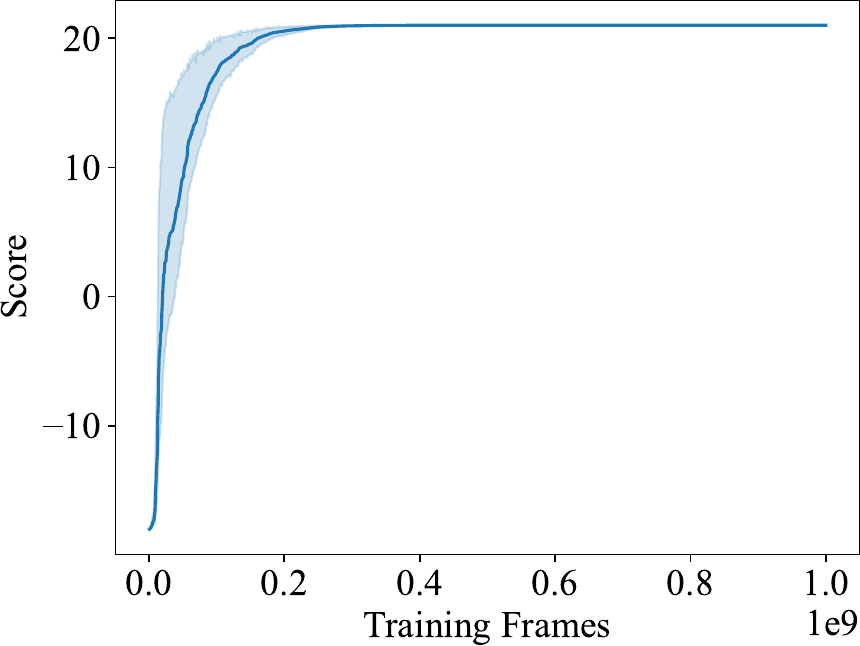}
    }
    \subfigure[Private\_Eye]{
    \includegraphics[width=0.3\textwidth,height=0.15\textheight]{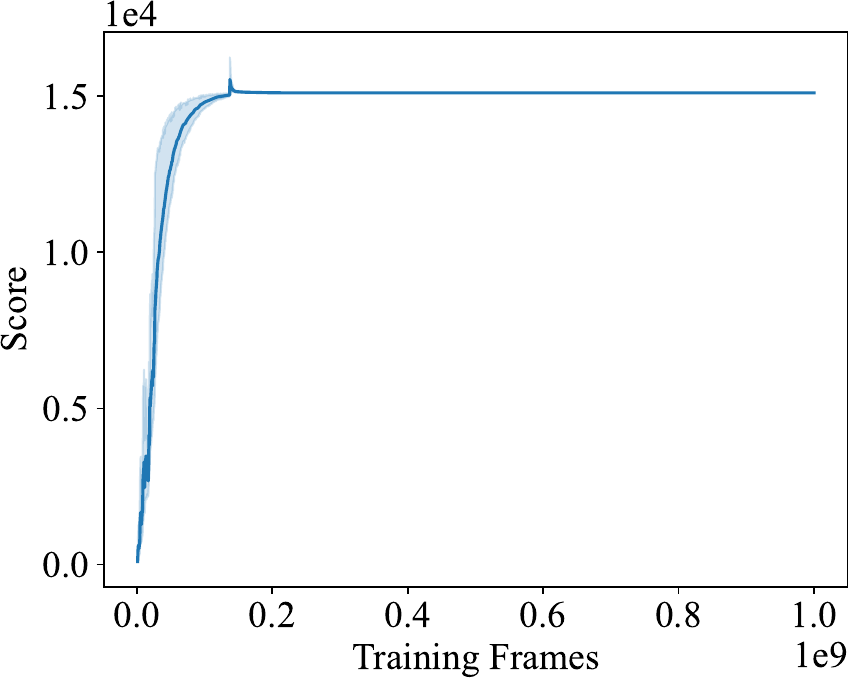}
    }
    \subfigure[Qbert]{
     \includegraphics[width=0.3\textwidth,height=0.15\textheight]{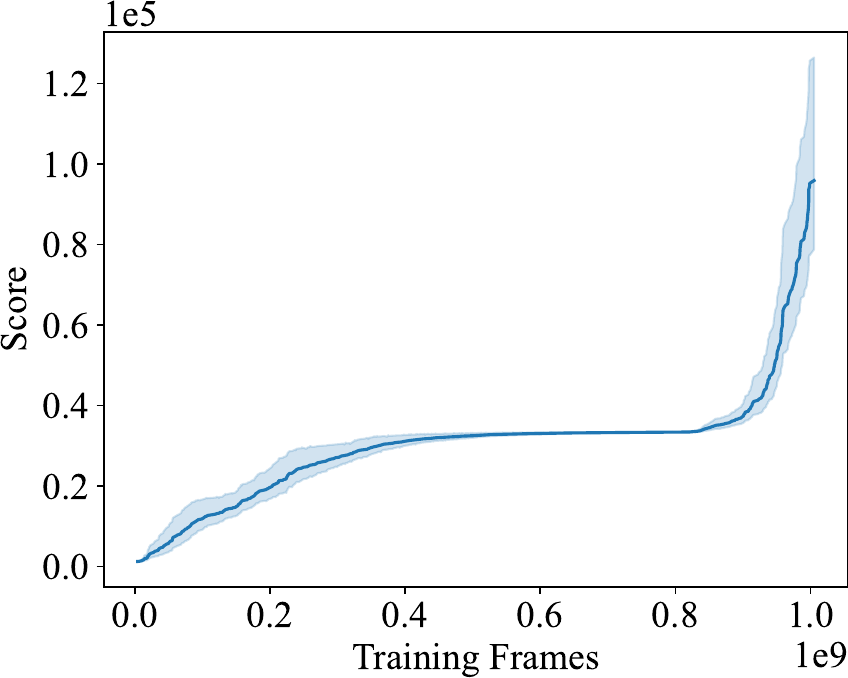}
    }
\end{figure}

\begin{figure}[!ht]
    \subfigure[Riverraid]{
    \includegraphics[width=0.3\textwidth,height=0.15\textheight]{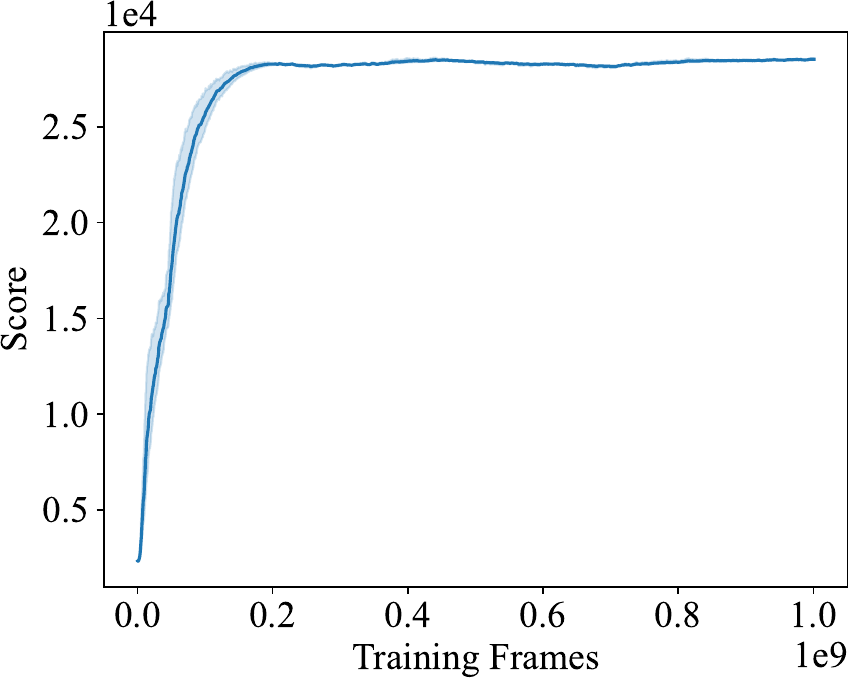}
    }
    \subfigure[Road\_Runner]{
    \includegraphics[width=0.3\textwidth,height=0.15\textheight]{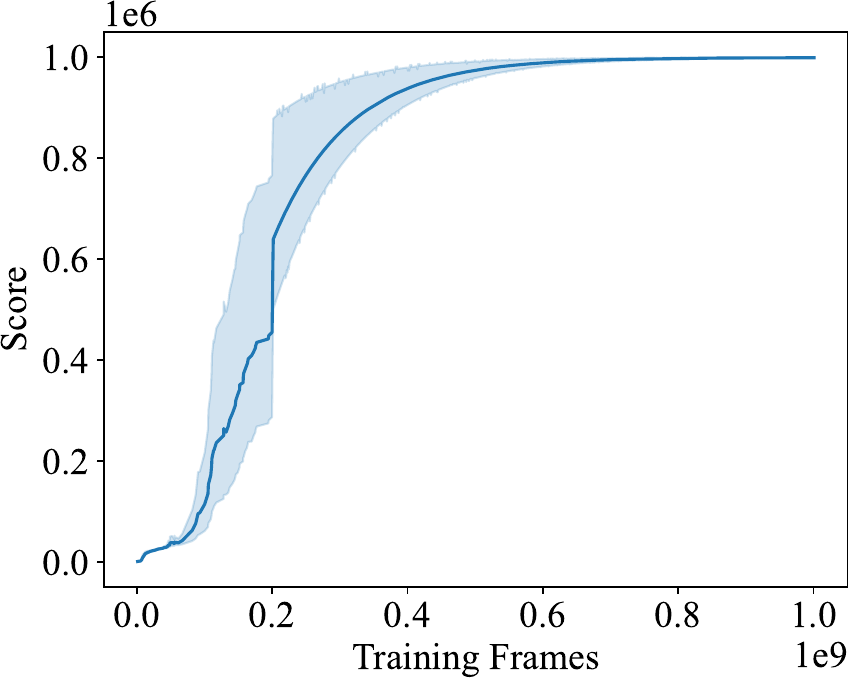}
    }
    \subfigure[Robotank]{
    \includegraphics[width=0.3\textwidth,height=0.15\textheight]{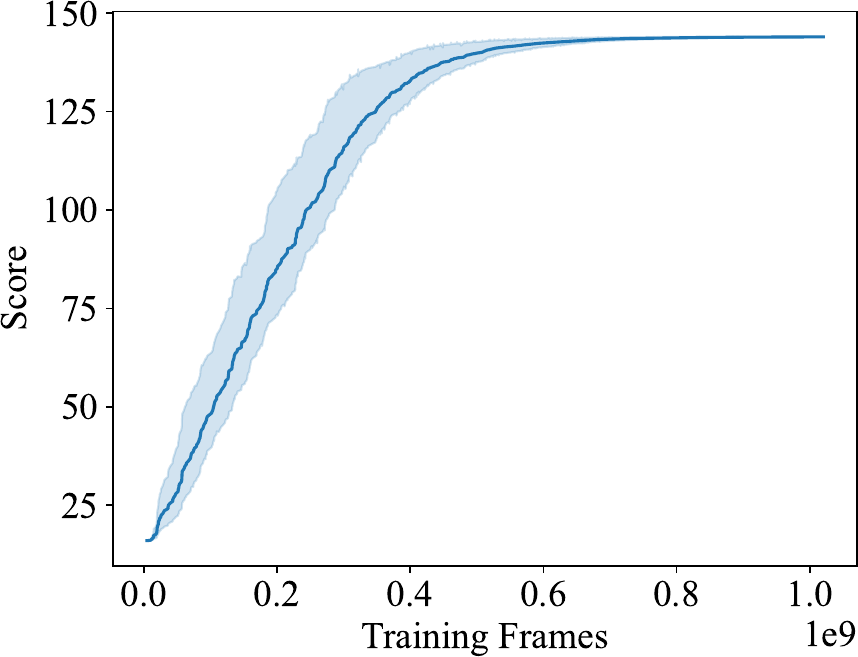}
    }
\end{figure}

\begin{figure}[!ht]
    \subfigure[Seaquest]{
    \includegraphics[width=0.3\textwidth,height=0.15\textheight]{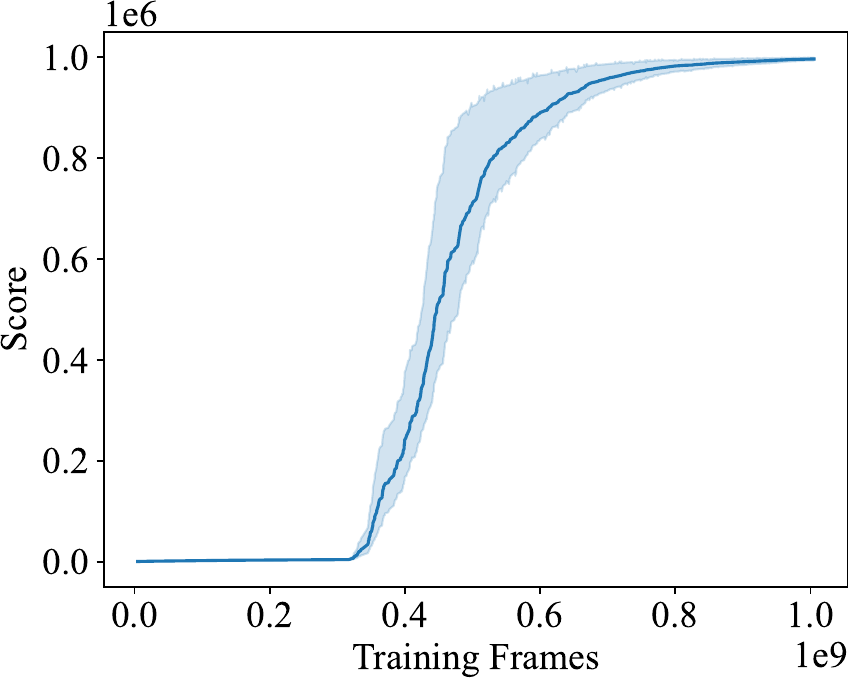}
    }
    \subfigure[Skiing]{
    \includegraphics[width=0.3\textwidth,height=0.15\textheight]{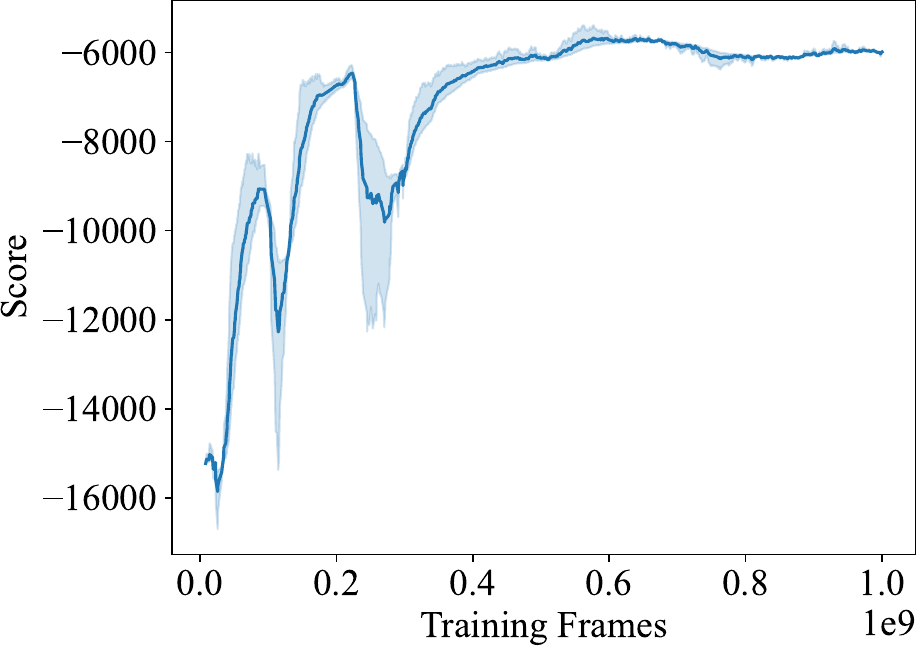}
    }
    \subfigure[Solaris]{
    \includegraphics[width=0.3\textwidth,height=0.15\textheight]{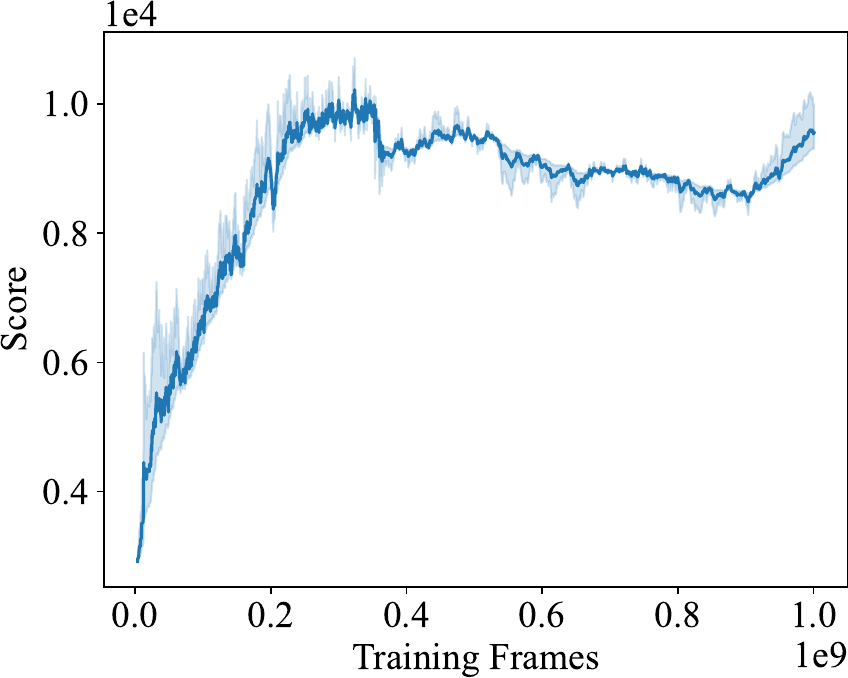}
    }
\end{figure}

\begin{figure}[!ht]
    \subfigure[Space\_Invaders]{
     \includegraphics[width=0.3\textwidth,height=0.15\textheight]{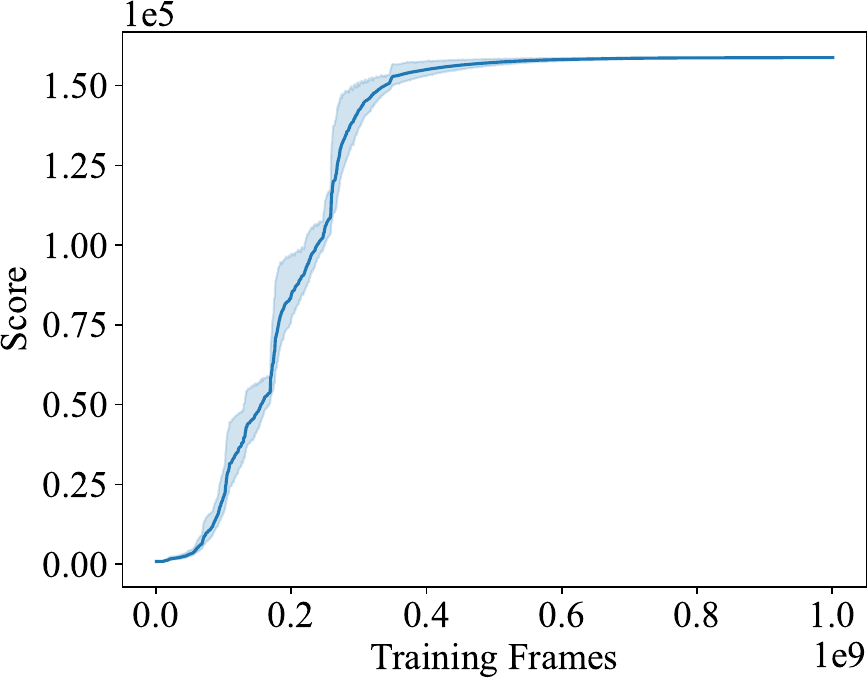}
    }
    \subfigure[Star\_Gunner]{
    \includegraphics[width=0.3\textwidth,height=0.15\textheight]{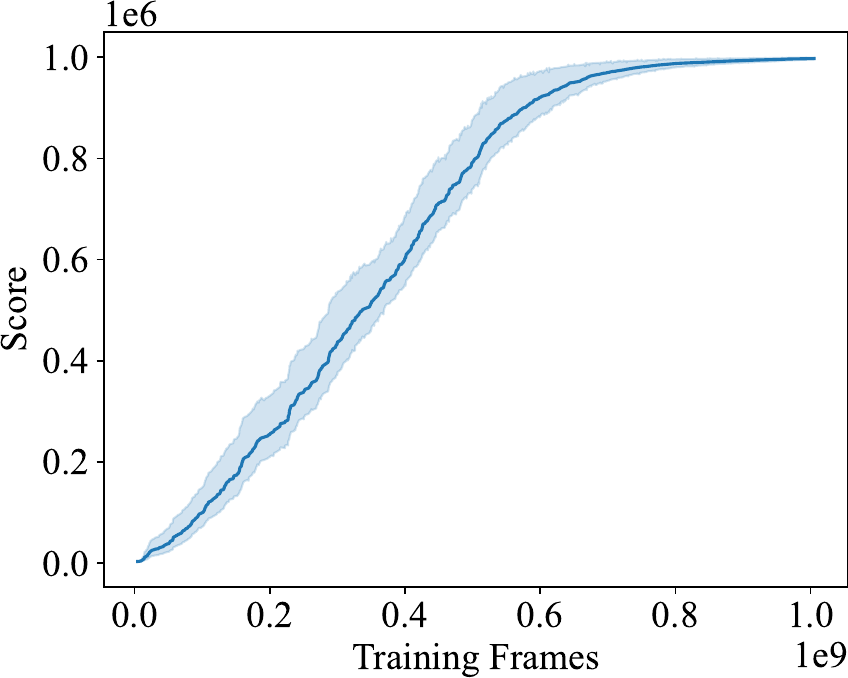}
    }
    \subfigure[Surround]{
    \includegraphics[width=0.3\textwidth,height=0.15\textheight]{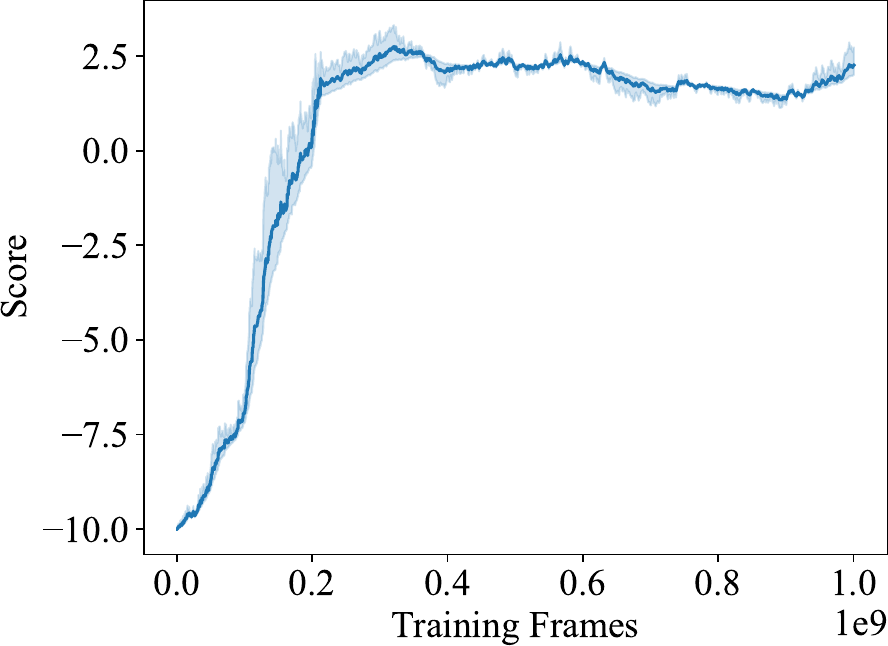}
    }
\end{figure}

\begin{figure}[!ht]
    \subfigure[Tennis]{
    \includegraphics[width=0.3\textwidth,height=0.15\textheight]{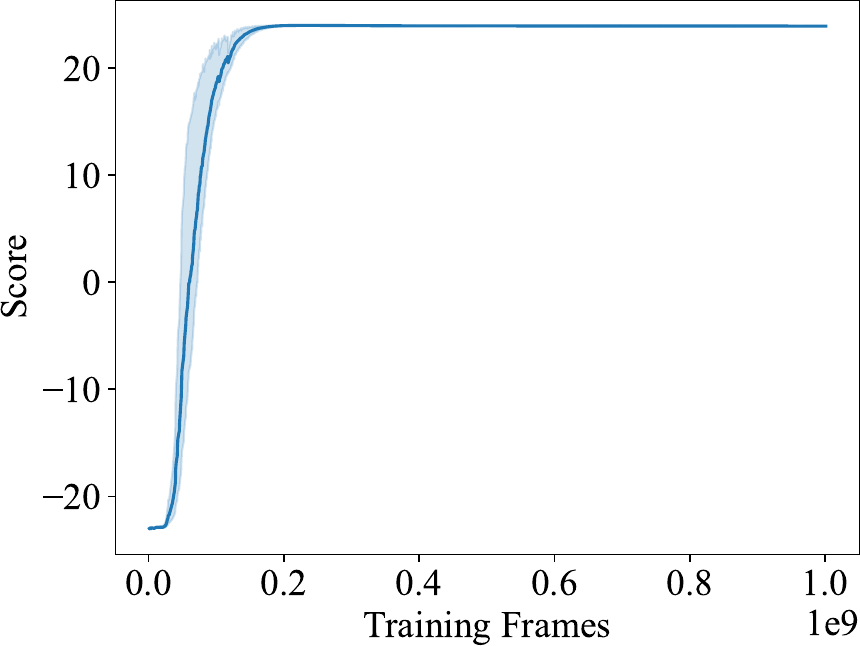}
    }
    \subfigure[Time\_Pilot]{
    \includegraphics[width=0.3\textwidth,height=0.15\textheight]{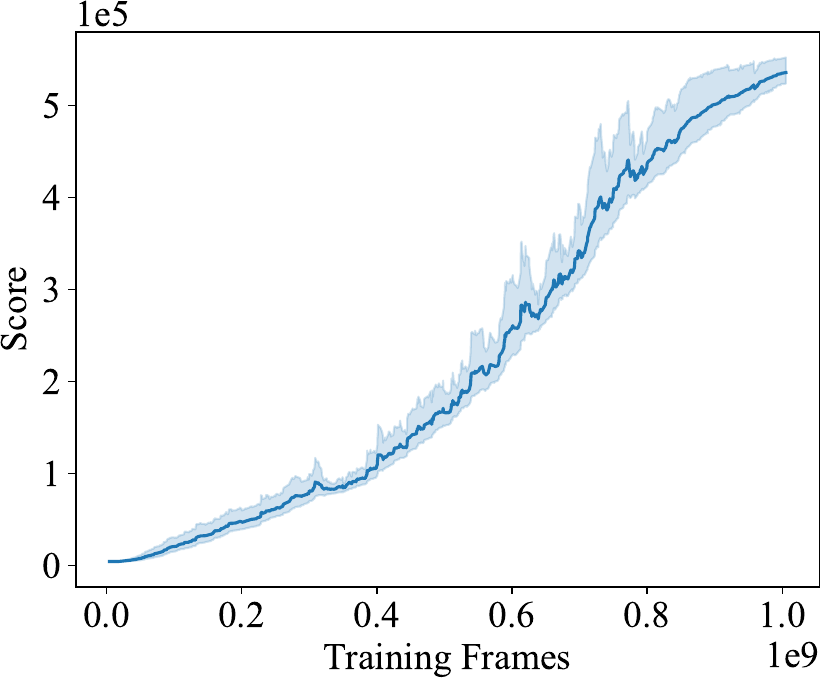}
    }
    \subfigure[Tutankham]{
    \includegraphics[width=0.3\textwidth,height=0.15\textheight]{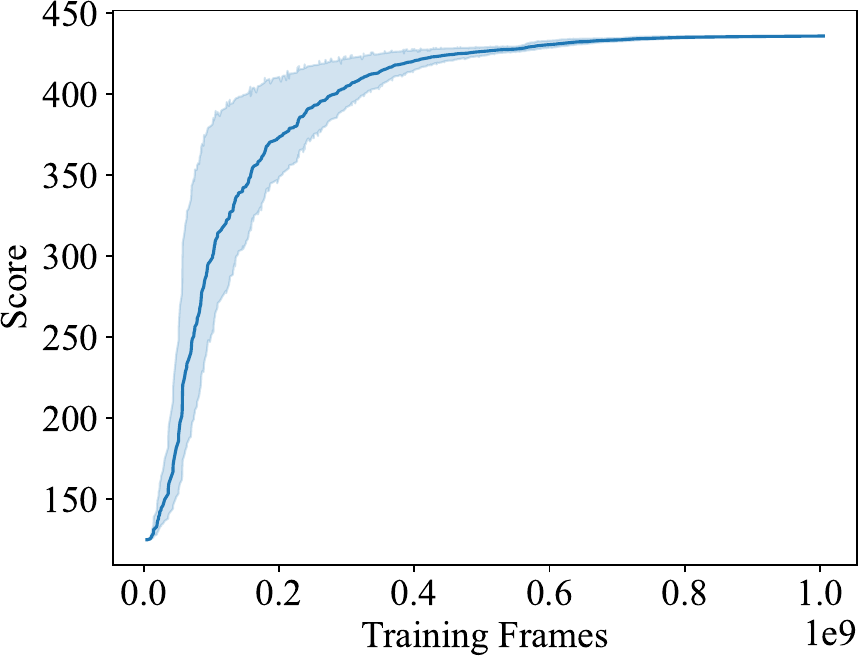}
    }
\end{figure}

\begin{figure}[!ht]
    \subfigure[Up\_N\_Down]{
    \includegraphics[width=0.3\textwidth,height=0.15\textheight]{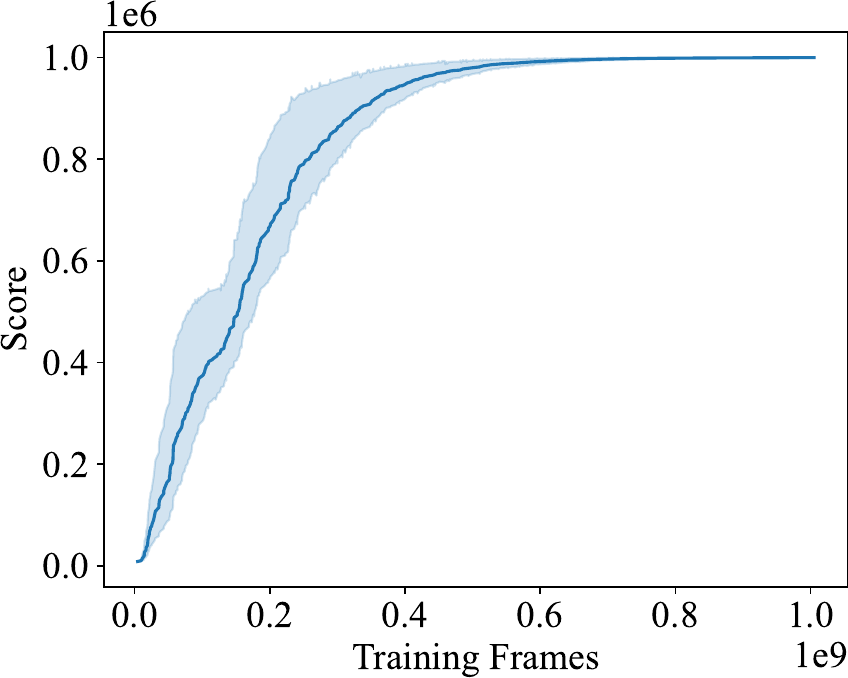}
    }
    \subfigure[Venture]{
    \includegraphics[width=0.3\textwidth,height=0.15\textheight]{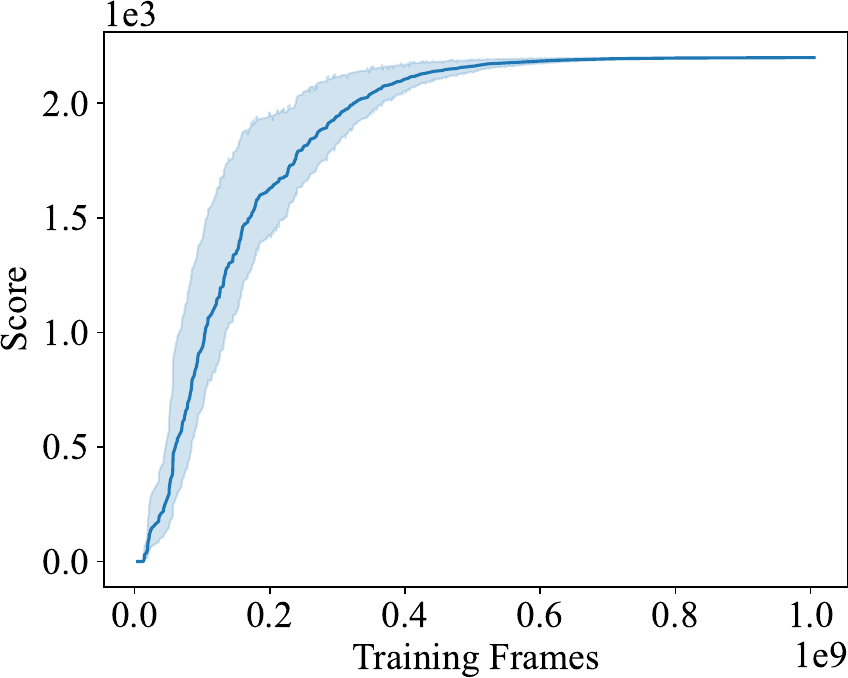}
    }
    \subfigure[Video\_Pinball]{
    \includegraphics[width=0.3\textwidth,height=0.15\textheight]{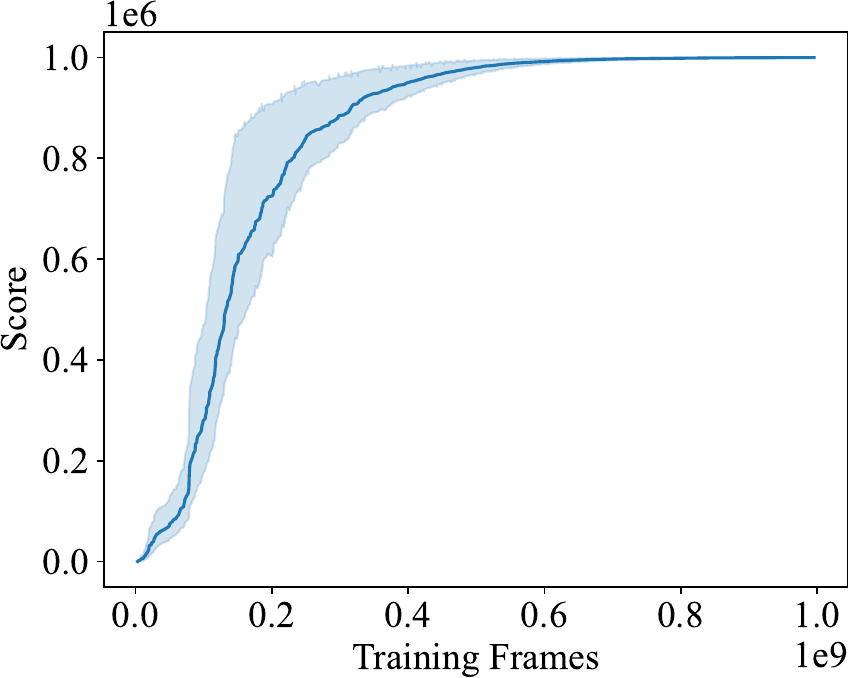}}
\end{figure}

\begin{figure}[!ht]
    \subfigure[Wizard\_of\_Wor]{
    \includegraphics[width=0.3\textwidth,height=0.15\textheight]{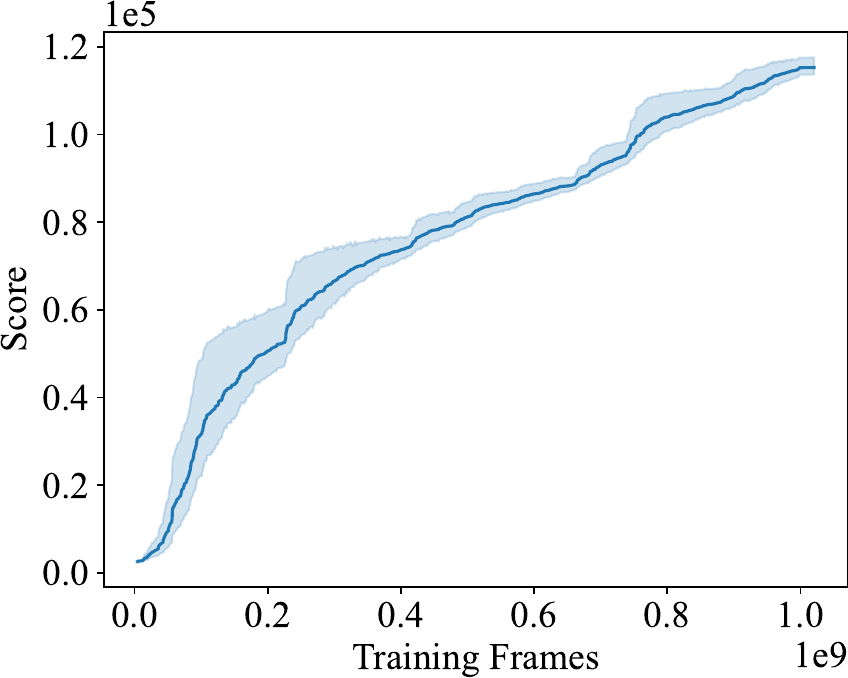}
    }
    \subfigure[Yars\_Revenge]{
    \includegraphics[width=0.3\textwidth,height=0.15\textheight]{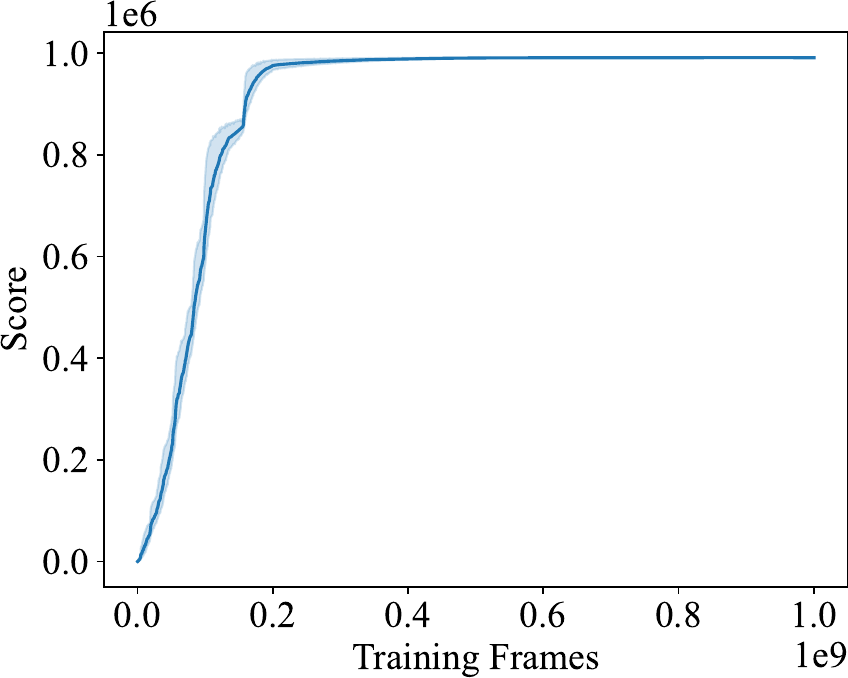}
    }
    \subfigure[Zaxxon]{
    \includegraphics[width=0.3\textwidth,height=0.15\textheight]{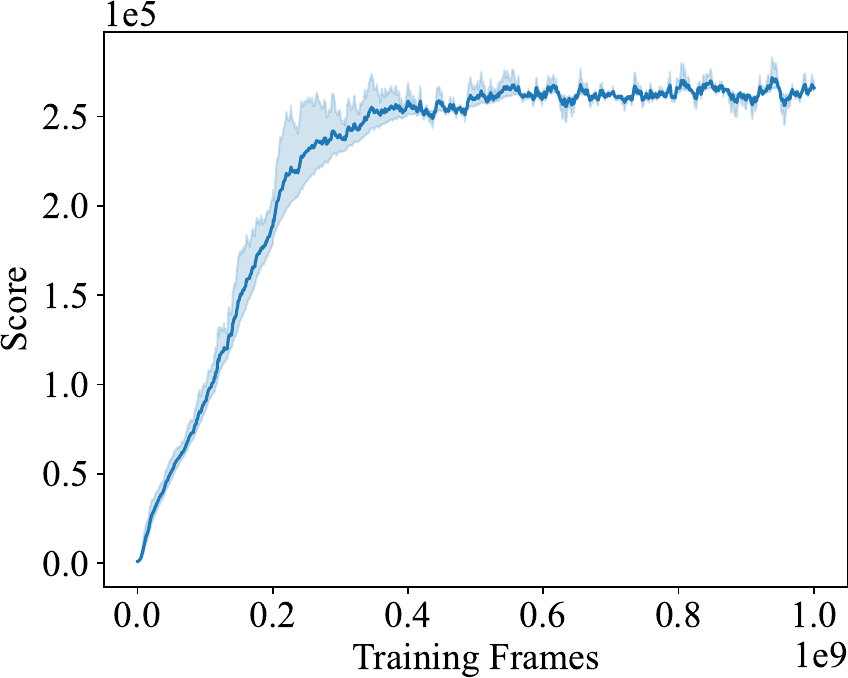}
    }
\end{figure}

\clearpage

%%%%%%%%%%%%%%%%%% define performance of agent57
\def\agentmeanhns{4762.17}
\def\agentmeanhnsle{4.76E-10}

\def\agentmedianhns{1933.49}
\def\agentmedianhnsle{1.93E-10}

\def\agentHWRB{18}
\def\agentHWRBle{1.80E-10}

\def\agentmeanHWRNS{125.92}
\def\agentmeanHWRNSle{1.26E-11}

\def\agentmedianHWRNS{43.62}
\def\agentmedianHWRNSle{4.36E-12}

\def\agentmeanSABER{76.26}
\def\agentmeanSABERle{7.63E-12}

\def\agentmedianSABER{43.62}
\def\agentmedianSABERle{4.36E-12}

\def\agentnumframes{1.00E+11}
\def\agentgametime{19290}

%%%%%%%%%%%%%%%%%% define performance of muzero
\def\muzeromeanhns{4994.97}
\def\muzeromeanhnsle{2.50E-09}

\def\muzeromedianhns{2041.12}
\def\muzeromedianhnsle{1.02E-09}

\def\muzeroHWRB{19}
\def\muzeroHWRBle{9.50E-10}

\def\muzeromeanHWRNS{152.10}
\def\muzeromeanHWRNSle{7.61E-11}

\def\muzeromedianHWRNS{49.80}
\def\muzeromedianHWRNSle{2.49E-11}

\def\muzeromeanSABER{71.94}
\def\muzeromeanSABERle{3.60E-11}

\def\muzeromedianSABER{49.80}
\def\muzeromedianSABERle{2.49E-11}

\def\muzeronumframes{2.00E+10}
\def\muzerogametime{3858}

%%%%%%%%%%%%%%%%%% define performance of dreamerv2
\def\dreamermeanhns{642.49}
\def\dreamermeanhnsle{3.21E-08}

\def\dreamermedianhns{178.04}
\def\dreamermedianhnsle{8.90E-09}

\def\dreamerHWRB{3}
\def\dreamerHWRBle{1.50E-08}

\def\dreamermeanHWRNS{38.60}
\def\dreamermeanHWRNSle{1.93E-09}

\def\dreamermedianHWRNS{4.29}
\def\dreamermedianHWRNSle{2.14E-10}

\def\dreamermeanSABER{27.73}
\def\dreamermeanSABERle{1.39E-09}

\def\dreamermedianSABER{4.29}
\def\dreamermedianSABERle{2.14E-10}

\def\dreamernumframes{2.00E+08}
\def\dreamergametime{38.58 }

%%%%%%%%%%%%%%%%%% define performance of simple
\def\simplemeanhns{194.3}
\def\simplemeanhnsle{2.58E-07}

\def\simplemedianhns{109}
\def\simplemedianhnsle{5.55E-08}

\def\simpleHWRB{0}
\def\simpleHWRBle{0.00E+00}

\def\simplemeanHWRNS{4.80}
\def\simplemeanHWRNSle{4.80E-08}

\def\simplemedianHWRNS{0.13}
\def\simplemedianHWRNSle{1.25E-09}

\def\simplemeanSABER{4.80}
\def\simplemeanSABERle{4.80E-08}

\def\simplemedianSABER{0.13}
\def\simplemedianSABERle{1.25E-09}

\def\simplenumframes{1.00E+06}
\def\simplegametime{0.19}

%%%%%%%%%%%%%%%%%% define performance of muesli
\def\mueslimeanhns{2538.12}
\def\mueslimeanhnsle{1.27E-07}

\def\mueslimedianhns{1077.47}
\def\mueslimedianhnsle{5.39E-08}

\def\muesliHWRB{5}
\def\muesliHWRBle{2.50E-08}

\def\mueslimeanHWRNS{75.52}
\def\mueslimeanHWRNSle{3.78E-09}

\def\mueslimedianHWRNS{24.86}
\def\mueslimedianHWRNSle{1.24E-09}

\def\mueslimeanSABER{48.74}
\def\mueslimeanSABERle{2.44E-09}

\def\mueslimedianSABER{24.86}
\def\mueslimedianSABERle{1.24E-09}

\def\mueslinumframes{2.00E+08}
\def\muesligametime{38.5}

%%%%%%%%%%%%%%%%%% define performance of go-explore
\def\goexploremeanhns{4989.31}
\def\goexploremeanhnsle{4.99E-09}

\def\goexploremedianhns{1451.55}
\def\goexploremedianhnsle{1.45E-09}

\def\goexploreHWRB{15}
\def\goexploreHWRBle{1.50E-09}

\def\goexploremeanHWRNS{116.89}
\def\goexploremeanHWRNSle{1.17E-10}

\def\goexploremedianHWRNS{50.50}
\def\goexploremedianHWRNSle{5.05E-11}

\def\goexploremeanSABER{71.80}
\def\goexploremeanSABERle{7.18E-11}

\def\goexploremedianSABER{50.50}
\def\goexploremedianSABERle{5.05E-11}

\def\goexplorenumframes{1.00E+10}
\def\goexploregametime{1929}

%%%%%%%%%%%%%%%%%% define performance of LBC H
\def\LBCHmeanhns{10077.52}
\def\LBCHmeanhnsle{4.81E-07 }

\def\LBCHmedianhns{1665.60}
\def\LBCHmedianhnsle{5.73E-08 }

\def\LBCHHWRB{24 }
\def\LBCHHWRBle{1.10E-07 }

\def\LBCHmeanHWRNS{154.27}
\def\LBCHmeanHWRNSle{7.71E-09 }

\def\LBCHmedianHWRNS{50.63}
\def\LBCHmedianHWRNSle{2.53E-09 }

\def\LBCHmeanSABER{71.26}
\def\LBCHmeanSABERle{3.56E-09 }

\def\LBCHmedianSABER{50.63}
\def\LBCHmedianSABERle{2.53E-09 }

\def\LBCHnumframes{1.00E+09 }
\def\LBCHgametime{192.5 }

%%%%%%%%%%%%%%%%%% define performance of meme
\def\mememeanhns{4081.14}

\def\mememedianhns{1225.19}

\def\memeHWRB{24 }
\def\memeHWRBle{1.10E-07 }

\def\mememeanHWRNS{154.27}
\def\mememeanHWRNSle{7.71E-09 }

\def\mememedianHWRNS{50.63}
\def\mememedianHWRNSle{2.53E-09 }

\def\mememeanSABER{71.26}
\def\mememeanSABERle{3.56E-09 }

\def\mememedianSABER{50.63}
\def\mememedianSABERle{2.53E-09 }

\def\memenumframes{1.00E+09 }
\def\memegametime{192.5 }

\subsection{Atari Games Table of Scores Based on Human Average Scores}
\label{app: Atari Games Table of Scores Based on Human Average Records}
We present the raw score of several typical SOTA algorithms, including model-free SOTA algorithms, model-based SOTA algorithms, and additional SOTA algorithms. We provide the Human Normalized Scores (HNS) for each algorithm in the Atari 57 games in addition to presenting the raw score for each game. More details on these algorithms can see \cite{ale2,atarihuman,atari_review}.

\clearpage

\begin{table}[!hb]
\footnotesize
\begin{center}
\caption{Score table of  SOTA  model-free algorithms on HNS(\%).}
\label{Tab:Score table of SOTA  model-free algorithms on HNS.}
\setlength{\tabcolsep}{1.0pt}
\begin{tabular}{c c c c c c c cc}
\toprule
Games & RND & Average Human & AGENT57 & HNS(\%) & Ours & HNS(\%) & MEME & HNS(\%) \\
        Scale & ~ & ~ & 100B & ~ & 1B & ~ & 1B & ~\\ \midrule
        Alien & 227.8 & 7127.8 & \textbf{297638.17} & \textbf{4310.30\%} & 279703.5 & 4050.37\% & 83683.43&	1209.50\%\\  
        Amidar & 5.8 & 1719.5 & \textbf{29660.08} & \textbf{1730.42\%} & 12996.3 & 758.04\%  & 14368.9	&838.13\%\\  
        Assault & 222.4 & 742 & \textbf{67212.67} & \textbf{12892.66\%} & 62025.7 & 11894.40\% & 46635.86	& 8932.54\%\\  
        Asterix & 210 & 8503.3 & 991384.42 & 11951.51\% & \textbf{999999} & \textbf{12055.38\%} & 769803.92	& 9279.71\%\\  
        Asteroids & 719 & 47388.7 & 150854.61 & 321.70\% & \textbf{1106603.5} & \textbf{2369.60\%} & 364492.07	& 779.46\%\\  
        Atlantis & 12850 & 29028.1 & 1528841.76 & 9370.64\% & \textbf{3824506.3} & \textbf{23560.59\%} & 1669226.33	& 10238.39\%\\  
        Bank Heist & 14.2 & 753.1 & 23071.5& 3120.49\% & 1410 & 188.90\% & \textbf{87792.55}	& \textbf{11879.60\%}\\  
        Battle Zone & 236 & 37187.5 & \textbf{934134.88} & \textbf{2527.36\%} & 857369 & 2319.62\% & 776770	& 2101.50\%\\  
        Beam Rider & 363.9 & 16926.5 & 300509.8 & 1812.19\% & \textbf{457321} & \textbf{2758.97\%} & 51870.2 &	310.98\%\\  
        Berzerk & 123.7 & 2630.4 & \textbf{61507.83} & \textbf{2448.80\%} & 35340 & 1404.89\% & 38838.35	& 1544.45\%\\  
        Bowling & 23.1 & 160.7 & 251.18 & 165.76\% & 233.1 & 152.62\% & \textbf{261.74} &	\textbf{173.43\%}\\  
        Boxing & 0.1 & 12.1 & \textbf{100} & \textbf{832.50}\% & \textbf{100} & \textbf{832.50\%} & 99.85 &	831.25\%\\  
        Breakout & 1.7 & 30.5 & 790.4 & 2738.54\% & \textbf{864} & \textbf{2994.10\%} & 831.08 &	2879.79\% \\  
        Centipede & 2090.9 & 12017 & 412847.86 & 4138.15\% & \textbf{728080} & \textbf{7313.94\%} & 245892.18 & 2456.16\%\\  
        Chopper Command & 811 & 7387.8 & 999900 & 15191.11\% & \textbf{999999} & \textbf{15192.62\%} & 912225	&13858.02\%\\  
        Crazy Climber & 10780.5 & 36829.4 & \textbf{565909.85} & \textbf{2131.10\%} & 233090 & 853.43\% & 339274.67	& 1261.07\%\\  
        Defender & 2874.5 & 18688.9 & 677642.78 & 4266.80\% & \textbf{995950} & \textbf{6279.56\%} & 543979.5	& 3421.60\%\\  
        Demon Attack & 152.1 & 1971 & 143161.44 & 7862.41\% & \textbf{900170} & \textbf{49481.44\%} & 142176.58	& 7808.26\%\\  
        Double Dunk & -18.6 & -16.4 & 23.93 & 1933.18\% & \textbf{24} & \textbf{1936.36\%} & 23.7	&1922.73\%\\  
        Enduro & 0 & 860.5 & 2367.71 & 275.16\% & \textbf{14332.5} & \textbf{1665.60\%} & 2360.64	&274.33\%\\  
        Fishing Derby & -91.7 & -38.8 & \textbf{86.97} & \textbf{337.75\%} & 75 & 315.12\% & 77.05	&319.00\%\\  
        Freeway & 0 & 29.6 & 32.59 & 110.10\% & \textbf{34} & \textbf{114.86\%} & 33.97	& 114.76\%\\  
        Frostbite & 65.2 & 4334.7 & \textbf{541280.88} & \textbf{12676.32\%} & 13792.4 & 321.52\% & 526239.5	& 12324.03\%\\  
        Gopher & 257.6 & 2412.5 & 117777.08 & 5453.59\% & \textbf{488900} & \textbf{22675.87\%} & 119457.53	& 5531.58\%\\  
        Gravitar & 173 & 3351.4 & 19213.96 & 599.07\% & 6372.5 & 195.05\% & \textbf{20875}	& \textbf{651.33\%}\\  
        Hero & 1027 & 30826.4 & 114736.26 & 381.58\% & 37545.6 & 122.55\% & \textbf{199880.6}	& \textbf{667.31\%}\\  
        Ice Hockey & -11.2 & 0.9 & \textbf{63.64} & \textbf{618.51\%} & 47.53 & 485.37\% & 47.22	& 482.81\%\\  
        Jamesbond & 29 & 302.8 & 135784.96 & 49582.16\% & \textbf{623300.5} & \textbf{227637.51\%} & 117009.92	&42724.95\%\\  
        Kangaroo & 52 & 3035 & \textbf{24034.16} & \textbf{803.96\%} & 14372.6 & 480.07\% & 17311.17 & 578.58\%\\  
        Krull & 1598 & 2665.5 & 251997.31 & 23456.61\% & \textbf{593679.5} & \textbf{55464.31\%} & 155915.32 & 14455.96\%\\  
        Kung Fu Master & 258.5 & 22736.3 & 206845.82 & 919.07\% & \textbf{1666665} & \textbf{7413.57\%} & 476539.53	& 2118.90\%\\  
        Montezuma Revenge & 0 & 4753.3 & 9352.01 & 196.75\%& 2500 & 52.60\% & \textbf{12437}	& \textbf{261.65\%}\\  
        Ms Pacman & 307.3 & 6951.6 & \textbf{63994.44} & \textbf{958.52\%} & 31403 & 468.01\% & 29747.91	& 443.10\%\\  
        Name This Game & 2292.3 & 8049 & 54386.77 & 904.94\% & \textbf{81473} & \textbf{1375.45\%} & 40077.73	& 656.37\%\\  
        Phoenix & 761.5 & 7242.6 & 908264.15 & 14002.29\% & \textbf{999999} & \textbf{15417.71\%} & 849969.25	&13102.83\%\\  
        Pitfall & -229.4 & 6463.7 & 18756.01 & 283.66\% & -1 & 3.41\% & \textbf{46734.79}	& \textbf{701.68\%}\\  
        Pong & -20.7 & 14.6 & 20.67 & 117.20\% & \textbf{21} & \textbf{118.13\%} & 19.31	& 113.34\%\\  
        Private Eye & 24.9 & 69571.3 & 79716.46 & 114.59\% & 15100 & 21.68\% & \textbf{100798.9}	& \textbf{144.90\%}\\  
        Qbert & 163.9 & 13455 & \textbf{580328.14} & \textbf{4365.06\%} & 151730 & 1140.36\% & 238453.5	& 1792.85\%\\  
        Riverraid & 1338.5 & 17118 & 63318.67 & 392.79\% & 27964.3 & 168.74\% & \textbf{90333.12}	& \textbf{563.99\%}\\  
        Road Runner & 11.5 & 7845 & 243025.8 & 3102.24\% & \textbf{999999} & \textbf{12765.53\%} & 399511.83	& 5099.90\%\\  
        Robotank & 2.2 & 11.9 & 127.32 & 1289.90\% & \textbf{144} & \textbf{1461.86\%} & 114.46	& 1157.32\%\\  
        Seaquest & 68.4 & 42054.7 & 999997.63 & 2381.56\% & \textbf{1000000} & \textbf{2381.57\%} & 960181.39	& 2286.73\%\\  
        Skiing & -17098 & -4336.9 & -4202.6 &101.05\%& -5903.34 & 87.72\% & \textbf{-3273.43}	&\textbf{108.33\%}\\  
        Solaris & 1236.3 & 12326.7 & \textbf{44199.93} & \textbf{387.39\%} & 10732.5 & 85.63\% & 28175.53	& 242.91\%\\  
        Space Invaders & 148 & 1668.7 & 48680.86 & 3191.48\% & \textbf{159999.6} & \textbf{10511.71\%} & 57828.45	& 3793.02\%\\  
        Star Gunner & 664 & 10250 & 839573.53 & 8751.40\% & \textbf{999999} & \textbf{10424.94\%} & 264286.33	&2750.08\%\\  
        Surround & -10 & 6.5 & 9.5 & 118.18\% & 2.726 & 77.13\% & \textbf{9.82}	&\textbf{120.12\%}\\  
        Tennis & -23.8 & -8.3 & 23.84 & 307.35\% & \textbf{24} & \textbf{308.39\%} & 22.79	&300.58\%\\  
        Time Pilot & 3568 & 5229.2 & 405425.31 & 24190.78\% & \textbf{531614} & \textbf{31787.02\%} & 404751.67	&24150.23\%\\  
        Tutankham & 11.4 & 167.6 & \textbf{2354.91} & \textbf{1500.33\%} & 436.2 & 271.96\% & 1030.27	& 652.29\%\\  
        Up N Down & 533.4 & 11693.2 & 623805.73 & 5584.98\% & \textbf{999999} & \textbf{8955.95\%} & 524631	&4696.30\%\\  
        Venture & 0 & 1187.5 & 2623.71 & 220.94\% & 2200 & 185.26\% & \textbf{2859.83}	&\textbf{240.83\%}\\  
        Video Pinball & 0 & 17667.9 & 992340.74 & 5616.63\% & \textbf{999999} & \textbf{5659.98\%} & 617640.95	&3495.84\%\\  
        Wizard of Wor & 563.5 & 4756.5 & \textbf{157306.41} & \textbf{3738.20\%} & 118900 & 2822.24\% & 71942	& 1702.33\%\\  
        Yars Revenge & 3092.9 & 54576.9 & 998532.37 & 1933.49\% & \textbf{998970} & \textbf{1934.34\%}& 633867.66	& 1225.19\% \\  
        Zaxxon & 32.5 & 9173.3 & \textbf{249808.9} & \textbf{2732.54\%} & 241570.6 & 2642.42\% & 77942.17	&852.33\%\\
        \midrule
         Mean HNS & & & &   \agentmeanhns \%  & & \textbf{\LBCHmeanhns\%}  & & \mememeanhns\\ 
         Median HNS & & & & \textbf{\agentmedianhns\%} & & \LBCHmedianhns\%  & & \mememedianhns\\ 
         \bottomrule
\end{tabular}
\end{center}
\end{table}

\clearpage

\begin{table}[!hb]
\footnotesize
\begin{center}
\caption{Score table of  SOTA  model-based algorithms on HNS(\%).}
\label{Tab:Score table of SOTA  model-based algorithms on HNS.}
\setlength{\tabcolsep}{1.0pt}
\begin{tabular}{c cc cc cc }
\toprule
 Games & MuZero & HNS(\%) & EfficientZero & HNS(\%) & Ours & HNS(\%) \\ 
        Scale & 20B & ~ & 100K & ~ & 1B & ~ \\ 
        \midrule
        
        Alien & \textbf{741812.63} & \textbf{10747.61\%} & 808.5 & 8.42\% & 279703.5 & 4050.37\% \\  
        Amidar & \textbf{28634.39} & \textbf{1670.57\%} & 148.6 & 8.33\% & 12996.3 & 758.04\% \\  
        Assault & \textbf{143972.03} & \textbf{27665.44\%} & 1263.1 & 200.29\% & 62025.7 & 11894.40\% \\  
        Asterix & 998425 & 12036.40\% & 25557.8 & 305.64\% & \textbf{999999} & \textbf{12055.38\%} \\  
        Asteroids & 678558.64 & 1452.42\% & N/A & N/A & \textbf{1106603.5} & \textbf{2369.60\%} \\  
        Atlantis & 1674767.2 & 10272.64\% & N/A & N/A & \textbf{3824506.3} & \textbf{23560.59\%} \\  
        Bank Heist & 1278.98 & 171.17\% & 351 & 45.58\% & \textbf{1410} & \textbf{188.90\%} \\  
        Battle Zone & 848623 & 2295.95\% & 13871.2 & 36.90\% & \textbf{857369} & \textbf{2319.62\%} \\  
        Beam Rider & 454993.53 & 2744.92\% & N/A & N/A & \textbf{457321} & \textbf{2758.97\%} \\  
        Berzerk & \textbf{85932.6} & \textbf{3423.18\%} & N/A & N/A & 35340 & 1404.89\% \\  
        Bowling & \textbf{260.13} & \textbf{172.26\%} & N/A & N/A & 233.1 & 152.62\% \\  
        Boxing & \textbf{100} & \textbf{832.50\%} & 52.7 & 438.33\% & \textbf{100} & \textbf{832.50\%} \\  
        Breakout & \textbf{864} & \textbf{2994.10\%} & 414.1 & 1431.94\% & \textbf{864} & \textbf{2994.10\%} \\  
        Centipede & \textbf{1159049.27} & \textbf{11655.72\%} & N/A & N/A & 728080 & 7313.94\% \\  
        Chopper Command & 991039.7 & 15056.39\% & 1117.3 & 4.66\% & \textbf{999999} & \textbf{15192.62\%} \\  
        Crazy Climber & \textbf{458315.4} & \textbf{1718.06\%} & 83940.2 & 280.86\% & 233090 & 853.43\% \\  
        Defender & 839642.95 & 5291.18\% & N/A & N/A & \textbf{995950} & \textbf{6279.56\%} \\  
        Demon Attack & 143964.26 & 7906.55\% & 13003.9 & 706.57\% & \textbf{900170} & \textbf{49481.44\%} \\  
        Double Dunk & 23.94 & 1933.64\% & N/A & N/A & \textbf{24} & \textbf{1936.36\%} \\  
        Enduro & 2382.44 & 276.87\% & N/A & N/A & \textbf{14332.5} & \textbf{1665.60\%} \\  
        Fishing Derby & \textbf{91.16} & \textbf{345.67\%} & N/A & N/A & 75 & 315.12\% \\  
        Freeway & 33.03 & 111.59\% & 21.8 & 73.65\% & \textbf{34} & \textbf{114.86\%} \\  
        Frostbite & \textbf{631378.53} & \textbf{14786.59\%} & 296.3 & 5.41\% & 13792.4 & 321.52\% \\  
        Gopher & 130345.58 & 6036.85\% & 3260.3 & 139.34\% & \textbf{488900} & \textbf{22675.87\%} \\  
        Gravitar & \textbf{6682.7} & \textbf{204.81\%} & N/A & N/A & 6372.5 & 195.05\% \\  
        Hero & \textbf{49244.11} & \textbf{161.81\%} & 3915.9 & 9.69\% & 37545.6 & 122.55\% \\  
        Ice Hockey & \textbf{67.04} & \textbf{646.61\%} & N/A & N/A & 47.53 & 485.37\% \\  
        Jamesbond & 41063.25 & 14986.94\% & 517 & 178.23\% & \textbf{623300.5} & \textbf{227637.51\%} \\  
        Kangaroo & \textbf{16763.6} & \textbf{560.23\%} & 724.1 & 22.53\% & 14372.6 & 480.07\% \\  
        Krull & 269358.27 & 25082.93\% & 5663.3 & 380.82\% & \textbf{593679.5} & \textbf{55464.31\%} \\  
        Kung Fu Master & 204824 & 910.08\% & 30944.8 & 136.52\% & \textbf{1666665} & \textbf{7413.57\%} \\  
        Montezuma Revenge & 0 & 0.00\% & N/A & N/A & \textbf{2500} & \textbf{52.60\%} \\  
        Ms Pacman & \textbf{243401.1} & \textbf{3658.68\%} & 1281.2 & 14.66\% & 31403 & 468.01\% \\  
        Name This Game & \textbf{157177.85} & \textbf{2690.53\%} & N/A & N/A & 81473 & 1375.45\% \\  
        Phoenix & 955137.84 & 14725.53\% & N/A & N/A & \textbf{999999} & \textbf{15417.71\%} \\  
        Pitfall & \textbf{0} & \textbf{3.43\%} & N/A & N/A & -1 & 3.41\% \\  
        Pong & \textbf{21} & \textbf{118.13\%} & 20.1 & 115.58\% & \textbf{21} & \textbf{118.13\%} \\  
        Private Eye & \textbf{15299.98} & \textbf{21.96\%} & 96.7 & 0.10\% & 15100 & 21.68\% \\  
        Qbert & 72276 & 542.56\% & 14448.5 & 107.47\% & \textbf{151730} & \textbf{1140.36\%} \\  
        Riverraid & \textbf{323417.18} & \textbf{2041.12\%} & N/A & N/A & 27964.3 & 168.74\% \\  
        Road Runner & 613411.8 & 7830.48\% & 17751.3 & 226.46\% & \textbf{999999} & \textbf{12765.53\%} \\  
        Robotank & 131.13 & 1329.18\% & N/A & N/A & \textbf{144} & \textbf{1461.86\%} \\  
        Seaquest & 999976.52 & 2381.51\% & 1100.2 & 2.46\% & \textbf{1000000}& \textbf{2381.57\%} \\  
        Skiing & -29968.36 & -100.86\% & N/A & N/A & \textbf{-5903.34} & \textbf{87.72\%} \\  
        Solaris & 56.62 & -10.64\% & N/A & N/A & \textbf{10732.5} & \textbf{85.63\%} \\  
        Space Invaders & 74335.3 & 4878.50\% & N/A & N/A & \textbf{159999.6} & \textbf{10511.71\%} \\  
        Star Gunner & 549271.7 & 5723.01\% & N/A & N/A & \textbf{999999} & \textbf{10424.94\%} \\  
        Surround & \textbf{9.99} & \textbf{121.15\%} & N/A & N/A & 2.726 & 77.13\% \\  
        Tennis & 0 & 153.55\% & N/A & N/A & \textbf{24} & \textbf{308.39\%} \\  
        Time Pilot & \textbf{476763.9} & \textbf{28485.19\%} & N/A & N/A & 531614 & 31787.02\% \\  
        Tutankham & \textbf{491.48} & \textbf{307.35\%} & N/A & N/A & 436.2 & 271.96\% \\  
        Up N Down & 715545.61 & 6407.03\% & 17264.2 & 149.92\% & \textbf{999999} & \textbf{8955.95\%} \\  
        Venture & 0.4 & 0.03\% & N/A & N/A & \textbf{2200} & \textbf{185.26\%} \\  
        Video Pinball & 981791.88 & 5556.92\% & N/A & N/A & \textbf{999999} & \textbf{5659.98\%} \\  
        Wizard of Wor & \textbf{197126} & \textbf{4687.87\%} & N/A & N/A & 118900 & 2822.24\% \\  
        Yars Revenge & 553311.46 & 1068.72\% & N/A & N/A & \textbf{998970} & \textbf{1934.34\%} \\  
        Zaxxon & \textbf{725853.9} & \textbf{7940.46\%} & N/A & N/A & 241570.6 & 2642.42\% \\ \midrule
         Mean HNS & &\muzeromeanhns\% & &   \simplemeanhns\%  & & \textbf{\LBCHmeanhns\%} \\ 
         Median HNS & &\textbf{\muzeromedianhns\%} & &  \simplemedianhns\% & & \LBCHmedianhns\%  \\
         \bottomrule
\end{tabular}
\end{center}
\end{table}

\clearpage  
\begin{table}[!hb]
\footnotesize
\begin{center}
\caption{Score table of  SOTA exploration-based algorithms on HNS(\%).}
\label{Tab:Score table of SOTA  exploration-based algorithms on HNS.}
\setlength{\tabcolsep}{1.0pt}
\begin{tabular}{c cc cc }
\toprule
Games & Go-Explore & HNS & Ours & HNS   \\ 
        Scale & 10B & ~ & 1B &    \\  \midrule
                Alien & \textbf{959312} & \textbf{13899.77\%} & 279703.5 & 4050.37\% \\  
        Amidar & \textbf{19083} & \textbf{1113.22\%} & 12996.3 & 758.04\% \\  
        Assault & 30773 & 5879.64\% & \textbf{62025.7} & \textbf{11894.40\%} \\  
        Asterix & 999500 & 12049.37\% & \textbf{999999} & \textbf{12055.38\%} \\  
        Asteroids & 112952 & 240.48\% & \textbf{1106603.5} & \textbf{2369.60\%} \\  
        Atlantis & 286460 & 1691.24\% & \textbf{3824506.3} & \textbf{23560.59\%} \\  
        Bank Heist & \textbf{3668} & \textbf{494.49\%} & 1410 & 188.90\% \\  
        Battle Zone & \textbf{998800} & \textbf{2702.36\%} & 857369 & 2319.62\% \\  
        Beam Rider & 371723 & 2242.15\% & \textbf{457321} & \textbf{2758.97\%} \\  
        Berzerk & \textbf{131417} & \textbf{5237.69\%} & 35340 & 1404.89\% \\  
        Bowling & \textbf{247} & \textbf{162.72\%} & 233.1 & 152.62\% \\  
        Boxing & 91 & 757.50\% & \textbf{100} & \textbf{832.50\%} \\  
        Breakout & 774 & 2681.60\% & \textbf{864} & \textbf{2994.10\%} \\  
        Centipede & 613815 & 6162.78\% & \textbf{728080} & \textbf{7313.94\%} \\  
        Chopper Command & 996220 & 15135.16\% & \textbf{999999} & \textbf{15192.62\%} \\  
        Crazy Climber & \textbf{235600} & \textbf{863.07\%} & 233090 & 853.43\% \\  
        Defender & N/A & N/A & \textbf{995950} & \textbf{6279.56\%} \\  
        Demon Attack & 239895 & 13180.65\% & \textbf{900170} & \textbf{49481.44\%} \\  
        Double Dunk & \textbf{24} & \textbf{1936.36\%} & \textbf{24} & \textbf{1936.36\%} \\  
        Enduro & 1031 & 119.81\% & \textbf{14332.5} & \textbf{1665.60\%} \\  
        Fishing Derby & 67 & 300.00\% & \textbf{75} & \textbf{315.12\%} \\  
        Freeway & \textbf{34} & \textbf{114.86\%} & \textbf{34} & \textbf{114.86\%} \\  
        Frostbite & \textbf{999990} & \textbf{23420.19\%} & 13792.4 & 321.52\% \\  
        Gopher & 134244 & 6217.75\% & \textbf{488900} & \textbf{22675.87\%} \\  
        Gravitar & \textbf{13385} & \textbf{415.68\%} & 6372.5 & 195.05\% \\  
        Hero & \textbf{37783} & \textbf{123.34\%} & 37545.6 & 122.55\% \\  
        Ice Hockey & 33 & 365.29\% & \textbf{47.53} & \textbf{485.37\%} \\  
        Jamesbond & 200810 & 73331.26\% & \textbf{623300.5} & \textbf{227637.51\%} \\  
        Kangaroo & \textbf{24300} & \textbf{812.87\%} & 14372.6 & 480.07\% \\  
        Krull & 63149 & 5765.90\% & \textbf{593679.5} &\textbf{ 55464.31\%} \\  
        Kung Fu Master & 24320 & 107.05\% & \textbf{1666665} & \textbf{7413.57\%} \\  
        Montezuma Revenge & \textbf{24758} & \textbf{520.86\%} & 2500 & 52.60\% \\  
        Ms Pacman & \textbf{456123} & \textbf{6860.25\%} & 31403 & 468.01\% \\  
        Name This Game & \textbf{212824} & \textbf{3657.16\%} & 81473 & 1375.45\% \\  
        Phoenix & 19200 & 284.50\% & \textbf{999999} & \textbf{15417.71\%} \\  
        Pitfall & \textbf{7875} & \textbf{121.09\%} & -1 & 3.41\% \\  
        Pong & \textbf{21} & \textbf{118.13\%} & \textbf{21} & \textbf{118.13\%} \\  
        Private Eye &\textbf{ 69976} & \textbf{100.58\%} & 15100 & 21.68\% \\  
        Qbert & \textbf{999975} & \textbf{7522.41\%} & 151730 & 1140.36\% \\  
        Riverraid & \textbf{35588} & \textbf{217.05\%} & 27964.3 & 168.74\% \\  
        Road Runner & 999900 & 12764.26\% & \textbf{999999} & \textbf{12765.53\%} \\  
        Robotank & 143 & 1451.55\% & \textbf{144} & \textbf{1461.86\%} \\  
        Seaquest & 539456 & 1284.68\% & \textbf{1000000} & \textbf{2381.57\%} \\  
        Skiing & \textbf{-4185} & \textbf{101.19\%} & -5903.34 & 87.72\% \\  
        Solaris & \textbf{20306} & \textbf{171.95\%} & 10732.5 & 85.63\% \\  
        Space Invaders & 93147 & 6115.54\% & \textbf{159999.6} & \textbf{10511.71\%} \\  
        Star Gunner & 609580 & 6352.14\% & \textbf{999999} & \textbf{10424.94\%} \\  
        Surround & N/A & N/A & \textbf{2.726} & \textbf{77.13\%} \\  
        Tennis & \textbf{24} & \textbf{308.39\%} & \textbf{24} & \textbf{308.39\%} \\  
        Time Pilot & 183620 & 10838.67\% & \textbf{531614} & \textbf{31787.02\%} \\  
        Tutankham & \textbf{528} & \textbf{330.73\%} & 436.2 & 271.96\% \\  
        Up N Down & 553718 & 4956.94\% & \textbf{999999} & \textbf{8955.95\%} \\  
        Venture & \textbf{3074} & \textbf{258.86\%} & 2200 & 185.26\% \\  
        Video Pinball & \textbf{999999} & \textbf{5659.98\%} & \textbf{999999} & \textbf{5659.98\%} \\  
        Wizard of Wor & \textbf{199900} & \textbf{4754.03}\% & 118900 & 2822.24\% \\  
        Yars Revenge & \textbf{999998} & \textbf{1936.34\%} & 998970 & 1934.34\% \\  
        Zaxxon & 18340 & 200.28\% & \textbf{241570.6} & \textbf{2642.42\%} \\ \midrule
                 Mean HNS & &\goexploremeanhns\%   & & \textbf{\LBCHmeanhns\%} \\ 
         Median HNS & &\goexploremedianhns\%& & \textbf{\LBCHmedianhns\%}  \\ 
         \bottomrule
\end{tabular}
\end{center}
\end{table}

\clearpage

\begin{table}[!hb]
\footnotesize
\begin{center}
\caption{Score Table of GDI-H$^3$ and LBC-$\mathcal{BM}$ (Ours) on HNS(\%).}
\label{Tab:Score table of GDI and LBC on HNS.}
\setlength{\tabcolsep}{1.0pt}
\begin{tabular}{c cc cc }
\toprule
Games & GDI-H$^3$ & HNS & Ours & HNS   \\ 
        Scale & 200M & ~ & 1B &    \\  \midrule
        Alien & 48735	             &703.00\% & \textbf{279703.5} & \textbf{4050.37\%} \\  
        Amidar &1065              &61.81\% & \textbf{12996.3} & \textbf{758.04\%} \\  
        Assault &\textbf{97155}	             &\textbf{18655.23\%}  & 62025.7 & 11894.40\% \\  
        Asterix &\textbf{{999999}}   &\textbf{{12055.38\%}} & \textbf{999999} & \textbf{12055.38\%} \\  
        Asteroids &{760005}            &{1626.94\%}  & \textbf{1106603.5} & \textbf{2369.60\%} \\  
        Atlantis &\textbf{{3837300}}           &\textbf{{23639.67\%}}   & 3824506.3 & 23560.59\% \\  
        Bank Heist &1380              &184.84\%  & \textbf{1410} & \textbf{188.90\%} \\  
        Battle Zone &824360            &2230.29\% & \textbf{857369} & \textbf{2319.62\%} \\  
        Beam Rider &422390            &2548.07\% & \textbf{457321} & \textbf{2758.97\%} \\  
        Berzerk &14649             &579.46\% & \textbf{35340} & \textbf{1404.89\%} \\  
        Bowling &205.2             &132.34\% & \textbf{233.1} & \textbf{152.62\%} \\  
        Boxing &\textbf{{100}}      &\textbf{{832.50\%}}   & \textbf{100} & \textbf{832.50\%} \\  
        Breakout &\textbf{{864}}      &\textbf{{2994.10\%}} & \textbf{864} & \textbf{2994.10\%} \\  
        Centipede &195630            &1949.80\% & \textbf{728080} & \textbf{7313.94\%} \\  
        Chopper Command &\textbf{{999999}}   &\textbf{{15192.62\%}} & \textbf{999999} & \textbf{15192.62\%} \\  
        Crazy Climber &\textbf{241170}	            &\textbf{919.76\%} & 233090 & 853.43\% \\  
        Defender &{970540}   &{6118.89\%} & \textbf{995950} & \textbf{6279.56\%} \\  
        Demon Attack &{787985}    &{43313.70\%} & \textbf{900170} & \textbf{49481.44\%} \\  
        Double Dunk &\textbf{{24 }}      &\textbf{{1936.36\%}} & \textbf{24} & \textbf{1936.36\%} \\  
        Enduro &14300             &1661.82\% & \textbf{14332.5} & \textbf{1665.60\%} \\  
        Fishing Derby &65               &296.22\% & \textbf{75} & \textbf{315.12\%} \\  
        Freeway &\textbf{{34}}      &\textbf{{114.86\%}} & \textbf{34} & \textbf{114.86\%} \\  
        Frostbite &11330	            &263.84\% & \textbf{13792.4} & \textbf{321.52\%} \\  
        Gopher &473560           &21964.01\% & \textbf{488900} & \textbf{22675.87\%} \\  
        Gravitar  &5915             &180.66\% & \textbf{6372.5} & \textbf{195.05\%} \\  
        Hero &\textbf{38225}	            &\textbf{124.83\%} & 37545.6 & 122.55\% \\  
        Ice Hockey &47.11           &481.90\% & \textbf{47.53} & \textbf{485.37\%} \\  
        Jamesbond &{620780	}  &{226716.95\%} & \textbf{623300.5} & \textbf{227637.51\%} \\  
        Kangaroo &14636           &488.90\% & 14372.6 & 480.07\% \\  
        Krull &{594540}          &{55544.92\%} & 593679.5 &55464.31\% \\  
        Kung Fu Master &\textbf{{1666665}}	          &\textbf{{7413.57\%}} & \textbf{1666665} & \textbf{7413.57\%} \\  
        Montezuma Revenge &\textbf{2500}            &\textbf{52.60\%} & \textbf{2500} & \textbf{52.60\%} \\  
        Ms Pacman &\textbf{11573}           &\textbf{169.55\%} & \textbf{31403} & \textbf{468.01\%} \\  
        Name This Game  &36296           &590.68\% & \textbf{81473} & \textbf{1375.45\%} \\  
        Phoenix &{959580 }         &	{14794.07\%} & \textbf{999999} & \textbf{15417.71\%} \\
        Pitfall  &-4.3            &3.36 & \textbf{-1} & \textbf{3.41\%} \\  
        Pong &{21}     &{118.13\%}  & \textbf{21} & \textbf{118.13\%} \\  
        Private Eye &15100           &21.68\% & 15100 & 21.68\% \\  
        Qbert &28657           &214.38\% & \textbf{151730} & \textbf{1140.36\%} \\  
        Riverraid &\textbf{28349}           &\textbf{171.17\%} & 27964.3 & 168.74\% \\  
        Road Runner &{999999	} &{12765.53\%} & \textbf{999999} & \textbf{12765.53\%} \\  
        Robotank &113.4           &1146.39\% & \textbf{144} & \textbf{1461.86\%} \\  
        Seaquest&\textbf{{1000000}}          &\textbf{{2381.57\%}} & \textbf{1000000} & \textbf{2381.57\%} \\  
        Skiing &\textbf{{-6025}}  &\textbf{{86.77\%}} & -5903.34 & 87.72\% \\  
        Solaris &9105            &70.95\% & \textbf{10732.5} & \textbf{85.63\%} \\  
        Space Invaders  &{154380}          &{10142.17\%} & \textbf{159999.6} & \textbf{10511.71\%} \\  
        Star Gunner &{677590}          &{7061.61\%} & \textbf{999999} & \textbf{10424.94\%} \\ 
        Surround  &2.606           &76.40\%  & \textbf{2.726} & \textbf{77.13\%} \\  
        Tennis &\textbf{{24}}              &\textbf{{308.39\%}}   & \textbf{24} & \textbf{308.39\%} \\  
        Time Pilot &450810	          &26924.45\%  & \textbf{531614} & \textbf{31787.02\%} \\  
        Tutankham &418.2           &260.44\%  & \textbf{436.2} & \textbf{271.96\%} \\  
        Up N Down &966590          &8656.58\% & \textbf{999999} & \textbf{8955.95\%} \\  
        Venture  &2000	            &168.42\% & \textbf{2200} & \textbf{185.26\%} \\  
        Video Pinball &978190          &5536.54\% & \textbf{999999} & \textbf{5659.98\%} \\  
        Wizard of Wor  &63735           &1506.59\% & \textbf{118900} & \textbf{2822.24\%} \\  
        Yars Revenge &968090          &1874.36\% & \textbf{998970} & \textbf{1934.34\%} \\  
        Zaxxon &216020	          &2362.89\% & \textbf{241570.6} & \textbf{2642.42\%} \\ \midrule
                 Mean HNS & &\goexploremeanhns\%   & & \textbf{\LBCHmeanhns\%} \\ 
         Median HNS & &\goexploremedianhns\%& & \textbf{\LBCHmedianhns\%}  \\ 
         \bottomrule
\end{tabular}
\end{center}
\end{table}

\clearpage

\subsection{Atari Games Table of Scores Based on Human World Records}
\label{app: Atari Games Table of Scores Based on Human World Records}

The raw score of numerous typical SOTA algorithms, including model-free SOTA algorithms, model-based SOTA algorithms, and additional SOTA algorithms, is described in this section. In addition to the raw score, we also include the Human World Records and Breakthroughs (HWRB) for each Atari 57 game, as well as the individual game scores. You may get more information about these algorithms at \cite{ale2,atarihuman}.

\clearpage

\begin{table}[!hb]
\footnotesize
\begin{center}
\caption{Score table of  SOTA  model-free algorithms on HWRB.}
\label{Tab:Score table of SOTA  model-free algorithms on HWRB.}
\setlength{\tabcolsep}{1.0pt}
\begin{tabular}{c cc cc cc cc}
\toprule
Games & RND & Human World Records & AGENT57 & HWRB & Ours & HWRB & MEME & HWRB \\ 
        Scale & ~ & ~ & 100B & ~ & 1B & ~ & 1B & ~ \\ \midrule
        Alien & 227.8 & 251916 & 297638.17 & 1  & 279703.5 & 1  & 83683.43	&0\\  
        Amidar & 5.8 & 104159 & 29660.08 & 0  & 12996.3 & 0  & 14368.9	& 0\\  
        Assault & 222.4 & 8647 & 67212.67 & 1  & 62025.7 & 1 & 46635.86	&1 \\  
        Asterix & 210 & 1000000 & 991384.42 & 0  & 999999 & 0  & 769803.92	& 0\\  
        Asteroids & 719 & 10506650 & 150854.61 & 0  & 1106603.5 & 0  & 364492.07	&0\\  
        Atlantis & 12850 & 10604840 & 1528841.76 & 0  & 3824506.3 & 0  &1669226.33	&0\\  
        Bank Heist & 14.2 & 82058 & 23071.5 & 0  & 1410 & 0  & 87792.55	&1\\  
        Battle Zone & 236 & 801000 & 934134.88 & 1  & 857369 & 1  & 776770	&0\\  
        Beam Rider & 363.9 & 999999 & 300509.8 & 0  & 457321 & 0  & 51870.2	&0\\  
        Berzerk & 123.7 & 1057940 & 61507.83 & 0  & 35340 & 0  & 38838.35	&0\\  
        Bowling & 23.1 & 300 & 251.18 & 0  & 233.1 & 0  &261.74	& 0\\  
        Boxing & 0.1 & 100 & 100 & 1  & 100 & 1  & 99.85	&0\\  
        Breakout & 1.7 & 864 & 790.4 & 0  & 864 & 1  & 831.08	& 0\\  
        Centipede & 2090.9 & 1301709 & 412847.86 & 0  & 728080 & 0  & 245892.18	& 0\\  
        Chopper Command & 811 & 999999 & 999900 & 0  & 999999 & 1  & 912225	&0\\  
        Crazy Climber & 10780.5 & 219900 & 565909.85 & 1  & 233090 & 1  & 339274.67	&1\\  
        Defender & 2874.5 & 6010500 & 677642.78 & 0  & 995950 & 0  &543979.5	&0\\  
        Demon Attack & 152.1 & 1556345 & 143161.44 & 0  & 900170 & 0 & 142176.58	&0 \\  
        Double Dunk & -18.6 & 21 & 23.93 & 1  & 24 & 1 &23.7	&1 \\  
        Enduro & 0 & 9500 & 2367.71 & 0  & 14332.5 & 1  &2360.64	&0\\  
        Fishing Derby & -91.7 & 71 & 86.97 & 1  & 75 & 1  & 77.05	&1\\  
        Freeway & 0 & 38 & 32.59 & 0  & 34 & 0  & 33.97	& 0\\  
        Frostbite & 65.2 & 454830 & 541280.88 & 1  & 13792.4 & 0  & 526239.5	& 1\\  
        Gopher & 257.6 & 355040 & 117777.08 & 0  & 488900 & 1 & 119457.53	& 0 \\  
        Gravitar & 173 & 162850 & 19213.96 & 0  & 6372.5 & 0  & 20875	& 0\\  
        Hero & 1027 & 1000000 & 114736.26 & 0  & 37545.6 & 0  & 199880.6	& 0\\  
        Ice Hockey & -11.2 & 36 & 63.64 & 1  & 47.53 & 1  & 47.22	& 1\\  
        Jamesbond & 29 & 45550 & 135784.96 & 1  & 623300.5 & 1  & 117009.92	& 1\\  
        Kangaroo & 52 & 1424600 & 24034.16 & 0  & 14372.6 & 0  & 17311.17	& 0\\  
        Krull & 1598 & 104100 & 251997.31 & 1  & 593679.5 & 1  & 155915.32	& 1\\  
        Kung Fu Master & 258.5 & 1000000 & 206845.82 & 0  & 1666665 & 1  & 476539.53	& 0\\  
        Montezuma Revenge & 0 & 1219200 & 9352.01 & 0  & 2500 & 0 & 12437	& 0 \\  
        Ms Pacman & 307.3 & 290090 & 63994.44 & 0  & 31403 & 0  & 29747.91	& 0 \\  
        Name This Game & 2292.3 & 25220 & 54386.77 & 1  & 81473 & 1 & 40077.73	& 1 \\  
        Phoenix & 761.5 & 4014440 & 908264.15 & 0  & 999999 & 0  & 849969.25	& 0\\  
        Pitfall & -229.4 & 114000 & 18756.01 & 0  & -1 & 0 & 46734.79	& 0 \\  
        Pong & -20.7 & 21 & 20.67 & 0  & 21 & 1  & 19.31	& 0\\  
        Private Eye & 24.9 & 101800 & 79716.46 & 0  & 15100 & 0 & 100798.9	& 0 \\  
        Qbert & 163.9 & 2400000 & 580328.14 & 0  & 151730 & 0  & 238453.5	& 0\\  
        Riverraid & 1338.5 & 1000000 & 63318.67 & 0  & 27964.3 & 0 & 90333.12	& 0\\  
        Road Runner & 11.5 & 2038100 & 243025.8 & 0  & 999999 & 0  & 399511.83	& 0\\  
        Robotank & 2.2 & 76 & 127.32 & 1  & 144 & 1  & 114.46	& 1\\  
        Seaquest & 68.4 & 999999 & 999997.63 & 0  & 1000000 & 1 & 960181.39	& 0 \\  
        Skiing & -17098 & -3272 & -4202.6 & 0  & -5903.34 & 0 & -3273.43	& 0 \\  
        Solaris & 1236.3 & 111420 & 44199.93 & 0  & 10732.5 & 0  & 28175.53	& 0\\  
        Space Invaders & 148 & 621535 & 48680.86 & 0  & 159999.6 & 0 & 57828.45	& 0 \\  
        Star Gunner & 664 & 77400 & 839573.53 & 1  & 999999 & 1 & 264286.33	& 1 \\  
        Surround & -10 & 9.6 & 9.5 & 0  & 2.726 & 0  & 9.82	&1\\  
        Tennis & -23.8 & 21 & 23.84 & 1  & 24 & 1  & 22.79	&1\\  
        Time Pilot & 3568 & 65300 & 405425.31 & 1  & 531614 & 1 & 404751.67	& 1 \\  
        Tutankham & 11.4 & 5384 & 2354.91 & 0  & 436.2 & 0  & 1030.27	&0\\  
        Up n Down & 533.4 & 82840 & 623805.73 & 1  & 999999 & 1 & 524631	& 1 \\  
        Venture & 0 & 38900 & 2623.71 & 0  & 2200 & 0  & 2859.83	& 0\\  
        Video Pinball & 0 & 89218328 & 992340.74 & 0  & 999999 & 0 & 617640.95	& 0 \\  
        Wizard of Wor & 563.5 & 395300 & 157306.41 & 0  & 118900 & 0  & 71942	& 0\\  
        Yars Revenge & 3092.9 & 15000105 & 998532.37 & 0  & 998970 & 0 & 633867.66	& 0  \\  
        Zaxxon & 32.5 & 83700 & 249808.9 & 1  & 241570.6 & 1  & 77942.17	&0\\ 
        \midrule
        $\sum$ HWRB & & & & 18 & & \textbf{24}& & 16 \\ 
        \bottomrule
\end{tabular}
\end{center}
\end{table}

\clearpage

\begin{table}[!hb]
\footnotesize
\begin{center}
\caption{Score table of  SOTA  model-based algorithms on HWRB.}
\label{Tab:Score table of SOTA  model-based algorithms on HWRB.}
\setlength{\tabcolsep}{1.0pt}
\begin{tabular}{c cc cc cc}
\toprule
Games & MuZero &  HWRB & EfficientZero &  HWRB &Ours  &  HWRB \\ 
        Scale & 20B & ~ & 100K & ~ & 1B & ~ \\ \midrule
        Alien & 741812.63 & 1  & 808.5 & 0  & 279703.5 & 1  \\  
        Amidar & 28634.39 & 0  & 148.6 & 0  & 12996.3 & 0  \\  
        Assault & 143972.03 & 1  & 1263.1 & 0  & 62025.7 & 1  \\  
        Asterix & 998425 & 0  & 25557.8 & 0  & 999999 & 0  \\  
        Asteroids & 678558.64 & 0  & N/A & N/A & 1106603.5 & 0  \\  
        Atlantis & 1674767.2 & 0  & N/A & N/A & 3824506.3 & 0  \\  
        Bank Heist & 1278.98 & 0  & 351 & 0  & 1410 & 0  \\  
        Battle Zone & 848623 & 1  & 13871.2 & 0  & 857369 & 1  \\  
        Beam Rider & 454993.53 & 0  & N/A & N/A & 457321 & 0  \\  
        Berzerk & 85932.6 & 0  & N/A & N/A & 35340 & 0  \\  
        Bowling & 260.13 & 0  & N/A & N/A & 233.1 & 0  \\  
        Boxing & 100 & 1  & 52.7 & 0  & 100 & 1  \\  
        Breakout & 864 & 1  & 414.1 & 0  & 864 & 1  \\  
        Centipede & 1159049.27 & 0  & N/A & N/A & 728080 & 0  \\  
        Chopper Command & 991039.7 & 0  & 1117.3 & 0  & 999999 & 1  \\  
        Crazy Climber & 458315.4 & 1  & 83940.2 & 0  & 233090 & 1  \\  
        Defender & 839642.95 & 0  & N/A & N/A & 995950 & 0  \\  
        Demon Attack & 143964.26 & 0  & 13003.9 & 0  & 900170 & 0  \\  
        Double Dunk & 23.94 & 1  & N/A & N/A & 24 & 1  \\  
        Enduro & 2382.44 & 0  & N/A & N/A & 14332.5 & 1  \\  
        Fishing Derby & 91.16 & 1  & N/A & N/A & 75 & 1  \\  
        Freeway & 33.03 & 0  & 21.8 & 0  & 34 & 0  \\  
        Frostbite & 631378.53 & 1  & 296.3 & 0  & 13792.4 & 0  \\  
        Gopher & 130345.58 & 0  & 3260.3 & 0  & 488900 & 1  \\  
        Gravitar & 6682.7 & 0  & N/A & N/A & 6372.5 & 0  \\  
        Hero & 49244.11 & 0  & 3915.9 & 0  & 37545.6 & 0  \\  
        Ice Hockey & 67.04 & 1  & N/A & N/A & 47.53 & 1  \\  
        Jamesbond & 41063.25 & 0  & 517 & 0  & 623300.5 & 1  \\  
        Kangaroo & 16763.6 & 0  & 724.1 & 0  & 14372.6 & 0  \\  
        Krull & 269358.27 & 1  & 5663.3 & 0  & 593679.5 & 1  \\  
        Kung Fu Master & 204824 & 0  & 30944.8 & 0  & 1666665 & 1  \\  
        Montezuma Revenge & 0 & 0  & N/A & N/A & 2500 & 0  \\  
        Ms Pacman & 243401.1 & 0  & 1281.2 & 0  & 31403 & 0  \\  
        Name This Game & 157177.85 & 1  & N/A & N/A & 81473 & 1  \\  
        Phoenix & 955137.84 & 0  & N/A & N/A & 999999 & 0  \\  
        Pitfall & 0 & 0  & N/A & N/A & -1 & 0  \\  
        Pong & 21 & 1  & 20.1 & 0  & 21 & 1  \\  
        Private Eye & 15299.98 & 0  & 96.7 & 0  & 15100 & 0  \\  
        Qbert & 72276 & 0  & 14448.5 & 0  & 151730 & 0  \\  
        Riverraid & 323417.18 & 0  & N/A & N/A & 27964.3 & 0  \\  
        Road Runner & 613411.8 & 0  & 17751.3 & 0  & 999999 & 0  \\  
        Robotank & 131.13 & 1  & N/A & N/A & 144 & 1  \\  
        Seaquest & 999976.52 & 0  & 1100.2 & 0  & 1000000 & 1  \\  
        Skiing & -29968.36 & 0  & N/A & N/A & -5903.34 & 0  \\  
        Solaris & 56.62 & 0  & N/A & N/A & 10732.5 & 0  \\  
        Space Invaders & 74335.3 & 0  & N/A & N/A & 159999.6 & 0  \\  
        Star Gunner & 549271.7 & 1  & N/A & N/A & 999999 & 1  \\  
        Surround & 9.99 & 1  & N/A & N/A & 2.726 & 0  \\  
        Tennis & 0 & 0  & N/A & N/A & 24 & 1  \\  
        Time Pilot & 476763.9 & 1  & N/A & N/A & 531614 & 1  \\  
        Tutankham & 491.48 & 0  & N/A & N/A & 436.2 & 0  \\  
        Up N Down & 715545.61 & 1  & 17264.2 & 0  & 999999 & 1  \\  
        Venture & 0.4 & 0  & N/A & N/A & 2200 & 0  \\  
        Video Pinball & 981791.88 & 0  & N/A & N/A & 999999 & 0  \\  
        Wizard of Wor & 197126 & 0  & N/A & N/A & 118900 & 0  \\  
        Yars Revenge & 553311.46 & 0  & N/A & N/A & 998970 & 0  \\  
        Zaxxon & 725853.9 & 1  & N/A & N/A & 241570.6 & 1 \\ 
        \midrule
         $\sum$ HWRB & & 19 & & 0 & & \textbf{24} \\
         \bottomrule
\end{tabular}
\end{center}
\end{table}

\clearpage  
\begin{table}[!hb]
\footnotesize
\begin{center}
\caption{Score table of  SOTA exploration-based algorithms on HWRB.}
\label{Tab:Score table of SOTA  exploration-based algorithms on HWRB.}
\setlength{\tabcolsep}{1.0pt}
\begin{tabular}{c cc cc }
\toprule
Games & Go-Explore & HWRB & Ours & HWRB \\ 
        Scale & 10B & ~ & 1B &   \\ \midrule
        Alien & 959312 & 1  & 279703.5 & 1  \\  
        Amidar & 19083 & 0  & 12996.3 & 0  \\  
        Assault & 30773 & 1  & 62025.7 & 1  \\  
        Asterix & 999500 & 0  & 999999 & 0  \\  
        Asteroids & 112952 & 0  & 1106603.5 & 0  \\  
        Atlantis & 286460 & 0  & 3824506.3 & 0  \\  
        Bank Heist & 3668 & 0  & 1410 & 0  \\  
        Battle Zone & 998800 & 1  & 857369 & 1  \\  
        Beam Rider & 371723 & 0  & 457321 & 0  \\  
        Berzerk & 131417 & 0  & 35340 & 0  \\  
        Bowling & 247 & 0  & 233.1 & 0  \\  
        Boxing & 91 & 0  & 100 & 1  \\  
        Breakout & 774 & 0  & 864 & 1  \\  
        Centipede & 613815 & 0  & 728080 & 0  \\  
        Chopper Command & 996220 & 0  & 999999 & 1  \\  
        Crazy Climber & 235600 & 1  & 233090 & 1  \\  
        Defender & N/A & N/A & 995950 & 0  \\  
        Demon Attack & 239895 & 0  & 900170 & 0  \\  
        Double Dunk & 24 & 1  & 24 & 1  \\  
        Enduro & 1031 & 0  & 14332.5 & 1  \\  
        Fishing Derby & 67 & 0  & 75 & 1  \\  
        Freeway & 34 & 0  & 34 & 0  \\  
        Frostbite & 999990 & 1  & 13792.4 & 0  \\  
        Gopher & 134244 & 0  & 488900 & 1  \\  
        Gravitar & 13385 & 0  & 6372.5 & 0  \\  
        Hero & 37783 & 0  & 37545.6 & 0  \\  
        Ice Hockey & 33 & 0  & 47.53 & 1  \\  
        Jamesbond & 200810 & 1  & 623300.5 & 1  \\  
        Kangaroo & 24300 & 0  & 14372.6 & 0  \\  
        Krull & 63149 & 0  & 593679.5 & 1  \\  
        Kung Fu Master & 24320 & 0  & 1666665 & 1  \\  
        Montezuma Revenge & 24758 & 0  & 2500 & 0  \\  
        Ms Pacman & 456123 & 1  & 31403 & 0  \\  
        Name This Game & 212824 & 1  & 81473 & 1  \\  
        Phoenix & 19200 & 0  & 999999 & 0  \\  
        Pitfall & 7875 & 0  & -1 & 0  \\  
        Pong & 21 & 1  & 21 & 1  \\  
        Private Eye & 69976 & 0  & 15100 & 0  \\  
        Qbert & 999975 & 0  & 151730 & 0  \\  
        Riverraid & 35588 & 0  & 27964.3 & 0  \\  
        Road Runner & 999900 & 0  & 999999 & 0  \\  
        Robotank & 143 & 1  & 144 & 1  \\  
        Seaquest & 539456 & 0  & 1000000 & 1  \\  
        Skiing & -4185 & 0  & -5903.34 & 0  \\  
        Solaris & 20306 & 0  & 10732.5 & 0  \\  
        Space Invaders & 93147 & 0  & 159999.6 & 0  \\  
        Star Gunner & 609580 & 1  & 999999 & 1  \\  
        Surround & N/A & N/A & 2.726 & 0  \\  
        Tennis & 24 & 1  & 24 & 1  \\  
        Time Pilot & 183620 & 1  & 531614 & 1  \\  
        Tutankham & 528 & 0  & 436.2 & 0  \\  
        Up N Down & 553718 & 1  & 999999 & 1  \\  
        Venture & 3074 & 0  & 2200 & 0  \\  
        Video Pinball & 999999 & 0  & 999999 & 0  \\  
        Wizard of Wor & 199900 & 0  & 118900 & 0  \\  
        Yars Revenge & 999998 & 0  & 998970 & 0  \\  
        Zaxxon & 18340 & 0  & 241570.6 & 1 \\ 
        \midrule
         $\sum$ HWRB & & 15& & \textbf{24}  \\ 
         \bottomrule
\end{tabular}
\end{center}
\end{table}

\clearpage

\begin{table}[!hb]
\footnotesize
\begin{center}
\caption{Score Table of GDI-H$^3$ and LBC-$\mathcal{BM}$ (Ours) on HWRB.}
\label{Tab:Score table of GDI and LBC on HWRB.}
\setlength{\tabcolsep}{1.0pt}
\begin{tabular}{c cc cc }
\toprule
Games & GDI-H$^3$ & HWRB & Ours &HWRB   \\ 
        Scale & 200M & ~ & 1B &    \\  \midrule
        Alien & 48735	             & 0 & \textbf{279703.5} & 1 \\  
        Amidar &1065              & 0 & \textbf{12996.3} & 0 \\  
        Assault &\textbf{97155}	             & 1  & 62025.7 & 1 \\  
        Asterix &\textbf{{999999}}   & 0 & \textbf{999999} & 0 \\  
        Asteroids &{760005}            & 0 & \textbf{1106603.5} & 0 \\  
        Atlantis &\textbf{{3837300}}           & 0   & 3824506.3 & 0 \\  
        Bank Heist &1380              & 0  & \textbf{1410} & 0 \\   
        Battle Zone &824360            & 1 & \textbf{857369} & 1 \\  
        Beam Rider &422390            & 0 & \textbf{457321} & 0 \\  
        Berzerk &14649             & 0 & \textbf{35340} & 0 \\    
        Bowling &205.2             & 0 & \textbf{233.1} & 0 \\  
        Boxing &\textbf{{100}}      & 1  & \textbf{100} &1 \\  
        Breakout &\textbf{{864}}     & 1& \textbf{864} & 1 \\   
        Centipede &195630            & 0 & \textbf{728080} & 0 \\  
        Chopper Command &\textbf{{999999}}   & 1 & \textbf{999999} & 1 \\    
        Crazy Climber &\textbf{241170}	            & 1 & 233090 &1 \\  
        Defender &{970540}  & 0 & \textbf{995950} & 0 \\    
        Demon Attack &{787985}    & 0 & \textbf{900170} & 0 \\  
        Double Dunk &\textbf{{24 }}      & 1 & \textbf{24} & 1 \\  
        Enduro &14300            & 1 & \textbf{14332.5} & 1 \\  
        Fishing Derby &65              & 0 & \textbf{75} & 1 \\  
        Freeway &\textbf{{34}}     & 0 & \textbf{34} & 0 \\    
        Frostbite &11330	            & 0 & \textbf{13792.4} & 0 \\  
        Gopher &473560           & 1 & \textbf{488900} & 1 \\   
        Gravitar  &5915             & 0 & \textbf{6372.5} & 0 \\  
        Hero &\textbf{38225}	           & 0 & 37545.6 & 0 \\   
        Ice Hockey &47.11           & 0  & \textbf{47.53} & 1 \\  
        Jamesbond &{620780	}  & 1 & \textbf{623300.5} & 1 \\    
        Kangaroo &14636           & 0 & 14372.6 & 0 \\  
        Krull &{594540}          & 1 & 593679.5 &1 \\  
        Kung Fu Master &\textbf{{1666665}}	         & 1 & \textbf{1666665} &1 \\    
        Montezuma Revenge &\textbf{2500}            & 0 & \textbf{2500} & 0 \\  
        Ms Pacman &\textbf{11573}           & 0 & \textbf{31403} & 0 \\   
        Name This Game  &36296           & 1 & \textbf{81473} & 1\\  
        Phoenix &{959580 }        & 0 & \textbf{999999} & 0 \\  
        Pitfall  &-4.3            & 0 & \textbf{-1} & 0 \\  
        Pong &{21}     & 1 & \textbf{21} & 1 \\  
        Private Eye &15100          & 0 & 15100 & 0 \\   
        Qbert &28657          & 0 & \textbf{151730} & 0 \\  
        Riverraid &\textbf{28349}           & 0 & 27964.3 & 0 \\   
        Road Runner &{999999	} & 0  & \textbf{999999} &0 \\   
        Robotank &113.4           & 0 & \textbf{144} & 0 \\  
        Seaquest&\textbf{{1000000}}         & 1 & \textbf{1000000} & 1 \\  
        Skiing &\textbf{{-6025}}  & 0 & -5903.34 & 0 \\  
        Solaris &9105           & 0  & \textbf{10732.5} & 0 \\   
        Space Invaders  &{154380}         & 0 & \textbf{159999.6} & 0 \\  
        Star Gunner &{677590}         & 1 & \textbf{999999} & 1 \\  
        Surround  &2.606           & 0  & \textbf{2.726} & 0 \\  
        Tennis &\textbf{{24}}            & 1  & \textbf{24} & 1 \\   
        Time Pilot &450810	         & 1  & \textbf{531614} & 1 \\   
        Tutankham &418.2          & 0 & \textbf{436.2} & 0 \\   
        Up N Down &966590         & 1 & \textbf{999999} & 1\\  
        Venture  &2000	           & 0 & \textbf{2200} & 0 \\    
        Video Pinball &978190          & 0 & \textbf{999999} &0 \\  
        Wizard of Wor  &63735         & 0 & \textbf{118900} & 0 \\  
        Yars Revenge &968090         & 0& \textbf{998970} & 0 \\  
        Zaxxon &216020	          & 1 & \textbf{241570.6} & 1 \\  \midrule
                $\sum$ HWRB & &22   & &\textbf{24} \\ 
         \bottomrule
\end{tabular}
\end{center}
\end{table}

\clearpage

\section{Ablation Study}
\label{app: Ablation Study}

\renewcommand{\thesubfigure}{(\alph{subfigure})}
\setcounter{subfigure}{0}

\begin{figure*}[!t]
	\centering
% 	\vspace{-0.3in}
	\includegraphics[width=\textwidth]{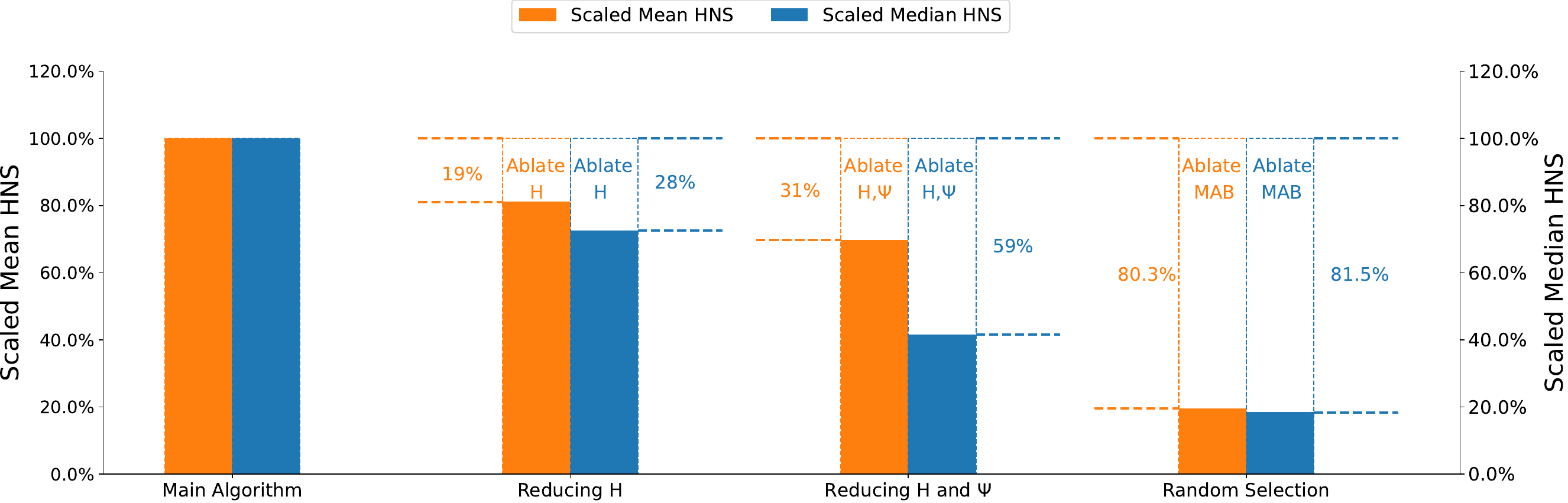}
	\centering
% 	\vspace{-0.1in}
\caption{Ablation Results on Atari Benchmark \citep{ale2}. All the results are scaled by  that of our main algorithm to improve readability. In these figures, we sequentially demonstrate how much performance (\%) will degrade after ablating each component of LBC.}
\label{app fig:ablation study}
% \vspace{-0.1in}
\end{figure*}

In this section, we will demonstrate the settings of our ablation studies first . Then we will introduce the algorithms of the ablation study, which has been concluded  in Tab. \ref{tab:Summary of the ablation experimental groups in Ablation Study}. After that, we will introduce the ablation study results and case studies of t-SNE, including the results on the Atari benchmark in Fig. \ref{app fig:ablation study} and t-SNE analysis in App. \ref{app: TSNE Analysis}.

\subsection{Ablation Study Setup}
\label{app: Ablation Study Setup}

 We  summarized all the algorithms of the ablation study in Tab. \ref{tab:Summary of the ablation experimental groups in Ablation Study}.  All algorithms are tested in the same experimental setup. More details on these experimental setups can see App. \ref{sec:app Experiment Details}.  The hyper-parameters can see App. \ref{Sec: appendix hyper-parameters}.

\begin{table*}[!htbp]
\small
\setlength{\tabcolsep}{2.0pt}
    \centering
    \caption{Algorithms of Ablation Study.}
    \label{tab:Summary of the ablation experimental groups in Ablation Study}
    \resizebox{\textwidth}{!}{% <------ Don't forget this %
    \begin{tabular}{l l l l l}
   \toprule
                Algorithm &   Main Algorithm &  Reducing $\mathbf{H}$  & 
                Reducing $\mathbf{H}$ and $\bm{\Psi}$ & Random Selection  \\
   \midrule
Ablation Variables & Baseline & $\mathbf{H}$& $\bm{\Psi}$ and $\mathbf{H}$& Meta-Controller ($\bm{\Psi}$)\\

            $\mathcal{F}_{\psi}$ & $\sum_{i=1}^{3} \omega_i \operatorname{Softmax}_{\tau_i}(\Phi_{\mathbf{h}_i})$ & $\sum_{i=1}^{3} \omega_i \operatorname{Softmax}_{\tau_i}(\Phi_{\mathbf{h}_i})$ & $ \operatorname{Softmax}_{\tau}(\Phi_{\mathbf{h}})$ & $\sum_{i=1}^{3} \omega_i \operatorname{Softmax}_{\tau_i}(\Phi_{\mathbf{h}_i})$ \\

                $\mathbf{\Phi}_{\mathbf{H}}$ & $(\Phi_{\mathbf{h}_1},...,\Phi_{\mathbf{h}_\mathrm{3}})$ & $\Phi_{\mathbf{h}_1}$ & $\Phi_{\mathbf{h}_1}$ & $(\Phi_{\mathbf{h}_1},...,\Phi_{\mathbf{h}_\mathrm{3}})$\\

                $\mathbf{h}_i$ & $(\gamma_i,\mathcal{RS}_i)$& $(\gamma_i,\mathcal{RS}_i)$ & $(\gamma_i,\mathcal{RS}_i)$ & $(\gamma_i,\mathcal{RS}_i)$ \\
                
                $\mathbf{H}$  & $\{\mathbf{h}_i|i=1,2,3\}$ & $\{\mathbf{h}_i|i=1\}$ & $\{\mathbf{h}_i|i=1\}$ & $\{\mathbf{h}_i|i=1,2,3\}$ \\

                $\bm{\psi}_i$ & $(\omega_1,\tau_1,\omega_2,\tau_2,\omega_3,\tau_3)$ & $(\omega_1,\tau_1,\omega_2,\tau_2,\omega_3,\tau_3)$ & $(\tau)$& $(\omega_1,\tau_1,\omega_2,\tau_2,\omega_3,\tau_3)$\\ 
                
                   $\bm{\Psi}$  & $\{\bm{\psi}_i|i=1,...,\infty\}$ & $\{\bm{\psi}_i|i=1,...,\infty\}$ & $\{\bm{\psi}_i|i=1,...,\infty\}$ & $\{\bm{\psi}_i|i=1,...,\infty\}$\\

                   $ |\mathbf{H}| $ & 3 & 1 & 1 &  3 \\

                   $ |\bm{\Psi}| $ & $\infty$ & $\infty$ & $\infty$ & $\infty$  \\

                 Category   & LBC-H$^\text{2}_\text{3}$-$\bm{\Psi}^\text{6}_{\infty}$ &  LBC-H$^\text{2}_\text{1}$-$\bm{\Psi}^\text{6}_{\infty}$  & LBC-H$^\text{2}_\text{1}$-$\bm{\Psi}^\text{1}_{\infty}$ & LBC-H$^\text{2}_\text{3}$-$\bm{\Psi}^\text{6}_{\infty}$  \\
                 
                 $|\mathbf{M}_{\mathbf{H},\bm{\Psi}} |$ & $\infty$ & $\infty$ & $\infty$ & $\infty$ \\
                 
                 Meta-Controller ($\bm{\Psi}$)
                 & MAB   
                 & MAB   
                 & MAB 
                 & Random Selection             \\
    \bottomrule
    \end{tabular}
    }
\end{table*}
\normalsize

\subsection{Ablation Study Results}

% \begin{Proposition}
    
% \end{Proposition}

The ablation study results can be found in Fig. \ref{app fig:ablation study}. From left to right, the behavior space of
the first three algorithms decreases in turn, and the final performance of these three algorithms decreases in turn. We can draw the following corollary:

\begin{Corollary}[Smaller Behavior Space, Lower Final Performance]
\label{Corollary: Smaller Behavior Space, Lower Final Performance}
    Given any RL methods, assuming each behavior can be visited infinitely, decreasing the behavior space and keeping other conditions unchanged will degrade the final performance of the algorithm,  and vice versa.
\end{Corollary}

The behavior space of Random Selection is the same as our main algorithm. Obviously, the appropriate behaviors fail to be selected with a random selection, resulting in a great decrease of the performance in limited training frames.

\subsection{t-SNE Analysis}
\label{app: TSNE Analysis}

 In this paper, we adopt the ucb-score to encourage the actors to try more different behavior. To demonstrate the effectiveness of the ucb-score, we conduct the t-SNE analysis of the methods removing the ucb-score in 1 and 2 of  Fig. \ref{fig:tsne study}. 

 To demonstrate that the behavior diversity can be boosted by our algorithm, we conducted  the t-SNE analysis of the methods with rule-based $\mathcal{F}$ in 2 and 3 of Fig. Fig. \ref{fig:tsne study}.

From (a) and (b) of Fig. \ref{fig:tsne study}, we find that removing the UCB item (i.e., $\sqrt{\frac{\log (1 + \sum_{j \neq k}^{\mathrm{K}} N_{\Psi_j})}{1 + N_{\Psi_k}}}$) from the optimization target of behavior selection, the behavior diversity fade away. It can prove the effectiveness of the diversity control of our methods.

From (c) and (d) of Fig. \ref{fig:tsne study}, we find that compared with the rule-based $\mathcal{F}$, our method can acquire a diverse set of behaviors though we do not contain a diversity-based multi-objective model training which confirms the Corollary \ref{Corollary: behavior mapping optimization is an Antidote for Behavior Circling}.

\begin{figure*}[!t]
\vspace{-0.1in}
\subfigure[Main Algorithm]{
		\includegraphics[width=0.23\textwidth,height=0.2\textheight]{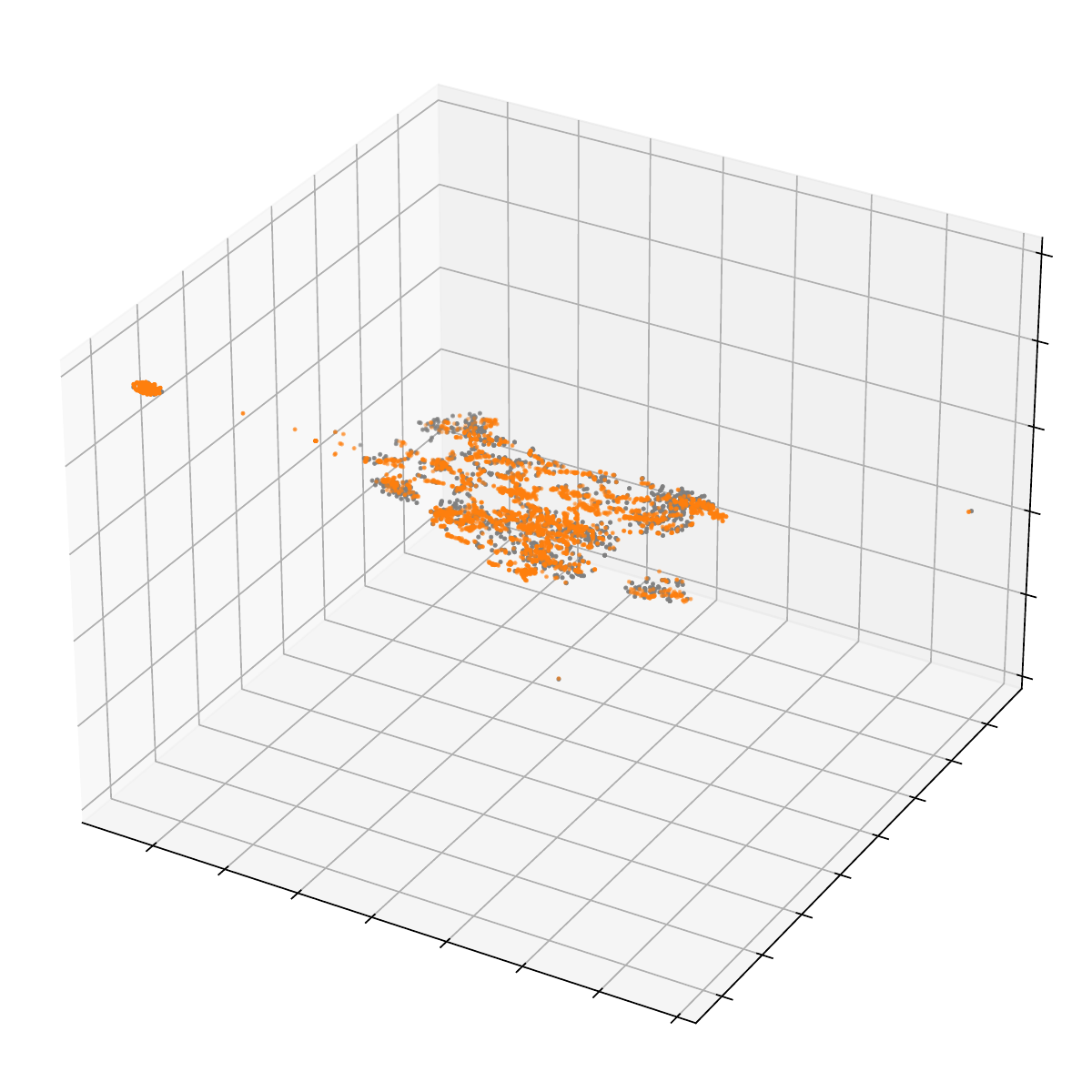}
	}
	\subfigure[w/o UCB Item]{
		\includegraphics[width=0.23\textwidth,height=0.2\textheight]{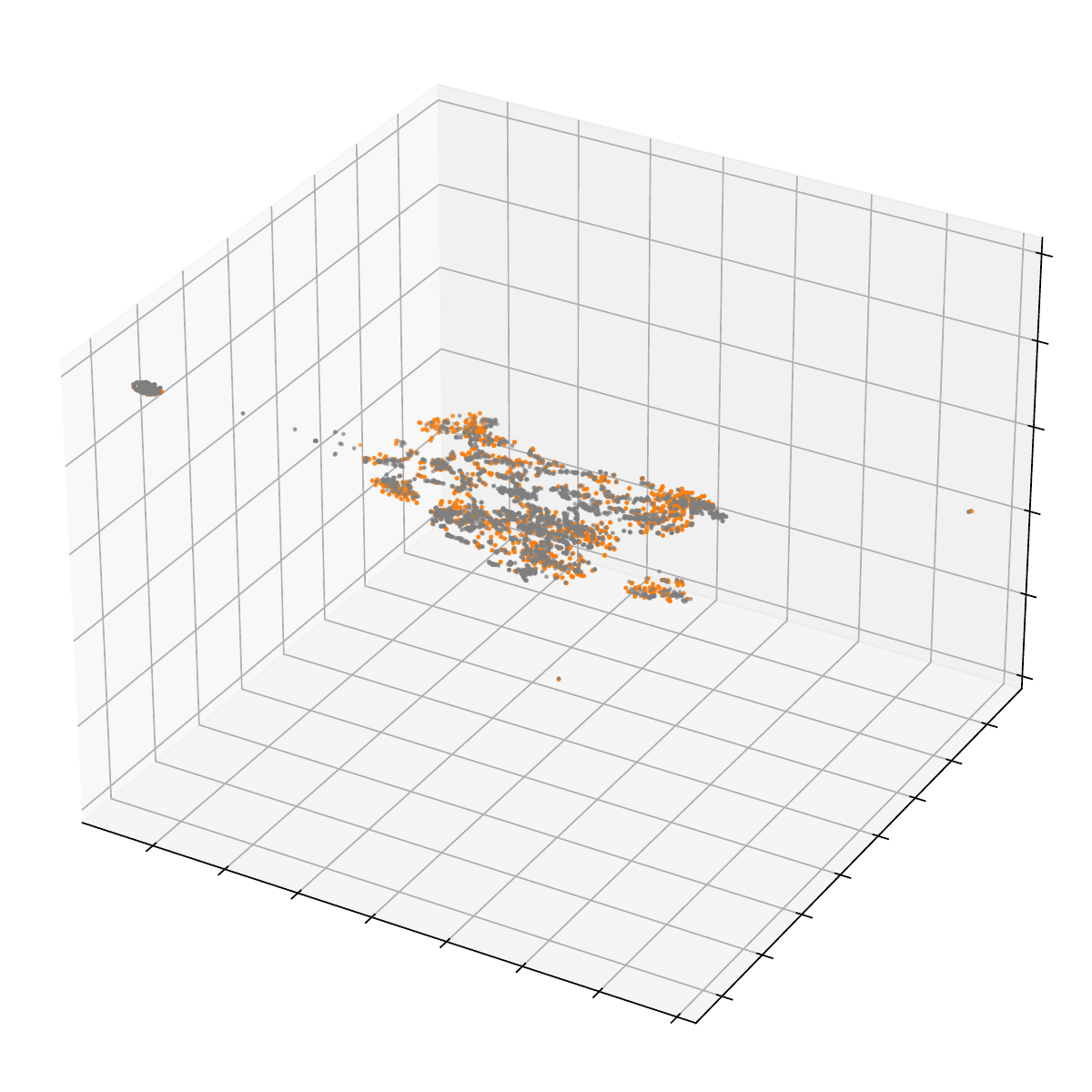}
	}
	\subfigure[Main Algorithm]{
		\includegraphics[width=0.23\textwidth,height=0.2\textheight]{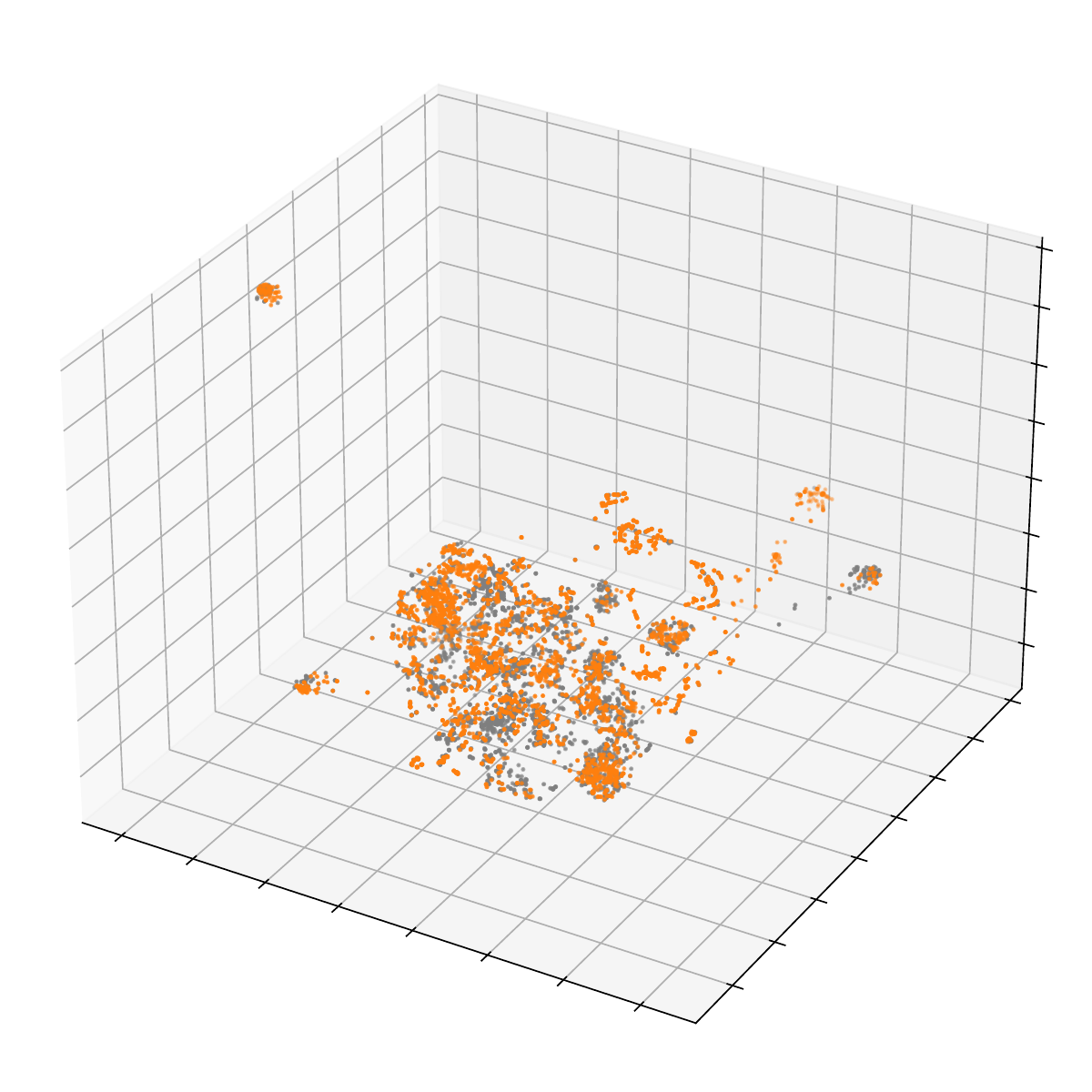}
	}
	\subfigure[Rule-based $\mathcal{F}$]{
		\includegraphics[width=0.23\textwidth,height=0.2\textheight]{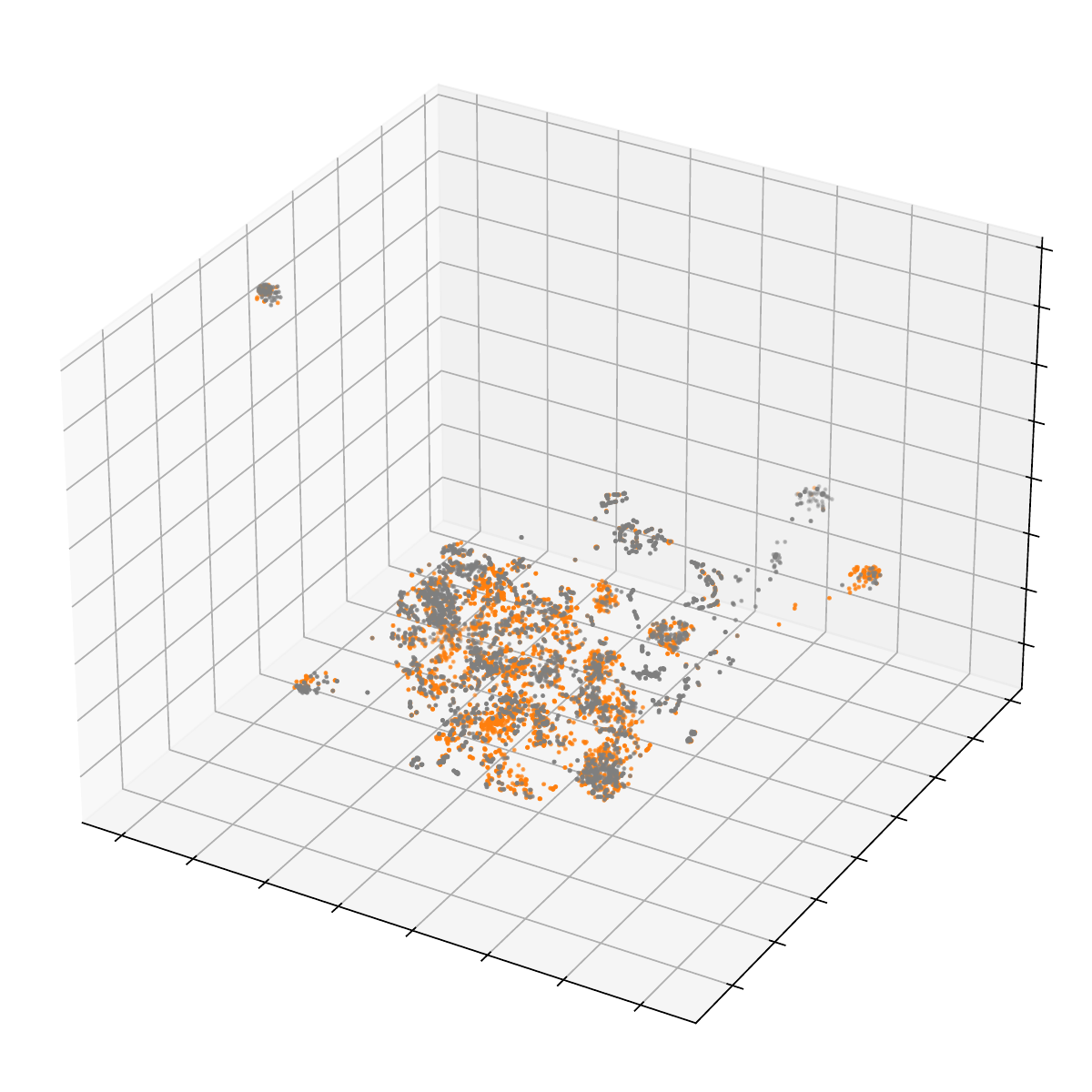}
	}
\caption{Visualizing Behavior Diversity via t-SNE. (a) and (b) are drawn from the t-SNE analysis of visited states (points highlighted  with \imgintext{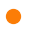}) in  Chopper Command, and (c) and (d) are drawn the t-SNE analysis of visited states in Atlantis.}
\label{fig:tsne study}
\end{figure*}

\clearpage

\section{Model Architecture}
\label{app: Model Architecture}

Since the network structure is not the focus of our work, we keep most of the components of the network, e.g., the LSTM Core, RL Head and Convolutional Layers the same as that of Agent57 \citep{agent57}. To improve the reproducibility of our work, we still summarize our  model architecture of our main algorithm in detail in Fig. \ref{fig:lbc model}.  Wherein, $\mathbf{h}_j^t$ is the the hyper-parameters (e.g., discounted factor $\gamma$ ) of policy model $j$ and $\bm{\bm{\psi}}_t$ is the parameters of the constructed behavior space (e.g., $\epsilon$ in $\epsilon$-greedy). More details on the hyper-parameters of each policy model can see App. \ref{Sec: appendix hyper-parameters}. $\bm{\bm{\psi}}_t$ will be adaptively selected to control the behaviors across learning process via MAB. More implementation details on the MAB can see App. \ref{Sec: appendix MAB}. It is worth noting that our framework is not limited to an implementation of report in our body. In Fig. \ref{fig:lbc model}, we show a general way of integrating multiple policy models and automatically adjusting the proportion of any multiple policy  models in the ensembled behavior policy.

\begin{figure*}[!ht]
	\centering
% 	\vspace{-0.3in}
	\includegraphics[width=\textwidth]{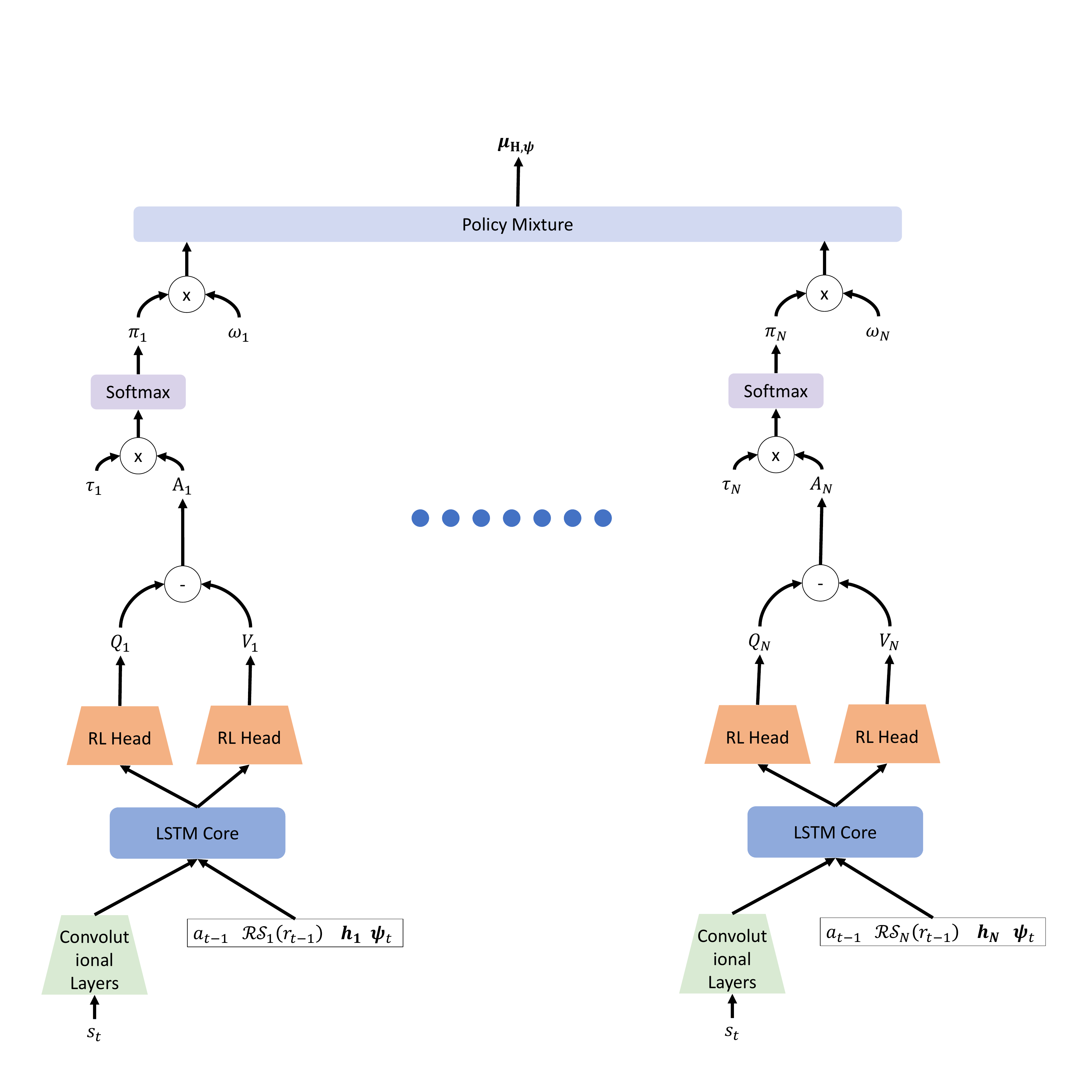}
	\centering
% 	\vspace{-0.1in}
\caption{Model Architecture of our main algorithm.}
\label{fig:lbc model}
% \vspace{-0.1in}
\end{figure*}

\clearpage

\section{Supplementary Material}
\label{app: Supplementary Material}

\begin{figure*}[!ht]
	\centering
% 	\vspace{-0.3in}
	\includegraphics[width=0.5\textwidth]{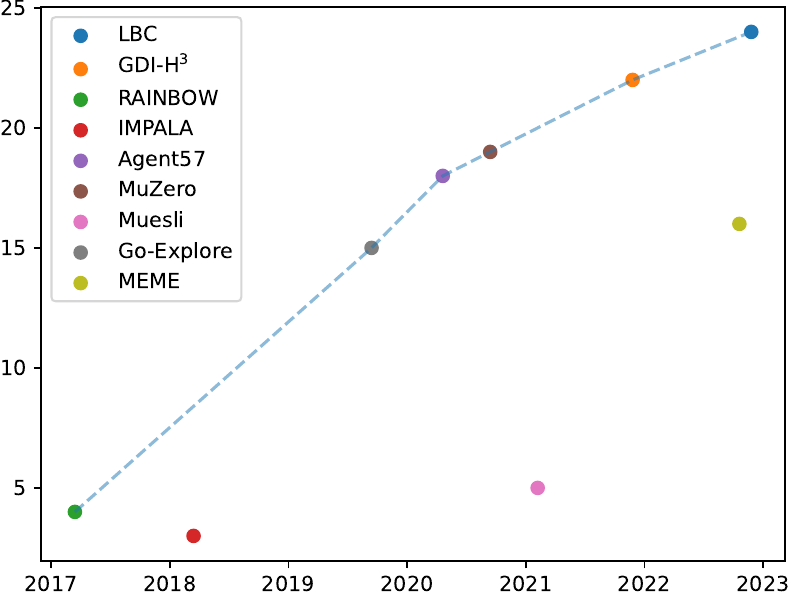}
	\centering
% 	\vspace{-0.1in}
\caption{Human World Records Breakthrough of Atari RL Benchmarks.}
\label{fig:hwrb benchmark}
% \vspace{-0.1in}
\end{figure*}

\begin{figure*}[!ht]
	\centering
% 	\vspace{-0.3in}
	\includegraphics[width=\textwidth]{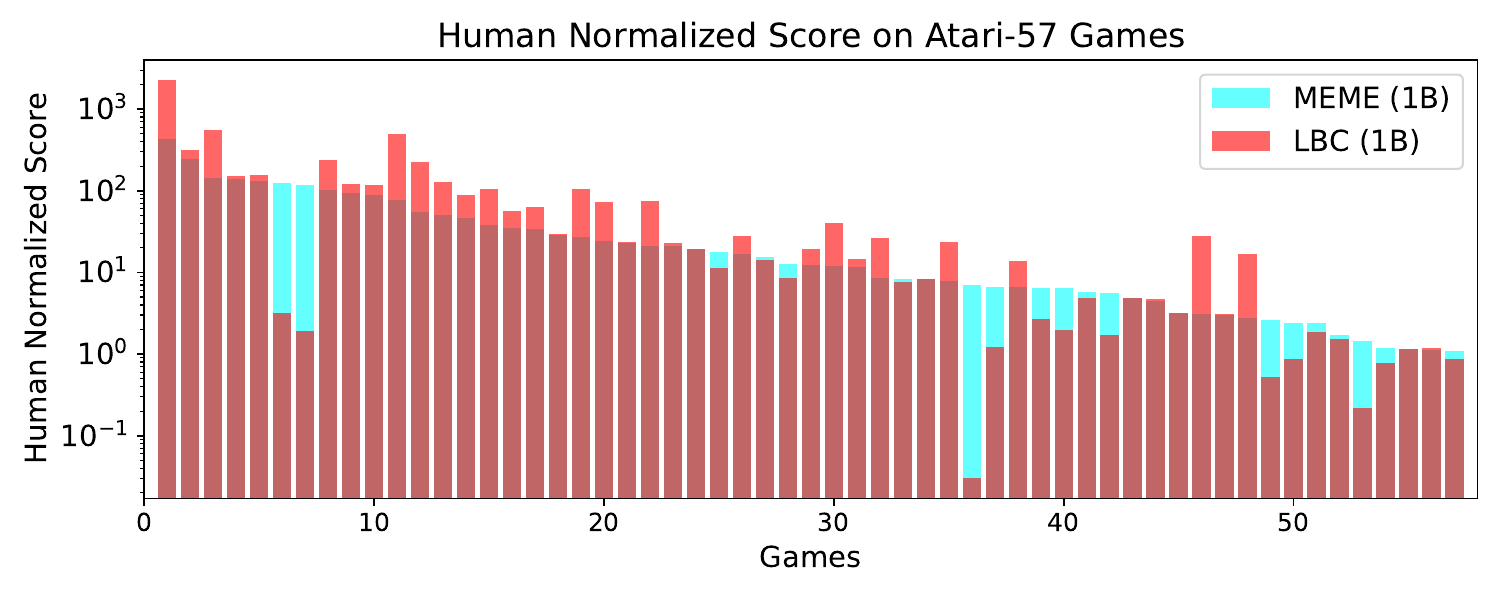}
	\centering
% 	\vspace{-0.1in}
\caption{Performance Comparison between MEME and LBC based on HNS.(log scale)}
\label{fig:performance meme and lbc}
% \vspace{-0.1in}
\end{figure*}

% \subsection{Different Learning Framework of GDI and LBC}

\begin{figure*}[!ht]
	\centering
% 	\vspace{-0.3in}
	\includegraphics[width=\textwidth]{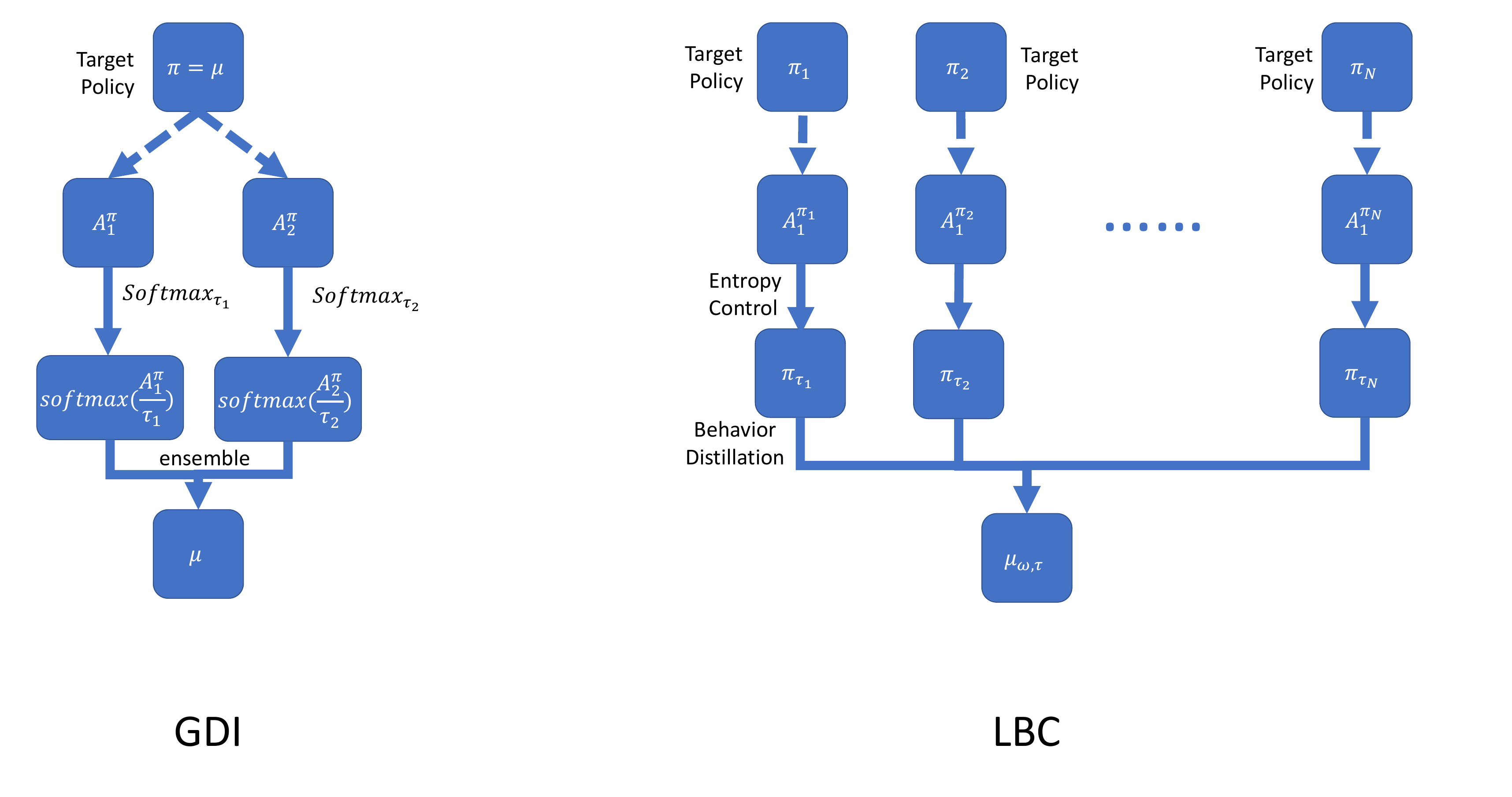}
	\centering
% 	\vspace{-0.1in}
\caption{Different Learning Framework of GDI-H$^3$ and LBC}
\label{fig:gdi and lbc}
% \vspace{-0.1in}
\end{figure*}

\clearpage

\section{Case Study: KL Divergence}
\label{sec: KL Divergence Between policy models}

In this section, to further investigate the cause of the degradation phenomenon of the behavior space in GDI-H$^3$ (i.e., due to a same learned policy $\pi$ for $A_1^{\pi}$ and $A_2^{\pi}$ under different reward shaping) and demonstrate that in our behavior space (i.e., different learned policies for $A_1^{\pi_1}$ and $A_2^{\pi_2}$), the diversity can be maintained since different policy models  are distinguishable across learning. For fairness, we designed two implementations to explore the degradation phenomenon in the behavior space of GDI-H$^3$ including: i) an implementation  with two different learned policies under different reward shaping (yellow in Fig. \ref{fig:KL}) ii) an implementation that learns two advantage functions of a same target policy (i.e., the behavior policy) under different reward shaping as GDI-H$^3$ (blue in Fig. \ref{fig:KL}). The learning framework of these two implementations can be found in \ref{fig:gdi and lbc}. 

For a fair comparison, we keep the two reward shaping the same as used in  GDI-H$^3$, namely,  i) $\log (abs (r) + 1.0) \cdot (2 \cdot 1_{\{r \geq 0\}} - 1_{\{r < 0\}})$ for $A_1$ and ii) $sign(r) \cdot ((abs (r) + 1.0)^{0.25} - 1.0) + 0.001 \cdot r$ for $A_2$.

\begin{figure*}[!ht]
\vspace{-0.1in}
\subfigure[$\operatorname{
Softmax
}_{\tau_1}(A_1)$ and $\operatorname{
Softmax
}_{\tau_2}(A_2)$]{
		\includegraphics[width=0.45\textwidth]{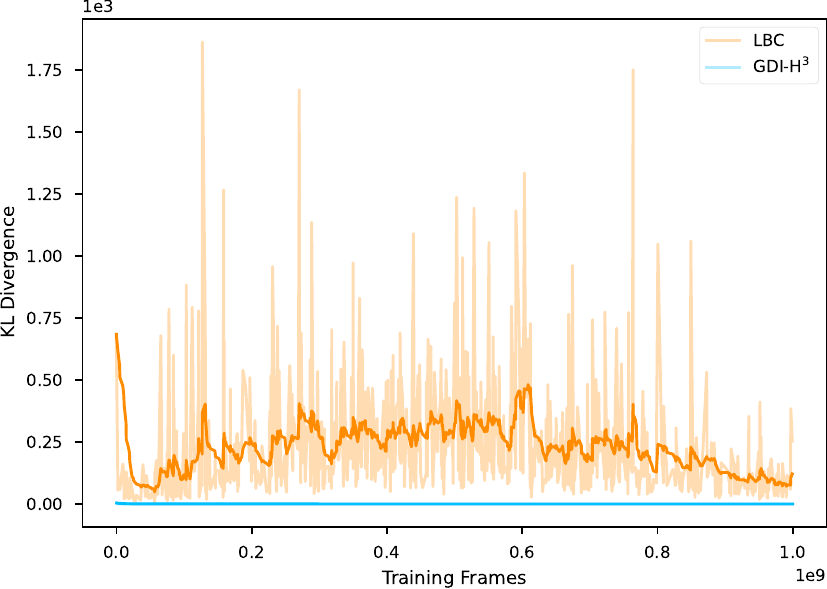}
	}
	\subfigure[$\operatorname{
Softmax
}(A_1)$ and $\operatorname{
Softmax
}(A_2)$]{
		\includegraphics[width=0.45\textwidth]{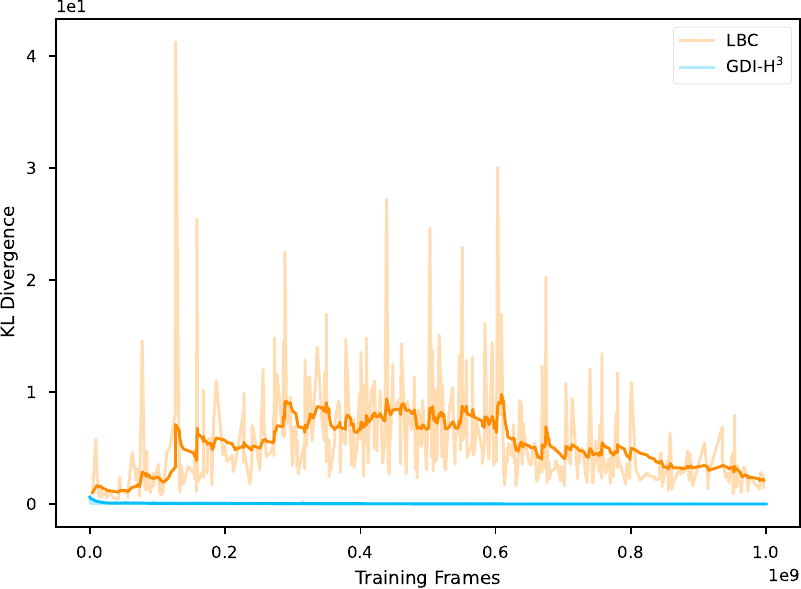}
	}

\subfigure[$\operatorname{
Softmax
}_{\tau_1}(A_1)$ and $\operatorname{
Softmax
}_{\tau_2}(A_2)$]{
		\includegraphics[width=0.45\textwidth]{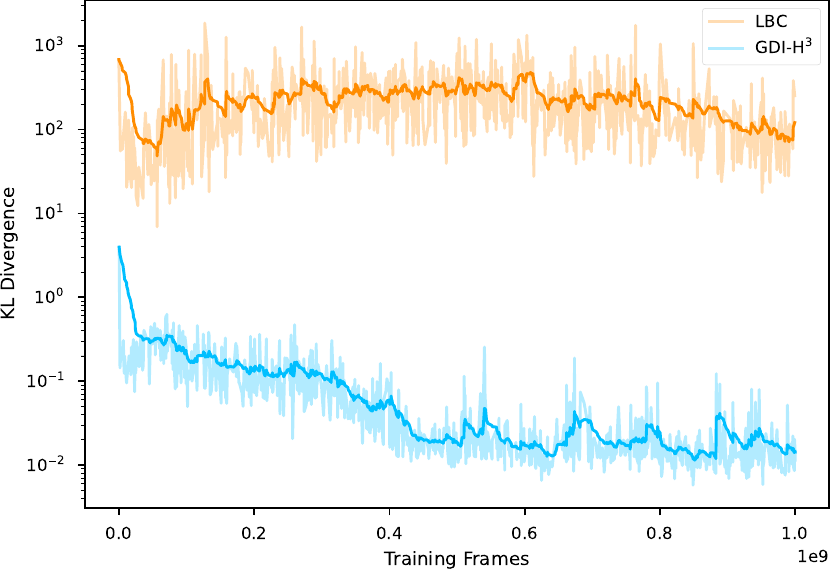}
	}
	\subfigure[$\operatorname{
Softmax
}(A_1)$ and $\operatorname{
Softmax
}(A_2)$]{
		\includegraphics[width=0.45\textwidth]{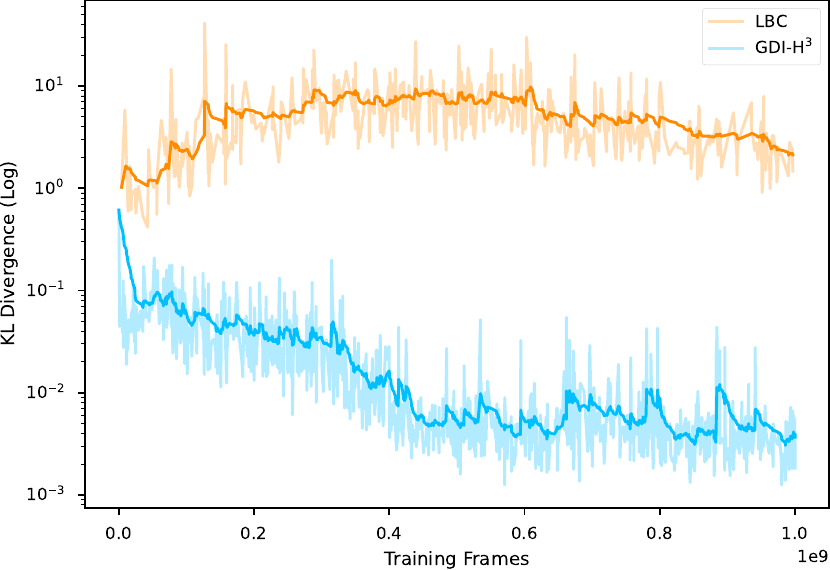}
	}
\caption{KL Divergence of GDI-H$^3$ and LBC in Chopper Command (Smoothed by 0.9 for the ease of reading).}
\label{fig:KL}
\end{figure*}

From Fig. \ref{fig:KL}, we can find the distance between $\operatorname{
Softmax
}_{\tau_1}(A_1^{\pi})$ and $\operatorname{
Softmax
}_{\tau_2}(A_2^{\pi})$ of $A_1^\pi$ and $A_2^\pi$ decrease rapidly in GDI-H$^3$ while LBC can maintain a more diverse set of policies. The optional behaviors for each actor gradually diminish, and the behavior space of GDI-H$^3$ degenerates across the learning process. In contrast, LBC can maintain the capacity of the behavior space and avoid degradation since LBC maintains a population of different policy models (Corresponding to a population of different policies.)